\pdfoutput=1

\documentclass[11pt]{article}

\usepackage[final]{acl}

\usepackage{times}
\usepackage{latexsym}

\usepackage[T1]{fontenc}

\usepackage[utf8]{inputenc}

\usepackage{microtype}

\usepackage{inconsolata}

\usepackage{graphicx}

\usepackage{makecell}

\usepackage{amsmath}
\usepackage{amssymb}
\usepackage{amsthm}

\newtheorem{theorem}{Theorem}

\newcommand{\x}{\mathbf{x}}
\newcommand{\y}{\mathbf{y}}
\newcommand{\f}{\mathbf{f}}
\newcommand{\g}{\mathbf{g}}
\newcommand{\R}{\mathbb{R}}
\newcommand{\EX}{\mathbb{E}}

\title{Prediction Hubs are Context-Informed Frequent Tokens in LLMs}

\author{Beatrix M.~G.~Nielsen \\
  Technical University of Denmark \\
  \texttt{bmgi@dtu.dk} \\\And
  Iuri Macocco \\
  Universitat Pompeu Fabra \\
  \texttt{iuri.macocco@upf.edu} \\\And
   Marco Baroni \\
  Universitat Pompeu Fabra/ICREA \\
  \texttt{marco.baroni@upf.edu} \\
  }

\begin{document}
\maketitle
\begin{abstract}
Hubness, the tendency for a few points to be among the nearest neighbours of a disproportionate number of other points, commonly arises when applying standard distance measures to high-dimensional data, often negatively impacting distance-based analysis. As autoregressive large language models (LLMs) operate on high-dimensional representations, we ask whether they are also affected by hubness. We first prove that the only large-scale representation comparison operation performed by LLMs, namely that between context and unembedding vectors to determine continuation probabilities, is not characterized by the concentration of distances phenomenon that typically causes the appearance of nuisance hubness. We then empirically show that this comparison still leads to a high degree of hubness, but the hubs in this case do not constitute a disturbance. They are rather the result of context-modulated frequent tokens often appearing in the pool of likely candidates for next token prediction. However, when other distances are used to compare LLM representations, we do not have the same theoretical guarantees, and, indeed, we see nuisance hubs appear. There are two main takeaways. First, hubness, while omnipresent in high-dimensional spaces, is not a negative property that needs to be mitigated when LLMs are being used for next token prediction. 
Second, when comparing representations from LLMs using Euclidean or cosine distance, there is a high risk of nuisance hubs and practitioners should use mitigation techniques if relevant.
\end{abstract}

\section{Introduction}

Hubness is a phenomenon which occurs in high-dimensional data \citep{radovanovic2010hubs}, where some data points (the hubs) are in the k nearest neighbours of many other points while most points (the anti-hubs) are in the k nearest neighbours of few or no other points. Hubness has been found in many different types of data: for example in time-series, biology and image processing \citep{tomavsev2011role,tomavsev2014hubness} and, in relation to text, in bag-of-words embeddings \citep{radovanovic2010hubs,schnitzer2012local}, dense word embeddings \citep{Dinu2014ImprovingZL}, dense sentence embeddings \citep{pmlr-v233-nielsen24a} and cross-modal embeddings \citep{bogolin2022cross}. Hubs arise due to intrinsic properties of certain distance measures applied to high-dimensional spaces, and they are typically considered a nuisance, as they obfuscate the genuine semantic landscape of the data of interest. Consequently, there is a general interest in techniques to reduce the hubness of a representation space \citep[see, for instance,][]{feldbauer2019comprehensive}.

Autoregressive large language models (LLMs) also trade in high-dimensional representations, and it is thus natural to ask whether hubs emerge in their distance computations. This is the question we answer in this study. In order to address it, it is fundamental to distinguish between the comparison operations a model is effectively performing when engaging in next-token prediction and distance-based comparisons we might decide to compute from its representations.

Concerning the distance-based comparisons actually performed by a standard autoregressive transformer-based LLM \citep{elhage2021mathematical}, we note that the model prediction is accomplished through the softmaxed dot product between a context representation and each row of the unembedding matrix. This operation effectively determines a rank over the whole token vocabulary of a model (typically made up of thousands of elements), and it can be seen as a distance-based measure that could be affected by nuisance hubs.\footnote{Technically, another dot product is computed, within the attention modules, between the query vector of a token and the key vectors of the preceding tokens. Since in this case the potential ``neighbours'' are constrained to be the tokens in the preceding context, which are meaningful elements (as long as we are looking at meaningful text), we do not expect nuisance hubs to affect this operation.}

We first present a theoretical analysis of the softmaxed context-unembedding dot product operation, which defines a measure that we will call, from now on, \textit{probability distance}. We show that probability distance, under reasonable assumptions, is not affected by the \textit{concentration of distances} phenomenon that typically leads to nuisance hubness in high-dimensional spaces. Interestingly, we also find, empirically, that probability distance is still characterized by high hubness, but these hubs are not noise. Instead, they correspond to context-modulated frequent tokens that are often reasonable guesses, given that natural language text is characterized by very skewed word distributions \citep{Baayen:2001}. Indeed, when the most likely continuation according to the model is a hub, this prediction is often the correct one.

On the other hand, a researcher might be interested in performing other similarity comparisons between inner representations of a LLM: for example, looking for the nearest neighbours of a sentence, as represented by its hidden-activation last-token vector, or of a vocabulary entry, as represented in the unembedding matrix.\footnote{We focus on the unembedding matrix because it is the one we are also studying in the context of probability distance computations, but we expect similar trends to emerge for the embedding matrix as well.} It is already theoretically known that, when using Euclidean distance in this context, hubs might arise due to concentration of distances. We confirm empirically that such measurements are generally affected by nuisance hubness, although, surprisingly, concentration of distances is not observed in all cases.

Our main contributions are as follows:

\begin{itemize}
    \item We present the first theoretical and empirical analysis of hubness in autoregressive, transformer-based LLMs;
    \item We show that the hubs that arise in the prediction computations of the model are not a trivial effect of concentration of distances, but reflect a guessing heuristic exploiting the skewed nature of word frequency distributions, and should thus not be eliminated;
    \item We show that other similarity computations involving LLM representations are instead affected by nuisance hubness, and thus they should only be performed in combination with hubness reduction techniques.
\end{itemize}

\section{Related Work}
\citet{radovanovic2010hubs} showed the ubiquity of hubs in many different kinds of datasets. 
Hubness is a cause of concern, as it can negatively impact many common tasks in data analysis and machine learning, such as regression, classification, outlier detection and clustering. Hubness was also shown to hinder the performance of nearest-neighbour algorithms in speech recognition, recommendation and multimedia retrieval \citep[see][and references therein]{feldbauer2019comprehensive}. Problematic hubness also occurs in distributed text representations analogous to those produced by a LLM. For example \citet{Dinu2014ImprovingZL}, \citet{smith2017offline}, \citet{lample2018word}, \citet{huang2020improving} and \citet{pmlr-v233-nielsen24a} studied hubness in word and text embeddings, while \citet{bogolin2022cross}, \citet{wang-etal-2023-balance} and \citet{chowdhury-etal-2024-nearest} looked at hubness in multimodal language models and cross-modal retrieval. 

Given the problems posed by hubs, various hubness reduction methods have been proposed, for example Local Scaling \citep{zelnik2004self}, Mutual Proximity \citep{schnitzer2012local}, Globally Corrected Rank \citep{Dinu2014ImprovingZL}, Inverted Softmax \citep{smith2017offline}, Cross-domain Similarity Local Scaling \citep{lample2018word}, Hubness Nearest Neighbor Search \citep{huang2020improving}, Querybank Normalisation \citep{bogolin2022cross}, DBNorm \citep{wang-etal-2023-balance}, Dual Inverted Softmax \citep{wang-etal-2023-balance}, F-norm \citep{pmlr-v233-nielsen24a} and Nearest Neighbor Normalization \citep{chowdhury-etal-2024-nearest}. These methods apply different strategies to reduce hubness and its effects. For example, Mutual Proximity makes the nearest neighbour relation more symmetric by considering the joint probability of two points being each other's nearest neighbours conditioned on the distances to all other points. On the other hand, F-norm forces the data to follow a normal distribution in each dimension and normalizes the lengths of the vectors, thus making the data closer to a distribution which does not usually exhibit hubness. As a further example, Globally Corrected Rank reverses similarity queries, so that the nearest neighbour of a point $x$ is the point $y$ to which $x$ is nearest, among all possible candidates (as opposed to the point that is nearest to $x$). Since many points are close to hubs (in relative terms), hubs are unlikely to have $x$ among their nearest neighbours. Many of these methods have been systematically compared by \citet{feldbauer2019comprehensive} and \citet{pmlr-v233-nielsen24a}, among others. 

As shown by the plethora of hubness reduction techniques, the focus has so far been on mitigating hubness, with little attention devoted to the question of whether hubness is actually always a nuisance phenomenon to be mitigated.

\section{Theoretical preliminaries}

We first define the $k$-occurrence, $N_k$, as in \citet{radovanovic2010hubs}. Given a set of points, the $k$-occurrence of a specific point $x$, $N_k(x)$, is the number of points for which $x$ is in the $k$-nearest neighbours. We define hubs as points, $h$, with high $k$-occurrence, i.e., where $N_k(h)$ is large. To get a sense of which values of $N_k(x)$ should be considered large, we can analyze the distribution of the $k$-occurrences of a dataset. If the neighbourhood relation is relatively symmetric, and most points are in the $k$ nearest neighbours of $k$ other points, the distribution of $k$-occurrences will have a peak at $k$ and also be relatively symmetric, see Fig.~\ref{fig:concentration_dist_syn_data} (Bottom). This is the usual case in low dimensions. However, if we have some points, hubs, with a $k$-occurrence much larger than $k$, we will get a skewed distribution. Thus, like in \citet{radovanovic2010hubs} and \citet{feldbauer2019comprehensive}, we use the skewness of the distribution of $k$-occurrences ($k$-skew) to measure the hubness of a dataset. Recall that for a collection of $n$ data points, $\mathbf{x}$, the skewness is calculated as
\begin{align}
    \text{skew}(\mathbf{x}) =\frac{1}{n} \sum_{i = 1}^n\left(\frac{\mathbf{x}_i - \mu_{\mathbf{x}}}{\sigma_{\mathbf{x}}}\right)^3
\end{align}
where $\mu_{\mathbf{x}}$ is the mean and $\sigma_{\mathbf{x}}$ is the standard deviation of $\mathbf{x}$. If the $k$-occurrence distribution is completely symmetric, we get a $k$-skew of 0.   

\subsection{Hubness and concentration of distances}
Concentration of distances happens when the difference between the largest and smallest distance to a point goes to zero as the dimension increases. Necessary and sufficient conditions for this to happen have been presented in \citet{beyer1999nearest,durrant2009nearest}. When concentration of distances occurs, for every query point, we have that every other point is almost equally far away, see Fig.~\ref{fig:concentration_dist_syn_data} (Top). 

\begin{figure}[htb]
  \includegraphics[width=\columnwidth]{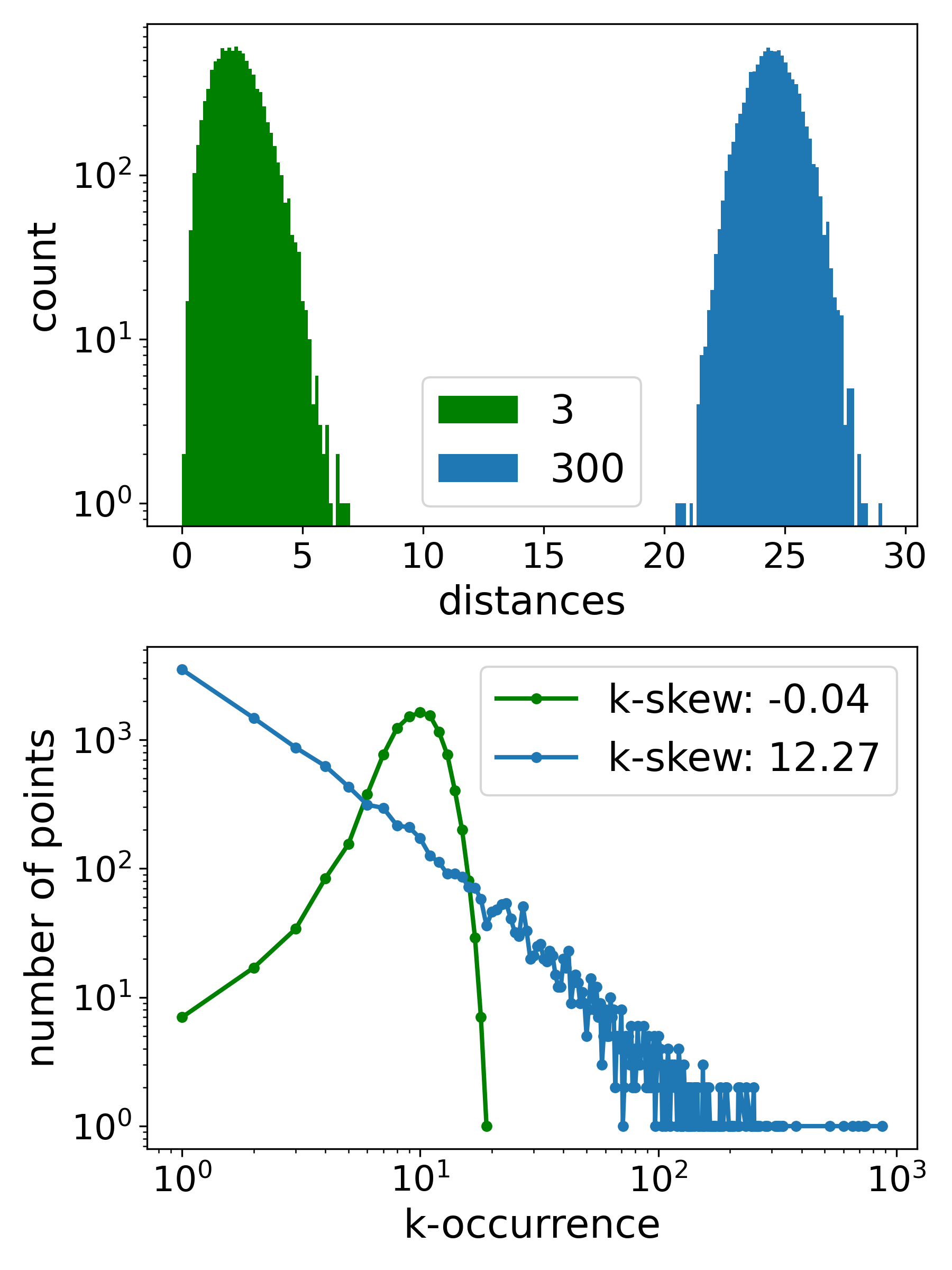}
  \caption{Illustrative example of concentration of distances and k-occurrence. (Top) Distribution of 10,000 Euclidean distances between query and comparison points from a standard Gaussian in 3 and 300 dimensions. In 300  dimensions, no pair of points has a distance between 0 and 20, and most have a distance around 25, so the distances ``concentrate''. (Bottom) K-occurrence distributions for the data in (Top). For 3 dimensions, k-skew is close to 0, so the neighbour relation is symmetric. For 300 dimensions, k-skew is quite high (about 12), so the neighbour relation is very skewed in accordance with the data exhibiting a concentration of distances.}
  \label{fig:concentration_dist_syn_data}
\end{figure}

A first effect of the concentration of distances is that, while every point will, trivially, still have a nearest neighbour, just adding a small amount of noise is likely to change which points are the closest. Another consequence is that, in high dimension, all points will be close to lying on a hypersphere, and be quite sparsely distributed. If we take a point which is slightly closer to the mean of the data than most other points, then this point will now be the closest neighbour of many other points (although it is still quite far away from everything), i.e.,  this point will be a hub. 

Therefore, if we are attempting to compare high-dimensional representations using a distance measure which exhibits concentration of distances, we will get that most representations are far away from each other. However, a few hubs will be the nearest neighbours of many other representations, with no guarantee that they are close in any meaningful sense. We call this kind of hubs, solely arising due to concentration of distances, \textit{nuisance hubs}.

\subsection{Probability distance in LLMs and concentration of distances}

When comparing the representations of LLMs, it is common to use Euclidean distance or cosine similarity, which is equivalent to normalized Euclidean distance in terms of neighbour ranking. However, Euclidean distance is affected by concentration of distances \citep{aggarwal2001surprising}. We thus expect to find nuisance hubs when using it to compare representations. 

Does this mean that LLMs are adversely affected by hubness? As discussed in the introduction, models are not using Euclidean-distance-based comparisons as part of their inner workings. They are trained instead to compare contexts with possible vocabulary items and give the most likely next items a high probability. We can interpret this as a dissimilarity measure, the \textit{probability distance}, by using $1-p(y \; \vert \; x)$, where $p(y \; \vert \; x)$ is the probability the model associates to item $y$ given the context $x$. In this way, we construct neighbourhoods for each context, with the closest items being the ones which are most likely.

The following theorem shows that, when using probability distance, we do not get concentration of distances unless the probabilities are uniform.

\begin{theorem}
\label{theorem:prob_dist_no_concentration}
Let $\x_i \in X$ be a data point. Let $\y_j$, $j\in \{1, ..., v\}$, be the possible labels of points from $X$, and let $p(\y_j|\x)$ be the probability of label $\y_j$ given $\x$ which uses representations $\f(\x), \g(\y) \in \R^m$.
We define the dissimilarity between $\x_i$ and $\y_j$ to be $d(\x_i, \y_j) = 1-p(\y_j|\x_i)$. Then, if the distribution over $\y$ does not go to the uniform distribution for every $\x$, $p(\y|\x) \not\to \textit{U}(\y)$, we will not get concentration of distances for this dissimilarity as the dimension $m \to \infty$. 
\end{theorem}
\begin{proof}
    In Appendix \ref{Appendix:proof_prob_dist_no_concentration}
\end{proof}

For LLM predictions in language models, this proof means that, as long as our models do not assign close to equal probabilities to all tokens for all the given contexts, there will be no concentration of distances. Table \ref{table:l2_dist_uni} in Appendix \ref{app:l2_dists_table_and_explanation} shows that, when we compare contexts with vocabulary items, the mean L2 distance to the uniform distribution is very far from zero for all models. This is expected since, for any given context, some items will be much more likely than others, and LLMs have been expressly trained to make accurate in-context predictions. 

Note that Theorem \ref{theorem:prob_dist_no_concentration} does not imply that there will be no hubs for the probability distance measure used by LLMs, but if hubs are present, they will not be nuisance hubs due to concentration of distances. Note also that the theorem does not say anything about what happens when using Euclidean or cosine distance to compare representations. 

\section{Experiments}

All code for experiments and plots can be found on github.\footnote{\url{https://github.com/bemigini/hubs-are-frequent-tokens}}

\subsection{Setup}

We experiment with five different autoregressive LLMs, namely OPT-6.7B \citep{Zhang:etal:2022b}, Llama-3-8B \citep{Meta:2024}, Pythia-6.9B \citep{Biderman:etal:2023}, OLMo-7B \cite{Groeneveld:etal:2024}, and Mistral-7B \citep{Jiang:etal:2023}, hereon referred to as Opt, Llama, Pythia, Olmo, and Mistral, respectively. As input to the models, we use the 3 datasets made available by \citet{Cheng:etal:2025}. Each of them consists of 50K sequences, or \textit{contexts}, as we will call them, of 20 orthographic tokens  randomly extracted from Bookcorpus \citep{Zhu:etal:2015}, Pile10k \citep{Gao:etal:2020} and WikiText-103 \citep{Merity:etal:2016}, respectively. Note that these contexts start and end at random points in a text (in particular, the last token is not necessarily a punctuation mark). In order to estimate domain-specific token frequency distributions, we use the full corpora the contexts were extracted from. 


\begin{figure}[htb]
  \includegraphics[width=\columnwidth]{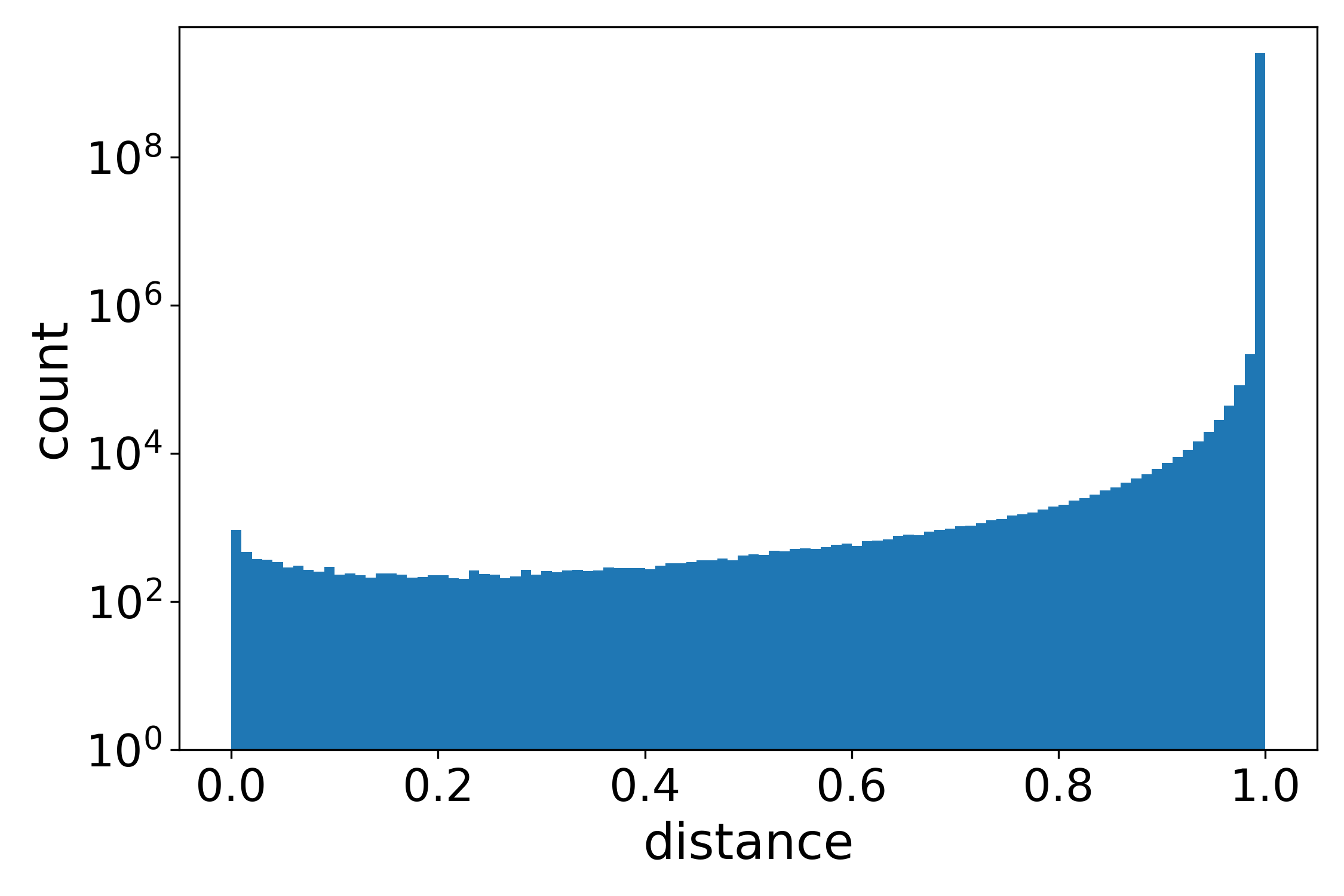}
  \caption{Probability distance distribution for Pythia on contexts from Pile10k. If we had had a concentration of distances, we would not see this spread of distances all the way to zero (compare with Fig.~\ref{fig:concentration_dist_syn_data}).}
  \label{fig:hist_ct_llama_pile}
\end{figure}

\begin{figure}[htb]
  \includegraphics[width=\columnwidth]{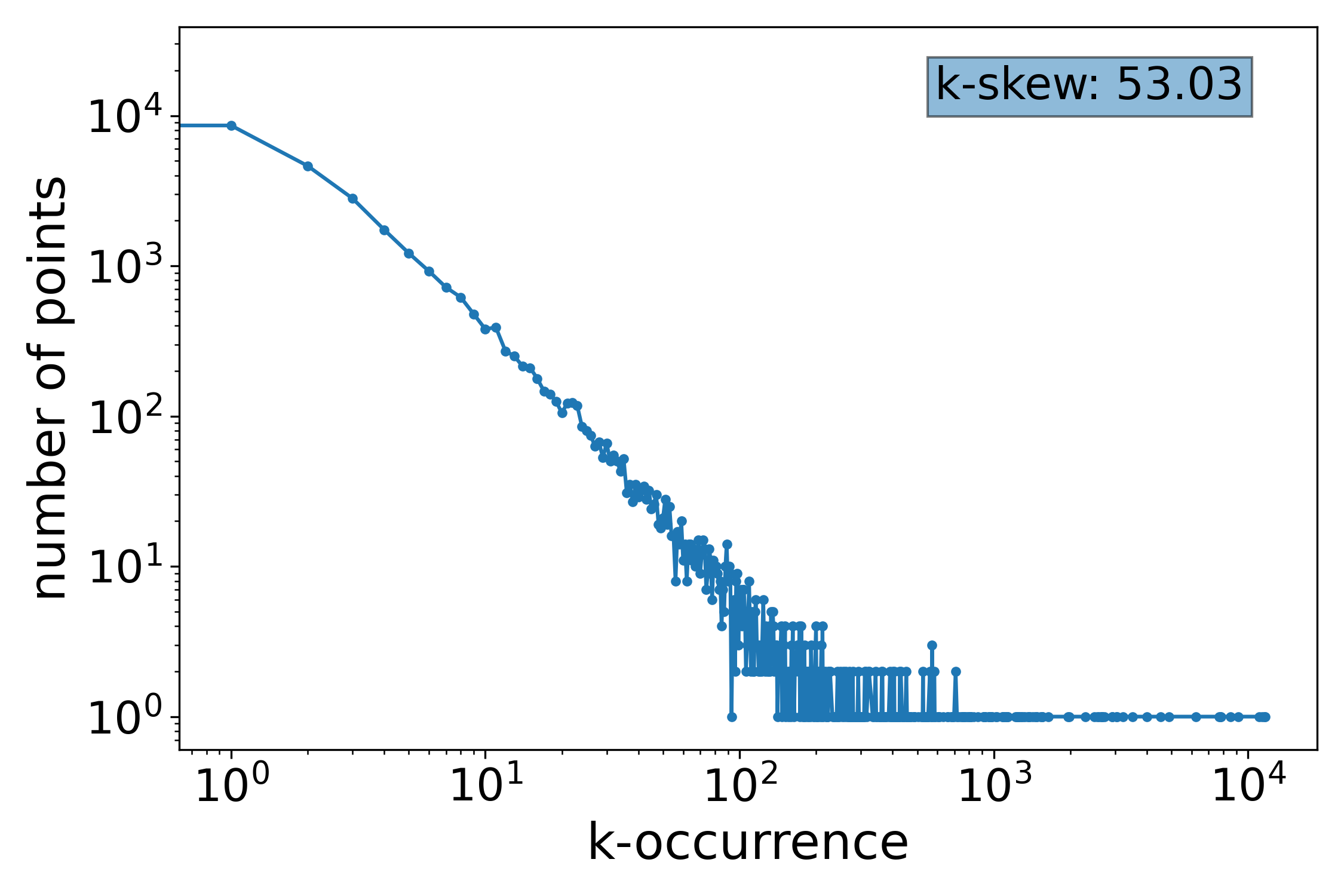}
  \caption{k-occurrence distribution for Pythia predictions on contexts from Pile10k. This distribution is highly skewed with many hubs (points with $k$-occurrence larger than 100). }
  \label{fig:llama_k_occ_dist_wikitext}
\end{figure}

To measure hubness, we set $k=10$ and define a point $x$ as a hub if it has $N_k(x) \geq 100$. That is, a point is a hub if it is in the $10$ nearest neighbours of $10$ times more points than we would expect if the relationship had been symmetric. We informally ascertained that our conclusions are robust to changes in these hyperparameters.

\subsection{Probability distance in LLMs}

In this section, we first confirm that the probability distances computed by LLMs do not exhibit concentration of distances. We then show that, despite this, all tested LLMs are characterized by high hubness. We find however that their hubs correspond to context-dependent frequent tokens, that tend to be reasonable prediction candidates.

Fig.~\ref{fig:hist_ct_llama_pile} shows, for Pythia and Pile10k, that there is no concentration of distances, as predicted by Theorem \ref{theorem:prob_dist_no_concentration}. This fact is confirmed for the other models in Appendix \ref{app:distribution_of_prob_distances}.

Given the lack of concentration of distances, LLM probability neighbourhoods should not be characterized by nuisance hubs. However, all models still have a very high $k$-skewness. A $k$-skewness of 3 already means that either there are many points which are in the $k$ nearest neighbours of more than $k$ other points (there are many points with a $k$-occurrence larger than the mean), or there are a few points which are in the $k$ nearest neighbours of substantially more than $k$ points (a few points have $k$-occurrences much larger than the mean). Thus, a $k$-skewness of $3$ could already be considered high, but all models have $k$-skewness higher than $40$ for all three datasets (Table \ref{table:pred_hub_occurrence} in Appendix \ref{app:hub_occurrence}). Indeed, in all cases we find hubs, that is, tokens with a $k$-occurrence larger than 100. In fact, all models have at least one vocabulary item with a $k$-occurrence higher than 10,000 for all datasets. As an example, the $k$-occurrence distribution of Pythia on Pile10k is shown in Fig.~\ref{fig:llama_k_occ_dist_wikitext}.

\begin{table*}
  \centering
  \begin{tabular}{l|lllll|lllll|lllll|}
    \hline
                     & \multicolumn{5}{|l|}{\textbf{Pile10k} }       & \multicolumn{5}{|l|}{\textbf{Bookcorpus} }   & \multicolumn{5}{|l|}{\textbf{Wikitext-103} }   \\
    \hline
    \textbf{Pythia}  &  \textbackslash n & and & the & , & in  &  the & . & , & and & \textbackslash n &  and & the & , & in & a \\
    \textbf{Olmo}    &  and & the & , & . & in &  the & . & , & and & \textbackslash n &  and & the & , & in & . \\
    \textbf{Opt}     &  \textbackslash n & and & the & , & .  &  the &. &and &, &\textbackslash n &  the &and &, & in & \textbackslash n \\
    \textbf{Mistral} & \textbackslash n & the & and & , & .  &  the & . & and & , & \textbackslash n & and & the & , & in & . \\
    \textbf{Llama}   & \textbackslash n & , & the & and & . & \textbackslash n & the & . & , & and & \textbackslash n & the & and & , & in \\
    \hline
  \end{tabular}
  \caption{\label{table:pred_hub_examples}
    Top five prediction hubs for the various LLMs on different datasets.  Intuitively, they are all very frequent tokens, that also coincide across models. 
  }
\end{table*}

If the hubs do not come from concentration of distances, where do they come from? By qualitative inspection, we observe that the hubs correspond to intuitively frequent tokens, as shown in Table \ref{table:pred_hub_examples}. To make this intuition more formal, we plotted the $k$-occurrence of the hubs against the frequencies of occurrence of the tokens in the various datasets. We found that, for all models, there is a high Spearman correlation ($0.63$ or larger) between the $k$-occurrence of the hubs and the frequencies of the vocabulary items in the dataset which the model is making predictions on.\footnote{In all plots using log scales, we have added a small constant, $10^{-9}$, to the frequencies, in order to make the points with 0 frequency visible. Tokens with 0 frequency therefore all lie on a horizontal line at $10^{-9}$ in our plots. Note that, for all models and all datasets, there are some vocabulary items which have frequency $0$ even though they are hubs in the predictions. These are tokens that do not occur in the datasets but are frequently predicted by the LLMs due to tokenization and pre-processing discrepancies between the training corpora and the datasets. For example, for Llama on Pile10k,  `.\textbackslash n' is frequently predicted, but it never occurs in the dataset (where periods and newlines were systematically separated during pre-processing). As another example, the Bookcorpus is systematically lower-cased, so a LLM will predict frequent capitalized tokens (e.g., \textit{The}) that never occur in this dataset.} %
 For example, comparing $k$-occurrences of hubs in Pythia's predictions on Pile10k with the frequency of tokens in Pile10k gives a Spearman correlation of $0.71$ (Fig.~\ref{fig:pythia_pred_hubs_pile_vs_frequency_pile}; all correlations in Table \ref{table:pred_hub_corr} of Appendix \ref{app:k-occurrence-frequency-correlation}).

\begin{figure}[htb]
  \includegraphics[width=\columnwidth]{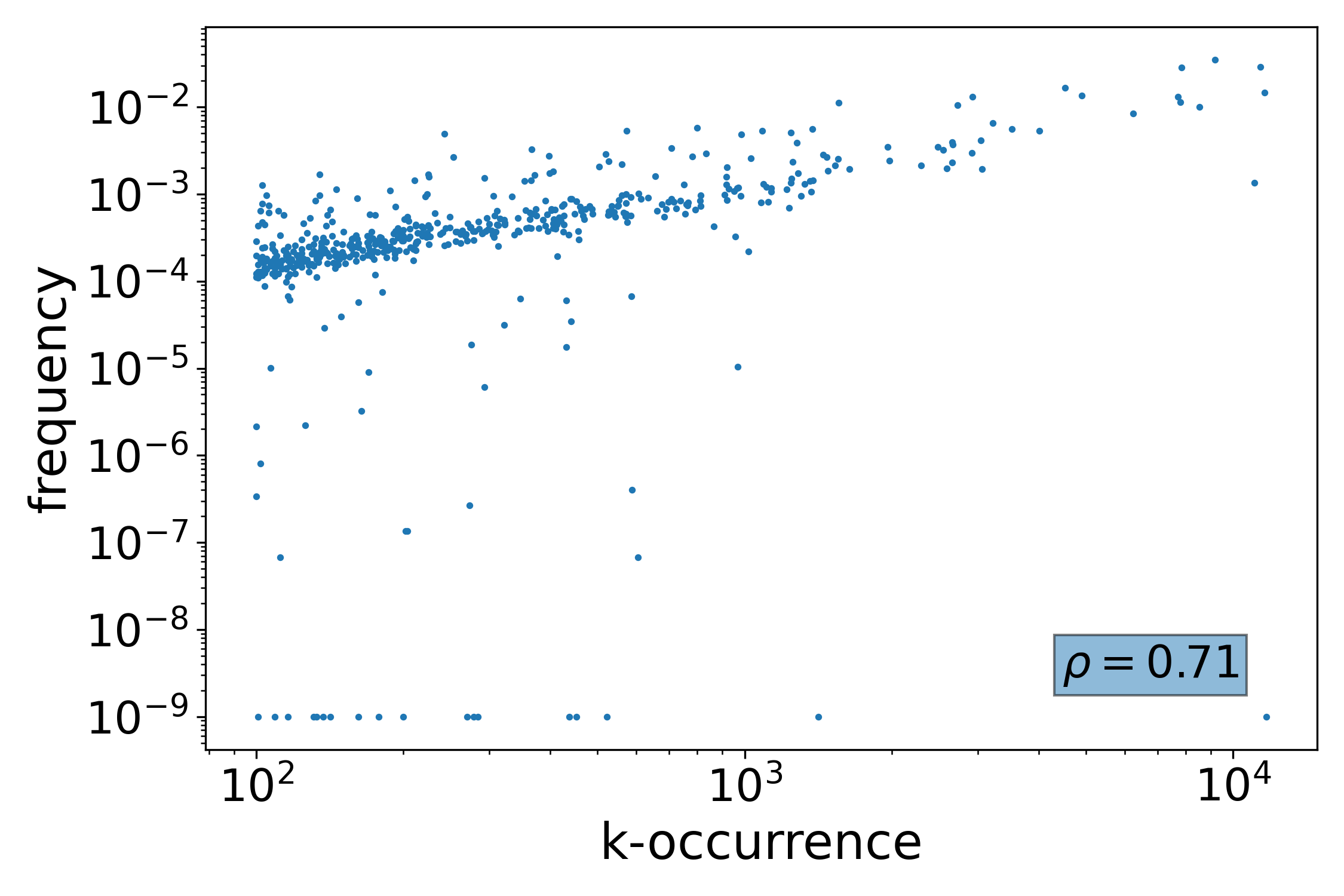}
  \caption{$k$-occurrence of hubs in Pythia predictions on contexts from Pile10k vs.~frequency of vocabulary items in Pile10k. $\rho$ is the Spearman correlation. }
  \label{fig:pythia_pred_hubs_pile_vs_frequency_pile}
\end{figure}

Thus the probability distance computed by LLMs during predictions \textit{is} characterized by high hubness, but this high hubness is \textit{not} a nuisance phenomenon, but the reflection of how LLMs adapted to word frequency distributions. Given that LLMs must predict the next token in natural text, and natural text is characterized by very skewed distributions, all models have learned to often predict very frequent tokens (punctuation marks, \textit{the}, \textit{of}, etc.).

\begin{figure}[htb]
  \includegraphics[width=\columnwidth]{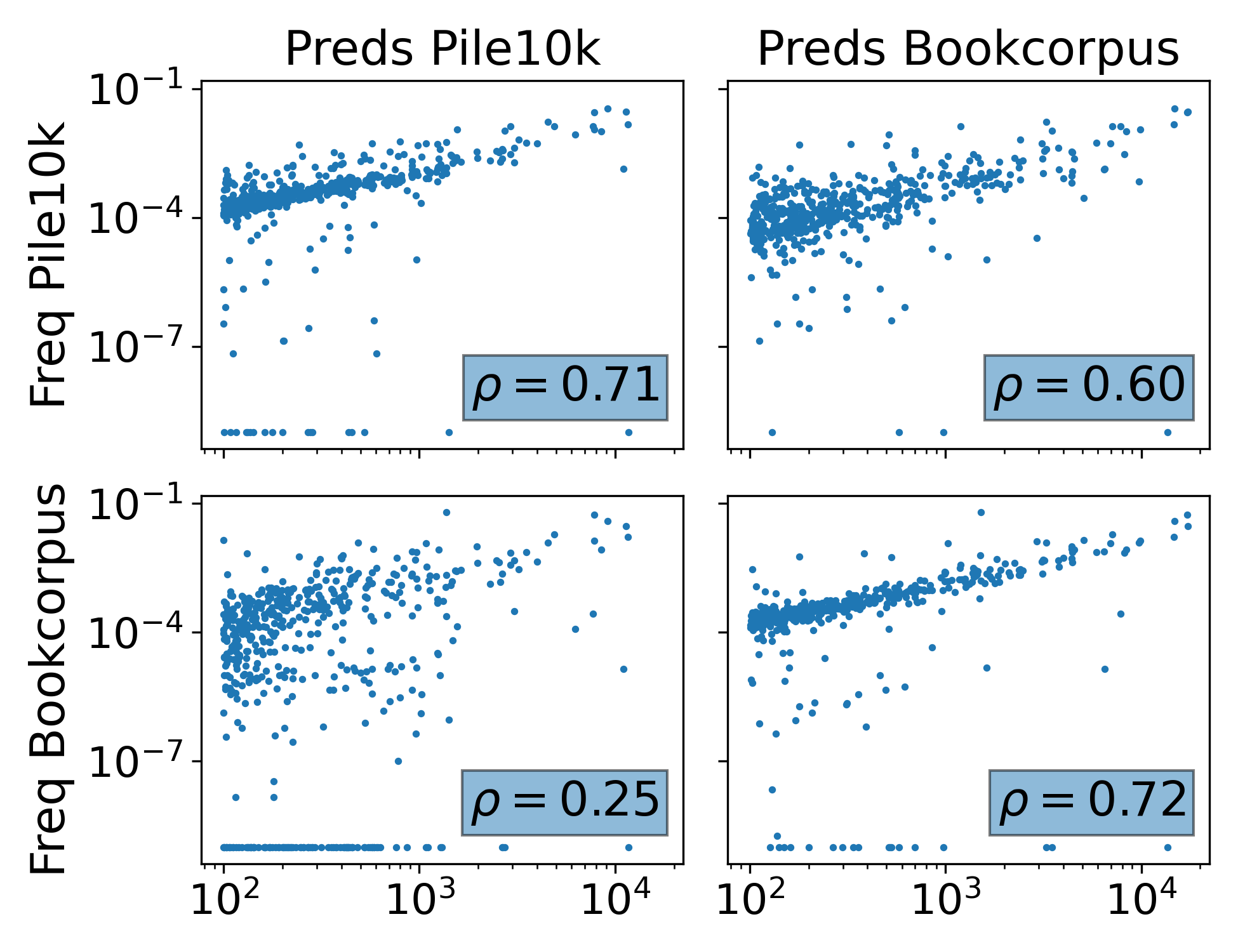}
  \caption{$k$-occurrence of hubs in Pythia predictions (x-axis) vs.~frequency of tokens (y-axis). $\rho$ is the Spearman correlation. Top row: Predictions made on contexts from Pile10k. Bottom row: Predictions made on contexts from Bookcorpus. First column: Frequency of tokens in Pile10k. Second column: Frequency of tokens in Bookcorpus. In both cases, correlation is higher when frequency is estimated on the same corpus as the contexts used for prediction.}
  \label{fig:pythia_pred_hubs_pile_bookcorpus_vs_frequency_pile_bookcorpus}
\end{figure}

Interestingly, the hubs are not simply fixed based on a single frequency distribution (e.g., that of the training corpus). Instead, they are modulated by the type of text the LLM is predicting. This is shown by the fact that, given a context extracted by one of the datasets, $k$-occurrence is more highly correlated with frequency estimates extracted from the corpus that dataset is extracted from, than with estimates from the other corpora. For example, Fig.~\ref{fig:pythia_pred_hubs_pile_bookcorpus_vs_frequency_pile_bookcorpus} shows that, for Pythia, the correlation of Pile10k hub $k$-occurrences  with frequencies estimated on the Bookcorpus is only $0.25$, but if we instead compare with frequencies from the Pile10k corpus we get a much higher correlation of $0.71$.

Unlike the nuisance hubs in the literature we reviewed above, which often harm performance, the context-modulated, frequent-token-predicting hubs emerging in LLMs look benign. Indeed, when a model predicts a hub as the most likely continuation, this actually leads on average to \textit{higher} accuracy than when the model is predicting a non-hub. For example, when Pythia predicts a non-hub for Pile10k contexts, it has an accuracy of about $28$\%, but when it predicts a hub, it has an accuracy of $39$\% (Table \ref{table:hub_accuracy}). 

\begin{table}    
    \begin{tabular}{llrrr}
    \hline
    \textbf{model} & \textbf{context} & \textbf{all} & \textbf{hub} & \textbf{non-hub} \\
    \hline
    Pythia & Pile10k & 0.37 & 0.39 & 0.28 \\
    Pythia & WikiText-103 & 0.36 & 0.38 & 0.30 \\
    Pythia & Bookcorpus \hspace{-0.6em} & 0.31 & 0.32 & 0.23 \\
    Olmo & Pile10k & 0.36 & 0.39 & 0.29 \\
    Olmo & WikiText-103 & 0.36 & 0.38 & 0.32 \\
    Olmo & Bookcorpus \hspace{-0.6em} & 0.32 & 0.33 & 0.24 \\
    Opt & Pile10k & 0.34 & 0.37 & 0.26 \\
    Opt & WikiText-103 & 0.35 & 0.37 & 0.31 \\
    Opt & Bookcorpus \hspace{-0.6em} & 0.30 & 0.31 & 0.22 \\
    Mistral & Pile10k & 0.35 & 0.38 & 0.27 \\
    Mistral & WikiText-103 & 0.36 & 0.37 & 0.31 \\
    Mistral & Bookcorpus \hspace{-0.6em} & 0.32 & 0.33 & 0.24 \\
    Llama & Pile10k & 0.37 & 0.40 & 0.31 \\
    Llama & WikiText-103 & 0.38 & 0.40 & 0.35 \\
    Llama & Bookcorpus \hspace{-0.6em} & 0.33 & 0.34 & 0.25 \\
    \hline
    \end{tabular}
    \caption{\label{table:hub_accuracy}
    Prediction accuracy over all contexts, accuracy on hubs and accuracy on non-hubs. Accuracy is higher for hubs than non-hubs for all models on all datasets.  
  }
\end{table}

\subsubsection{Emergence of frequency-sensitive prediction hubs during training}

Having established that hubs in LLMs are the product of a sensible token prediction heuristic, we might wonder if this behavior is due to an intrinsic model bias, or it emerges during training. Focusing on Pythia, whose intermediate training checkpoints are publicly available, we find that hubs appear in predictions from the very beginning, as shown by the $k$-skewness values reported in Table \ref{table:ct_pythia_steps_hub_occurrence} (Appendix \ref{app:hub_occurrence}). However, Fig.~\ref{fig:pythia_three_steps} shows that the correlation of $k$-occurrence with frequency is relatively low in the earlier stages of training, and becomes larger as training progresses. This suggests that, on the one hand, the model might have an intrinsic bias towards hubness in prediction, but, on the other hand, learning to constantly keep context-relevant frequent tokens in the top candidate pool is a strategy that is acquired during training, because it is advantageous for the prediction task.

\begin{figure}[htb]
  \includegraphics[width=\columnwidth]{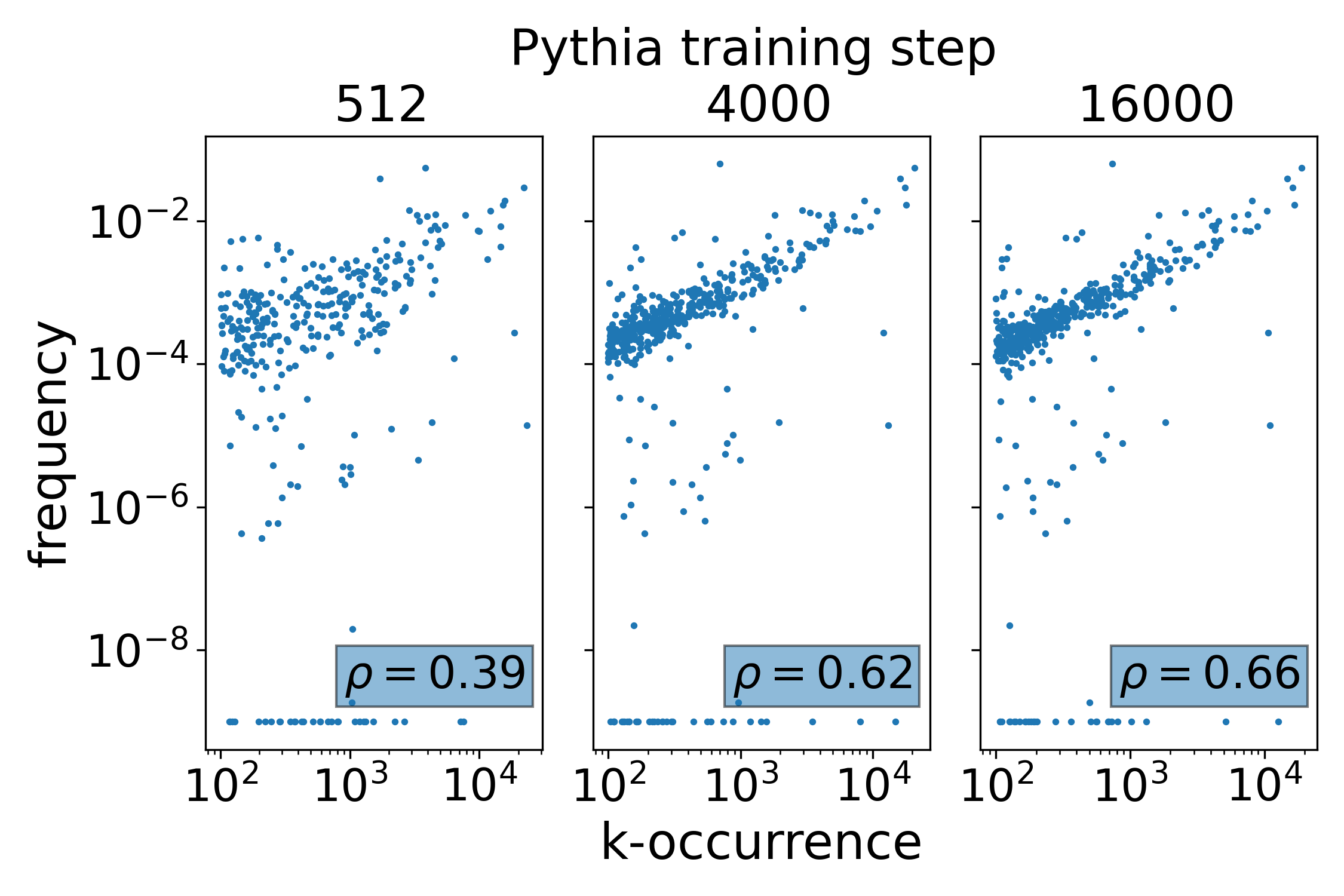}
  \caption{$k$-occurrence of hubs in Pythia predictions on Bookcorpus (x-axis) vs frequency from Bookcorpus (y-axis) for three checkpoints. $\rho$ is the Spearman correlation. The final number of training steps is 143,000, at which point $\rho=0.72$. The correlation saturates faster on Pile10k (a subset of Pythia's training data) than on Bookcorpus and WikiText-103, so we show an example from Bookcorpus to better display the gradual increase.}
  \label{fig:pythia_three_steps}
\end{figure}

\subsection{Comparing contexts or vocabulary items with Euclidean distance}


Having shown that the probability distance measure computed by LLMs during next token prediction is not affected by nuisance hubs, we turn to other comparisons that, while not relevant to LLM generation, might arise in LLM analysis or when using LLM-derived representations in downstream tasks. In particular, one might want to compute similarities between LLM representations of sequences or vocabulary entries for interpretability purposes or for specific downstream tasks that require measuring the similarity of two sentences or passages, as represented by their last-token activation vectors. In such cases, it is natural to use Euclidean distance or normalized Euclidean distance (or the rank-equivalent cosine) to compare representations. As we mentioned above, these measures \textit{are} affected by concentration of distances given various underlying distributions \citep{aggarwal2001surprising}, and we thus might observe the rise of nuisance hubs. We present here examples using Euclidean distance; normalized Euclidean and full results are in appendices \ref{app:distribution_of_cc_distances} and \ref{app:distribution_of_tt_distances}.

Starting with distance between context representations (that is, the last-layer/last-token representations of the sequences in our datasets), when we consider the distribution of distances between contexts using plain or normalized Euclidean distance, we get concentration of distances for all models, in the sense that the distance distributions do not have support all the way to zero. However, the distances are not as tightly concentrated around a single value as they were in the toy example of Fig.~\ref{fig:concentration_dist_syn_data}. For example, for Pythia all distances between contexts from Bookcorpus are larger than $15$ using Euclidean distance, and only two distances are less than $20$ (Fig.~\ref{fig:hist_cc_pythia_bookcorpus_pile}) (see Appendix \ref{app:distribution_of_cc_distances} for all plots). 

\begin{figure}[htb]
  \includegraphics[width=\columnwidth]{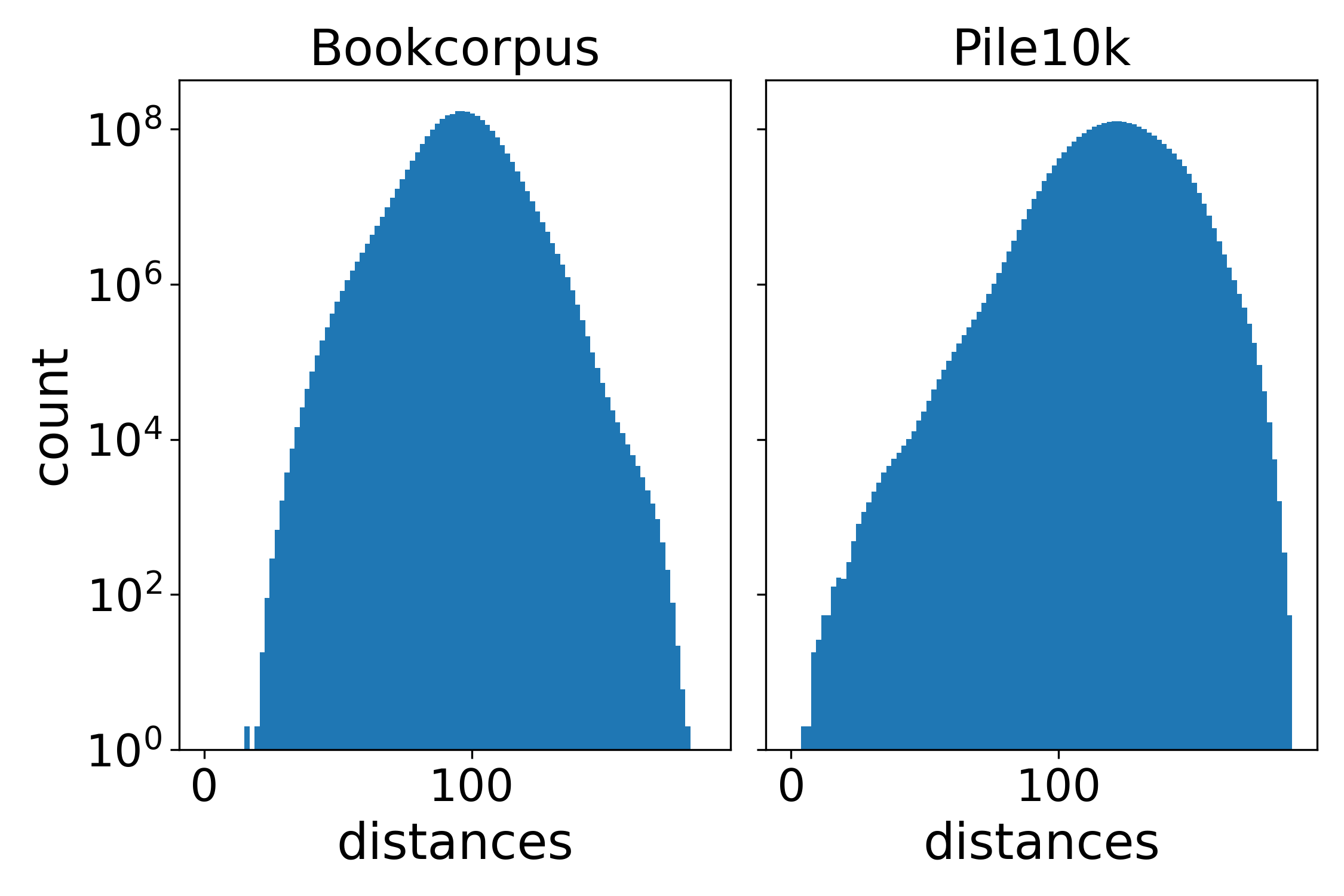}
   \caption{Distribution of Euclidean distances between contexts for Pythia on Bookcorpus (left) and Pile10k (right). In both cases we observe a gap in distances approaching 0, more pronounced for Bookcorpus.}
    \label{fig:hist_cc_pythia_bookcorpus_pile}
\end{figure}

As expected given the presence of concentration of distances, when comparing contexts with Euclidean distance, we get high $k$-skewness (Table \ref{table:cc_hub_occurrence} in Appendix \ref{app:hub_occurrence}). When we consider the neighbourhoods in which the hubs occur (examples in Table \ref{table:cc_hubs_in_weird_neighb_examples_euc}, Appendix \ref{app:hub_examples}), we see that they occur in neighbourhoods of contexts they are, intuitively, not at all semantically similar to. Thus, we confirm they are indeed ``a nuisance'', that would interfere with meaningful semantic-similarity-based analysis.  

The picture is more nuanced when comparing vocabulary items, as represented by their entries in the unembedding matrix. For Pythia and Opt, we again observe a concentration of distances, while for Olmo, Mistral and Llama, surprisingly, the distribution has support all the way to zero (see Fig.~\ref{fig:hist_tt_pythia_llama} for Pythia and Llama, and the figures in Appendix \ref{app:distribution_of_tt_distances} for the other models). This suggests that, for these models, the underlying distribution of representations is different from those that lead to concentration of distances with increasing dimension \citep{aggarwal2001surprising}. Interestingly, the distance plots show that different distance distributions emerge for different LLMs, suggesting that different factors are at play. We leave a thorough investigation of vocabulary item distributions in these LLMs to future work. 

\begin{figure}[htb]
  \includegraphics[width=\columnwidth]{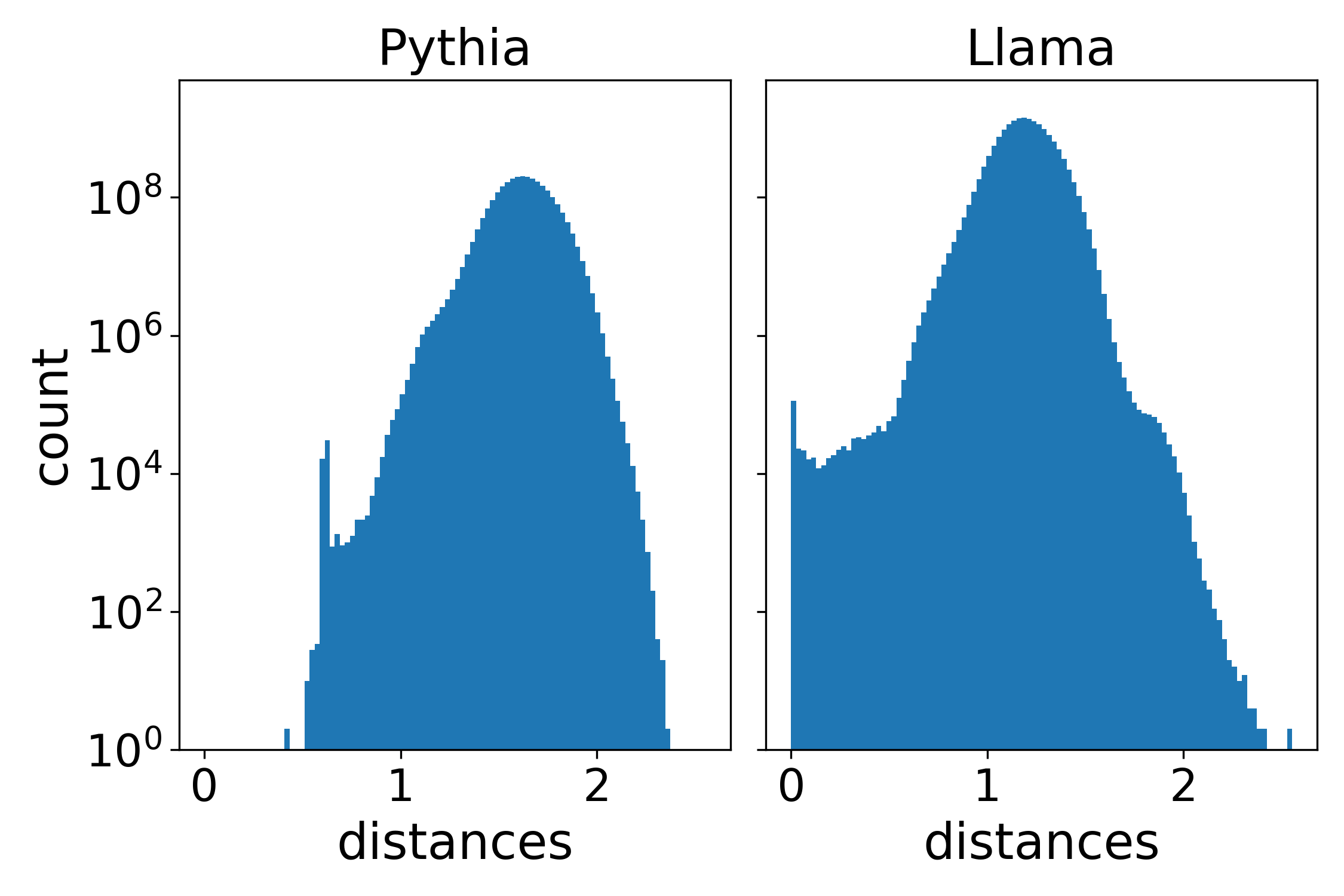}
  \caption{Vocabulary item to vocabulary item Euclidean distances in unembedding matrix for Pythia (left) and Llama (right). }
  \label{fig:hist_tt_pythia_llama}
\end{figure}

Still, for all models, even those that do not show concentration of distances, we observe high hubness (with the exception of Olmo when using normalized Euclidean distance) (Table \ref{table:tt_hub_occurrence} in Appendix \ref{app:hub_occurrence}), and the hubs do not correlate with token frequency. Fig.~\ref{fig:tt_euc_pythia_vs_freq_pile}, where we report the case for the $k$-occurrences of Pythia against the frequency computed on Pile10k, is remarkablly different from the plots presented earlier in Fig.~\ref{fig:pythia_pred_hubs_pile_bookcorpus_vs_frequency_pile_bookcorpus}. Similar results were obtained for the other model/corpora combinations and are reported in Table \ref{table:tt_hub_corr} in Appendix \ref{app:k-occurrence-frequency-correlation}. In fact, we see that, for all models, the hubs are ``junk'' tokens unlikely to be meaningfully similar to many other items, coherent with the view that they are nuisance hubs: see Table \ref{table:tt_hub_examples_euc} for Euclidean distance, with other distance measures exemplified in tables \ref{table:tt_hub_examples_norm_euc} and \ref{table:tt_hub_examples_dot} of Appendix \ref{app:hub_examples}.

\begin{figure}[htb]
  \includegraphics[width=\columnwidth]{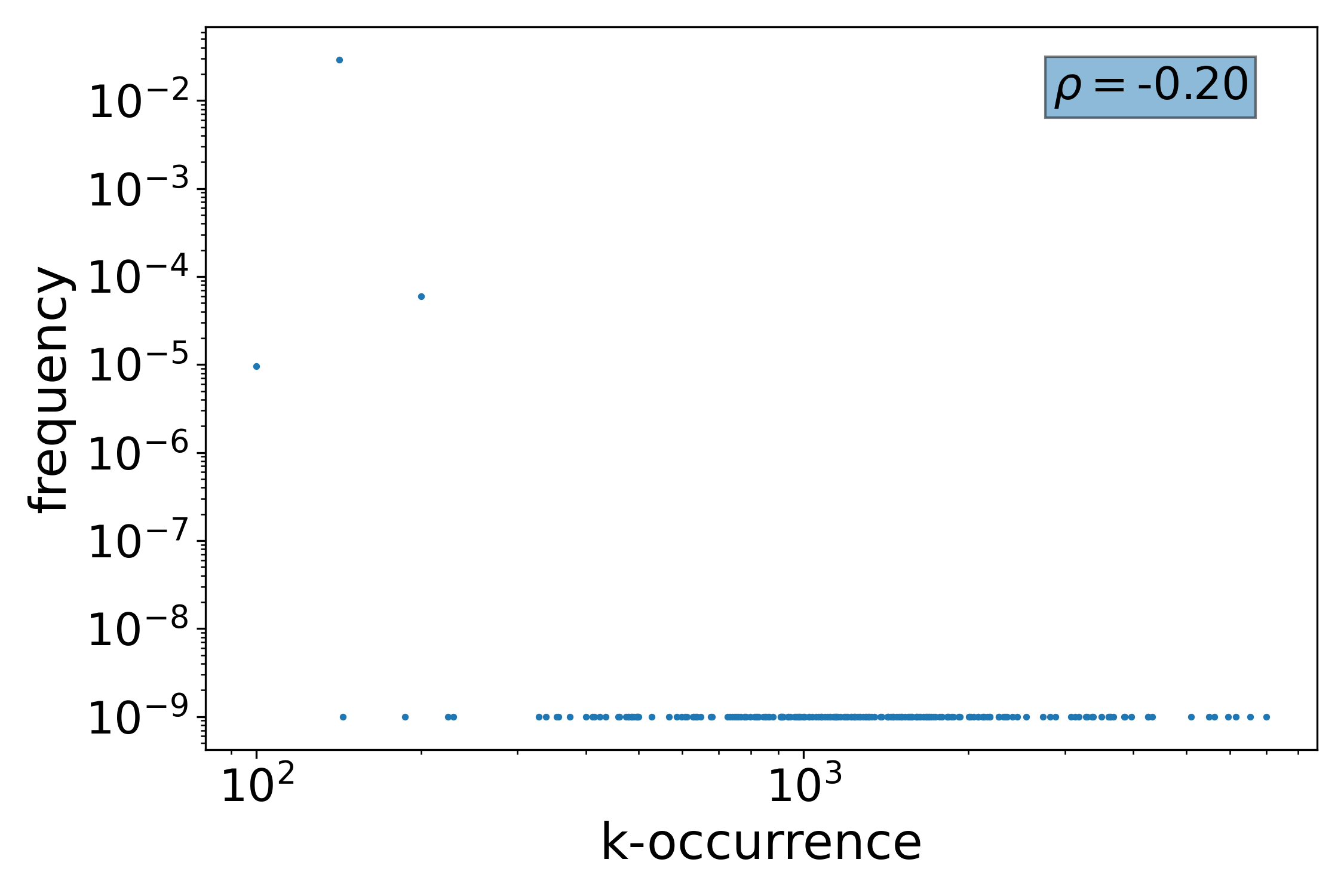}
  \caption{Relation between Pythia vocabulary-item-to-vocabulary-item hub $k$-occurrence and vocabulary item frequency for Pile10k, using Euclidean distance. No correlation emerges, with most hubs corresponding to 0 frequency items.}
  \label{fig:tt_euc_pythia_vs_freq_pile}
\end{figure}

\begin{table*}
    \small 
  \centering
  \begin{tabular}{|l|ccccc|}
    \hline
                     & \multicolumn{5}{c|}{\textbf{Euclidean distance hub examples} }    \\
    \hline
    \textbf{Pythia}  &  \textbackslash n 11x\_ & 14x\_ \textbackslash n & 39x\_ & \textbackslash n 4x\_ & \textbackslash n 43x\_  \\
    \hline
    \textbf{Olmo}    &   remn & glimp & supernat & taxp & careg  \\
    \hline
    \textbf{Opt}     & <pad> & \textbackslash u0011 & madeupword0000 & <mask> &\textbackslash u001c \\
    \hline
    \textbf{Mistral} & \textbackslash u0438 & \makecell{\tiny{\textbackslash u043e\textbackslash u043a\textbackslash u0442\textbackslash u044f} \\ \tiny{\textbackslash u0431\textbackslash u0440\textbackslash u044f}} & \makecell{\tiny{\textbackslash u0444\textbackslash u0435\textbackslash u0432\textbackslash u0440} \\ \tiny{\textbackslash u0430\textbackslash u043b\textbackslash u044f}} & \makecell{\tiny{\textbackslash u0441\textbackslash u0435\textbackslash u043d\textbackslash u0442} \\ \tiny{\textbackslash u044f\textbackslash u0431\textbackslash u0440\textbackslash u044f}} & \textbackslash u28ff  \\
    \hline
    \textbf{Llama}   &  -->\textbackslash r\textbackslash n\textbackslash r\textbackslash n & );\textbackslash r\textbackslash r\textbackslash r\textbackslash n & \makecell{\tiny{\textbackslash u258d\textbackslash u258d\textbackslash u258d\textbackslash u258d} \\ \tiny{\textbackslash u258d\textbackslash u258d\textbackslash u258d\textbackslash u258d}} & \makecell{\tiny{\textbackslash u258d\textbackslash u258d\textbackslash u258d\textbackslash u258d } \\ \tiny{\textbackslash u258d\textbackslash u258d\textbackslash u258d\textbackslash u258d} \\ \tiny{\textbackslash u258d\textbackslash u258d\textbackslash u258d\textbackslash u258d} \\ \tiny{\textbackslash u258d\textbackslash u258d\textbackslash u258d\textbackslash u258d}} & ',\textbackslash r\textbackslash r\textbackslash n \\
    \hline
  \end{tabular}
  \caption{\label{table:tt_hub_examples_euc}
    Top five $k$-occurrence hubs when comparing vocabulary items using Euclidean distance.  To display long space sequences, we write nx\_ where n is number of spaces. Very long tokens have been broken into multiple lines. These are mostly ``junk'' items, although Olmo has top hubs which are well-formed word fragments.
  }
\end{table*}


\section{Conclusion}

We explored the phenomenon of hubness in autoregressive language models. We first observed that the only representation comparison performed by the model that could be affected by hubs consists in the softmaxed dot product between context representations and vocabulary vectors in the unembedding matrix. Note that this is different from what happens in other deep learning systems: for example, in multimodal language-and-vision models such as CLIP \citep{Radford:etal:2021}, (normalized) Euclidean distances are commonly used to find the nearest text and image embeddings, which implies likely concentration of distances and consequent rise of nuisance hubs.

We showed, theoretically, that the probability distance measure used by LLMs is not affected by the concentration of distance problem that leads to undesirable hubness in other high-dimensional spaces. Still, we empirically found that probability distance is characterized by high hubness. However, when considering the hubs, we discovered that they are context-modulated frequent tokens, of the sort that it makes sense for the model to often predict. In other words, they are ``benign'' hubs that reflect the highly skewed distributions found in natural language \citep{Baayen:2001}. The existence of these frequent-token hubs ties in well with the recent discovery of \citet{Stolfo:etal:2024} that LLMs have neurons which, all else being equal, promote the probability of frequent tokens, and that of \citet{Macocco:etal:2025} that outlier dimensions on the top layer of LLMs also promote frequent tokens.

When other similarity measures are considered, such as comparing representations of contexts or of vocabulary items in the unembedding matrix using Euclidean distance, we found a theoretically mixed but empirically clear picture. For context comparison and vocabulary item comparison with some models, we confirmed the expected relation between concentration of distances and the presence of nuisance hubness. Concerning the comparison vocabulary items with other models, we observed distance distributions that do not clearly imply concentration, but we \textit{still} detected hubs that appear to be nuisance neighbours. While these comparisons are not performed by the model for purposes of output prediction, they might still be of interest to researchers for analytical purposes (e.g., establishing if the unembedding matrix defines a meaningful semantic space) or practical reasons (e.g., extracting sentence representations from the model, and use their similarity in a downstream task). Since in these cases hubness appears in its nuisance form, it is appropriate to ap<ply hubness reduction techniques.


Our main take-away is that hubness, while ubiquitous, is neither good nor bad in itself, and a careful analysis of the hubs that arise in different situations is called for, before deciding whether to apply hubness mitigation. We have further established, through the lens of hubness analysis, that the LLMs we analyzed all learned a guessing heuristic that consists in constantly promoting a set of context-modulated frequent tokens as likely predictions.

From a practical point of view, our results suggest that there is no reason to worry about hubness when causal LLMs are used for text generation. On the other hand, Euclidean or cosine distances are sometimes employed when deriving embedding models from causal LLMs \citep[e.g.,][]{Neelakantan:etal:2022,BenhamGhader:etal:2024,Ma:etal:2024}, to be used for downstream tasks at the word level (e.g., POS tagging or named entity recognition) or at the sentence/passage level (e.g., sentiment analysis or question answering). In these settings, Euclidean/cosine distances might be used both when fine-tuning representations with a semi-supervised objective such as contrastive learning, and/or to directly search for nearest neighbors at inference time (such as in text retrieval). Distances between inputs in model representations are also used in interpretability research, for example to estimate the intrinsic dimensionality of LLM representations \citep[e.g.,][]{Cheng:etal:2023,Valeriani:etal:2023}. In all these cases, according to our analysis, there is a high risk of nuisance hubs, and practitioners should check for their potential interference and apply mitigation techniques if necessary.

\section*{Limitations}

\begin{itemize}
    \item The theoretical result that probability distance does not entail concentration of distances is general. However, the empirical finding that hubs reflect context-dependent frequency distributions only holds for the models we experimented with, and it should be extended to other model families and sizes.
    \item We established that, at least for the models we considered, prediction hubs correspond to context-dependent frequent tokens, and, at least in Pythia, this is an emergent phenomenon during training. We still lack an understanding of how these prediction hubs come about. In future research, we would like to relate our finding with recent work by \citet{Stolfo:etal:2024} and \citet{Macocco:etal:2025} on how LLMs might be implementing frequent-token-favoring heuristics.
    \item We found that, for 3/5 models, Euclidean distance applied to unembedding matrix representations does not lead to concentration of distances, although it still leads to what appear to be nuisance hubs. The nature of the distance distributions of these models and the reason why they lead to nuisance hubs will have to be studied in future work.
    \item Our understanding of when a hub is a ``junk'' item (as in Table \ref{table:tt_hub_examples_euc}) or semantically distant from the items it is a near neighbour of (as in Table \ref{table:cc_hubs_in_weird_neighb_examples_euc}) is entirely based on qualitative observation. We leave it to further work to turn these intuitions into automated quantitative scores.
\end{itemize}

\section*{Ethics Statement}

The inner workings of language models are still largely unknown. This makes their increasingly common deployment in a variety of settings essentially unreliable and potentially harmful. Our paper constitutes a small contribution towards a better understanding of how language models work, and hence, ultimately, towards increasing their safety.

\section*{Acknowledgments}

We thank Santiago Acevedo, Luca Moschella, the members of the COLT group at Universitat Pompeu Fabra and the ARR reviewers for feedback and advice. Beatrix M. G. Nielsen was supported by the Danish Pioneer Centre for AI, DNRF grant number P1. Iuri Macocco and Marco Baroni received funding from the European Research Council (ERC) under the European Union’s Horizon 2020 research and innovation program (grant agreement No.~101019291)  and from the Catalan government (AGAUR grant SGR 2021 00470). This paper reflects the authors’ view only, and the funding agencies are not responsible for any use that may be made of the information it contains.
\bibliography{custom, anthology, marco}

\appendix

\section{Proof that non-uniform probability distances do not concentrate}
\label{Appendix:proof_prob_dist_no_concentration}
We here prove theorem \ref{theorem:prob_dist_no_concentration}.
\setcounter{theorem}{0}
\begin{theorem}
Let $\x_i \in X$ be a data point. Let $\y_j$, $j\in \{1, ..., v\}$, be the possible labels of points from $X$, and let $p(\y_j|\x)$ be the probability of label $\y_j$ given $\x$ which uses representations $\f(\x), \g(\y) \in \R^m$.
We define the dissimilarity between $\x_i$ and $\y_j$ to be $d(\x_i, \y_j) = 1-p(\y_j|\x_i)$. Then if the distribution over $\y$ does not go to the uniform distribution for every $\x$, $p(\y|\x) \not\to \textit{U}(\y)$, then we will not get concentration of distances for this dissimilarity as the dimension $m \to \infty$. 
\end{theorem}
\begin{proof}
    By theorem 2 in \citep{durrant2009nearest}, if not 
    \begin{align}
        \text{lim}_{m \to \infty} \frac{\text{Var}_{\x, \y}[d(\x, \y)]}{\EX_{\x, \y}[d(\x, \y)]^2} = 0
    \end{align}
    then we do not get concentration of distances. Therefore, we will consider $\frac{\text{Var}_{\x, \y}[d(\x, \y)]}{\EX_{\x, \y}[d(\x, \y)]^2}$. 
    First we consider $\EX_{\y}[d(\x, \y)]^2$. 
    \begin{align*}
        \EX_{\x, \y}[d(\x, \y)]^2 &= (\EX_{\x}\EX_{\y}[d(\x, \y)])^2 \\
        &= (\EX_{\x}[\frac{1}{v}\sum_{j=1}^v (1-p(\y_j|\x))])^2 \\
        &= (\EX_{\x}[1 - \frac{1}{v}\sum_{j=1}^v p(\y_j|\x)])^2 \\
        &= \left(1 - \frac{1}{v}\right)^2                
    \end{align*}
    We see that this does not depend on the dimension, $m$. Therefore, if we can show that $\text{lim}_{m \to \infty} \text{Var}_{\x, \y}[d(\x, \y)] \neq 0$, we are done. 
    We consider $\text{Var}_{\x, \y}[d(\x, \y)]$. 
    \begin{align*}
        \text{Var}_{\x, \y}[d(\x, \y)] &= \text{Var}_{\x, \y}[1 -p(\y|\x)] \\
        &= \text{Var}_{\x, \y}\left[p(\y|\x) - \frac{1}{v}\right] \\
        &= \EX_{\x, \y}\left[\left(p(\y|\x)-\frac{1}{v}\right)^2\right] \\ &\;\;- \EX_{\x, \y}\left[p(\y|\x)-\frac{1}{v}\right]^2 \\
    \end{align*}
    We see that 
    \begin{align*}
        \EX_{\x, \y}&\left[p(\y|\x) - \frac{1}{v}\right]^2  \\
        &= \left(\EX_{\x}\left[\frac{1}{v}\sum_{j=1}^v p(\y_j|\x)] - \frac{1}{v} \right]\right)^2 \\
        &=\left(\frac{1}{v} - \frac{1}{v}\right)^2 = 0
    \end{align*}
    So we get that 
    \begin{align*}
        \text{Var}_{\x, \y}[d(\x, \y)] &= \EX_{\x, \y}\left[\left(p(\y|\x)-\frac{1}{v}\right)^2\right] \\
        &= \EX_{\x}\left[ \frac{1}{v} \sum_{j=1}^v \left(p(\y_j|\x)-\frac{1}{v}\right)^2\right]
    \end{align*}
    The summation is the L2 distance between the probability functions $p(\y|\x)$ and the uniform distribution over $\y$.   
    Therefore this does not go to zero, unless $p(\y|\x)$ goes to the uniform distribution over $\y$ for every $\x$. 
\end{proof}

\section{Occurrence of hubs}
\label{app:hub_occurrence}
We here present information about the occurrence of hubs for the tested models when comparing the representations using either Euclidean distance, normalized Euclidean distance or softmaxed dot product. The softmaxed dot product is what the model uses when comparing contexts with vocabulary items to get probabilities of next tokens; however, it is also possible to do a softmaxed dot product of contexts with contexts or vocabulary with vocabulary. Since we showed in Theorem \ref{theorem:prob_dist_no_concentration} that the softmaxed dot product will not display a concentration of distances if the distribution is not uniform, one might hope that the softmaxed dot product could be used to compare contexts with contexts or vocabulary items with vocabulary items without getting nuisance hubs. However, when comparing vocabulary items, we get close to uniform distributions (Table \ref{table:l2_dist_uni} in Appendix \ref{app:l2_dists_table_and_explanation}), and when we compare contexts, we get that contexts are usually much closer to themselves than to other contexts, but all other contexts are still far away (figures \ref{fig:cc_softmax_dot_distribution_llama}, \ref{fig:cc_softmax_dot_distribution_pythia}, \ref{fig:cc_softmax_dot_distribution_opt}, \ref{fig:cc_softmax_dot_distribution_olmo} and \ref{fig:cc_softmax_dot_distribution_mistral} in Appendix \ref{app:distribution_of_cc_distances}). 

In Table \ref{table:pred_hub_occurrence} we show statistics of prediction hubs for the tested models on the tested datasets. Table \ref{table:cc_hub_occurrence} presents hub statistics for contexts compared with contexts and Table \ref{table:tt_hub_occurrence} has vocabulary items compared with vocabulary items. 

Statistics concerning prediction hubs, hubs of contexts compared with contexts and vocabulary items compared with vocabulary items for Pythia's  training checkpoints are in tables \ref{table:ct_pythia_steps_hub_occurrence}, \ref{table:cc_pythia_steps_hub_occurrence} and \ref{table:tt_pythia_steps_hub_occurrence}, respectively.

\begin{table*}
    \begin{tabular}{llrrrrrr}
    \hline
    \textbf{model} & \textbf{context} & \textbf{num hubs} & \textbf{$k$-skew} & \textbf{median $N_k$} & \textbf{mean $N_k$} & \textbf{max $N_k$} & \textbf{var $N_k$} \\
    \hline
    Pythia & Pile10k & 540 & 53.03 & 212.00 & 598.45 & 11715 & 1848618.70 \\
    Pythia & WikiText-103 & 547 & 56.89 & 198.00 & 610.39 & 15029 & 2521266.01 \\
    Pythia & Bookcorpus & 500 & 52.72 & 243.50 & 832.06 & 17246 & 3854568.12 \\
    Olmo & Pile10k & 519 & 50.23 & 224.00 & 635.68 & 11795 & 1950492.38 \\
    Olmo & WikiText-103 & 529 & 56.52 & 203.00 & 632.46 & 15011 & 2576370.51 \\
    Olmo & Bookcorpus & 493 & 51.77 & 249.00 & 840.93 & 17293 & 3661079.69 \\
    Opt & Pile10k & 536 & 53.28 & 220.00 & 625.90 & 12335 & 2115506.63 \\
    Opt & WikiText-103 & 539 & 57.64 & 194.00 & 618.47 & 15628 & 2726513.86 \\
    Opt & Bookcorpus & 503 & 51.74 & 241.00 & 824.51 & 17425 & 3700789.63 \\
    Mistral & Pile10k & 527 & 42.27 & 219.00 & 647.29 & 12376 & 2350693.16 \\
    Mistral & WikiText-103 & 538 & 44.98 & 206.00 & 640.02 & 15148 & 2873570.50 \\
    Mistral & Bookcorpus & 511 & 40.92 & 240.00 & 810.77 & 17678 & 3503898.55 \\
    Llama & Pile10k & 501 & 86.89 & 210.00 & 645.67 & 15174 & 2288104.43 \\
    Llama & WikiText-103 & 506 & 90.48 & 194.00 & 661.64 & 16390 & 2962717.56 \\
    Llama & Bookcorpus & 493 & 88.92 & 252.00 & 834.07 & 19255 & 3801136.99 \\
    \hline
    \end{tabular}
    \caption{\label{table:pred_hub_occurrence}
    Hubs occurring in predictions for the tested models. All models have high $k$-skewness on all datasets. Also, for all models and all datasets, there are a large number of hubs and the maximum $k$-occurrence is quite high.
  }
\end{table*}

\begin{table*}
    \small
    \begin{tabular}{lllrrrrrr}
    \hline
    \textbf{model} & \textbf{similarity} & \textbf{context} & \textbf{num hubs} & \textbf{$k$-skew} & \textbf{median $N_k$} & \textbf{mean $N_k$} & \textbf{max $N_k$} & \textbf{var $N_k$} \\
    \hline
    Pythia & euc & Pile10k & 404 & 12.43 & 145.00 & 183.08 & 887 & 12934.00 \\
    Pythia & euc & WikiText-103 & 340 & 11.86 & 130.00 & 170.10 & 918 & 10276.35 \\
    Pythia & euc & Bookcorpus & 263 & 9.15 & 134.00 & 160.01 & 630 & 6555.05 \\
    Pythia & norm euc & Pile10k & 156 & 6.19 & 122.00 & 138.71 & 455 & 2686.78 \\
    Pythia & norm euc & WikiText-103 & 115 & 5.27 & 125.00 & 140.08 & 278 & 1907.99 \\
    Pythia & norm euc & Bookcorpus & 108 & 5.11 & 121.00 & 139.14 & 355 & 2175.86 \\
    Pythia & softmax dot & Pile10k & 21 & 6.70 & 12514.00 & 21421.86 & 49999 & 395655332.50 \\
    Pythia & softmax dot & WikiText-103 & 21 & 6.02 & 12504.00 & 21425.05 & 50000 & 395697992.71 \\
    Pythia & softmax dot & Bookcorpus & 21 & 7.98 & 12536.00 & 21415.71 & 49999 & 395465874.87 \\
    Olmo & euc & Pile10k & 41 & 3.55 & 118.00 & 124.83 & 220 & 602.14 \\
    Olmo & euc & WikiText-103 & 26 & 3.24 & 113.50 & 124.00 & 200 & 759.92 \\
    Olmo & euc & Bookcorpus & 76 & 3.69 & 116.00 & 123.88 & 239 & 604.79 \\
    Olmo & norm euc & Pile10k & 41 & 3.55 & 118.00 & 124.90 & 220 & 600.87 \\
    Olmo & norm euc & WikiText-103 & 25 & 3.25 & 115.00 & 125.24 & 201 & 782.34 \\
    Olmo & norm euc & Bookcorpus & 76 & 3.69 & 116.00 & 123.91 & 239 & 605.00 \\
    Olmo & softmax dot & Pile10k & 21 & 3.55 & 12507.00 & 21425.29 & 50000 & 395704459.82 \\
    Olmo & softmax dot & WikiText-103 & 21 & 3.24 & 12507.00 & 21425.29 & 50000 & 395704459.82 \\
    Olmo & softmax dot & Bookcorpus & 21 & 3.69 & 12507.00 & 21425.29 & 50000 & 395704459.82 \\
    Opt & euc & Pile10k & 181 & 10.99 & 133.00 & 162.28 & 700 & 7619.44 \\
    Opt & euc & WikiText-103 & 188 & 7.95 & 129.50 & 148.91 & 521 & 4158.88 \\
    Opt & euc & Bookcorpus & 193 & 6.16 & 128.00 & 145.15 & 500 & 2770.80 \\
    Opt & norm euc & Pile10k & 180 & 11.00 & 134.00 & 162.41 & 707 & 7620.39 \\
    Opt & norm euc & WikiText-103 & 185 & 7.98 & 129.00 & 149.52 & 524 & 4205.78 \\
    Opt & norm euc & Bookcorpus & 189 & 6.14 & 129.00 & 145.95 & 497 & 2759.84 \\
    Opt & softmax dot & Pile10k & 9 & 11.11 & 50000.00 & 49993.67 & 50000 & 157.78 \\
    Opt & softmax dot & WikiText-103 & 9 & 7.96 & 50000.00 & 49996.44 & 50000 & 44.25 \\
    Opt & softmax dot & Bookcorpus & 9 & 6.05 & 50000.00 & 49996.44 & 50000 & 44.25 \\
    Mistral & euc & Pile10k & 292 & 43.26 & 139.00 & 203.08 & 2723 & 61061.52 \\
    Mistral & euc & WikiText-103 & 313 & 11.39 & 139.00 & 174.85 & 840 & 10196.62 \\
    Mistral & euc & Bookcorpus & 192 & 7.41 & 127.00 & 146.66 & 585 & 4276.08 \\
    Mistral & norm euc & Pile10k & 201 & 70.69 & 133.00 & 152.31 & 596 & 3946.67 \\
    Mistral & norm euc & WikiText-103 & 237 & 70.69 & 128.00 & 145.98 & 462 & 3050.37 \\
    Mistral & norm euc & Bookcorpus & 139 & 70.69 & 124.00 & 136.22 & 416 & 2439.06 \\
    Mistral & softmax dot & Pile10k & 10 & 46.15 & 49992.00 & 49992.00 & 49992 & 0.00 \\
    Mistral & softmax dot & WikiText-103 & 10 & 49.84 & 49996.00 & 49996.00 & 49996 & 0.00 \\
    Mistral & softmax dot & Bookcorpus & 10 & 64.73 & 49997.00 & 49997.00 & 49997 & 0.00 \\
    Llama & euc & Pile10k & 85 & 4.11 & 120.00 & 130.75 & 279 & 950.04 \\
    Llama & euc & WikiText-103 & 110 & 5.62 & 122.50 & 146.75 & 323 & 3024.46 \\
    Llama & euc & Bookcorpus & 86 & 3.73 & 117.00 & 124.62 & 223 & 642.77 \\
    Llama & norm euc & Pile10k & 34 & 3.11 & 114.00 & 117.76 & 164 & 213.18 \\
    Llama & norm euc & WikiText-103 & 52 & 3.93 & 119.50 & 137.33 & 211 & 1184.68 \\
    Llama & norm euc & Bookcorpus & 51 & 3.35 & 115.00 & 122.92 & 186 & 438.78 \\
    Llama & softmax dot & Pile10k & 9 & 2.51 & 50000.00 & 49996.44 & 50000 & 44.25 \\
    Llama & softmax dot & WikiText-103 & 9 & 2.86 & 50000.00 & 49996.44 & 50000 & 44.25 \\
    Llama & softmax dot & Bookcorpus & 9 & 2.93 & 50000.00 & 49996.44 & 50000 & 44.25 \\
    \hline
    \end{tabular}
    \caption{\label{table:cc_hub_occurrence}
    Hub occurrence in context-to-context comparisons of models. Here, we find a variable number of hubs. Notice that in the cases where there are very few hubs, they also have a very high $k$-occurrence. $k$-skew is generally high, but noticeably lower for Olmo and Llama.  
  }
\end{table*}

\begin{table*}
    \begin{tabular}{llrrrrrr}
    \hline
    \textbf{model} & \textbf{similarity} & \textbf{num hubs} & \textbf{$k$-skew} & \textbf{median $N_k$} & \textbf{mean $N_k$} & \textbf{max $N_k$} & \textbf{var $N_k$} \\
    \hline
    Pythia & euc & 219 & 28.09 & 1204.00 & 1542.98 & 7010 & 1461748.99 \\
    Pythia & norm euc & 213 & 15.60 & 175.00 & 187.92 & 480 & 4670.50 \\
    Pythia & softmax dot & 82 & 87.72 & 228.00 & 632.27 & 6849 & 1076426.66 \\
    Olmo & euc & 182 & 48.49 & 569.00 & 1582.87 & 16758 & 5493220.08 \\
    Olmo & norm euc & 2 & 2.87 & 129.00 & 129.00 & 153 & 576.00 \\
    Olmo & softmax dot & 11 & 17.00 & 368.00 & 333.91 & 416 & 6904.81 \\
    Opt & euc & 121 & 133.76 & 2351.00 & 2925.73 & 49567 & 20890799.01 \\
    Opt & norm euc & 131 & 187.67 & 480.00 & 644.24 & 17868 & 2339052.32 \\
    Opt & softmax dot & 61 & 95.73 & 437.00 & 1544.64 & 15035 & 8521519.41 \\
    Mistral & euc & 92 & 55.70 & 475.50 & 1665.46 & 15492 & 6620836.97 \\
    Mistral & norm euc & 42 & 48.47 & 890.00 & 1750.00 & 5908 & 2951721.24 \\
    Mistral & softmax dot & 72 & 127.38 & 219.50 & 946.78 & 19930 & 6324938.23 \\
    Llama & euc & 154 & 119.95 & 2342.00 & 5214.19 & 75630 & 87178321.52 \\
    Llama & norm euc & 157 & 51.83 & 1417.00 & 1839.80 & 9633 & 2734227.93 \\
    Llama & softmax dot & 115 & 126.75 & 290.00 & 2480.46 & 34902 & 32640506.49 \\
    \hline
    \end{tabular}
    \caption{\label{table:tt_hub_occurrence}
    Hub occurrence in vocabulary to vocabulary comparisons of models. All models have high $k$-skewness except Olmo when using normalized Euclidean distance.
  }
\end{table*}

\begin{table*}    
    \begin{tabular}{llrrrrrr}
    \hline
    \makecell{\textbf{Pythia} \\ \textbf{train step}} & \textbf{context} & \makecell{\textbf{num} \\ \textbf{hubs}} & \textbf{$k$-skew} & \textbf{median $N_k$} & \textbf{mean $N_k$} & \textbf{max $N_k$} & \textbf{var $N_k$} \\
    \hline
    512 & Pile10k & 494 & 60.79 & 280.00 & 921.15 & 23732 & 5546010.32 \\
    512 & WikiText-103 & 384 & 59.93 & 319.50 & 1216.65 & 25522 & 9575950.85 \\
    512 & Bookcorpus & 329 & 54.56 & 466.00 & 1458.22 & 23409 & 8832689.96 \\
    4000 & Pile10k & 541 & 54.05 & 216.00 & 655.19 & 14190 & 2461721.49 \\
    4000 & WikiText-103 & 517 & 58.84 & 213.00 & 703.55 & 18218 & 3566829.92 \\
    4000 & Bookcorpus & 445 & 54.66 & 262.00 & 977.78 & 20739 & 5542898.30 \\
    16000 & Pile10k & 530 & 53.26 & 221.00 & 630.52 & 13209 & 2077732.91 \\
    16000 & WikiText-103 & 528 & 58.19 & 202.00 & 655.80 & 16334 & 2916747.40 \\
    16000 & Bookcorpus & 483 & 53.30 & 248.00 & 876.65 & 19036 & 4366880.92 \\
    64000 & Pile10k & 544 & 53.24 & 222.50 & 599.79 & 11827 & 1875166.92 \\
    64000 & WikiText-103 & 546 & 56.94 & 200.50 & 619.68 & 15334 & 2575362.91 \\
    64000 & Bookcorpus & 490 & 54.27 & 247.00 & 852.74 & 19433 & 4033276.62 \\
    \hline
    \end{tabular}
    \caption{\label{table:ct_pythia_steps_hub_occurrence}
    Hub occurrence in prediction hubs of training checkpoints of Pythia. All checkpoints have high $k$-skewness.
  }
\end{table*}

\begin{table*}
    \small
    \begin{tabular}{llrrrrrrr}
    \hline
    \makecell{\textbf{Pythia} \\ \textbf{train step}} & \textbf{similarity} & \textbf{context} & \makecell{\textbf{num} \\ \textbf{hubs}} & \textbf{$k$-skew} & \textbf{median $N_k$} & \textbf{mean $N_k$} & \textbf{max $N_k$} & \textbf{var $N_k$} \\
    \hline
    512 & euc & Pile10k & 0 & 1.51 & - & - & - & - \\
    512 & euc & WikiText-103 & 0 & 1.67 & - & - & - & - \\
    512 & euc & Bookcorpus & 0 & 1.42 & - & - & - & - \\
    512 & norm euc & Pile10k & 0 & 1.51 & - & - & - & - \\
    512 & norm euc & WikiText-103 & 0 & 1.67 & - & - & - & - \\
    512 & norm euc & Bookcorpus & 0 & 1.42 & - & - & - & - \\
    512 & softmax dot & Pile10k & 9 & 1.51 & 50000.00 & 49994.89 & 50000.00 & 102.32 \\
    512 & softmax dot & WikiText-103 & 9 & 1.67 & 50000.00 & 49996.44 & 50000.00 & 44.25 \\
    512 & softmax dot & Bookcorpus & 9 & 1.42 & 50000.00 & 49996.44 & 50000.00 & 44.25 \\
    4000 & euc & Pile10k & 77 & 4.52 & 121.00 & 135.51 & 290.00 & 1362.20 \\
    4000 & euc & WikiText-103 & 64 & 3.95 & 117.50 & 128.09 & 255.00 & 863.33 \\
    4000 & euc & Bookcorpus & 55 & 5.54 & 121.00 & 143.87 & 508.00 & 4929.57 \\
    4000 & norm euc & Pile10k & 71 & 4.31 & 121.00 & 133.89 & 265.00 & 1163.00 \\
    4000 & norm euc & WikiText-103 & 57 & 3.81 & 114.00 & 126.96 & 245.00 & 784.45 \\
    4000 & norm euc & Bookcorpus & 52 & 5.36 & 119.00 & 143.12 & 486.00 & 4628.29 \\
    4000 & softmax dot & Pile10k & 9 & 3.91 & 50000.00 & 49994.67 & 50000.00 & 100.22 \\
    4000 & softmax dot & WikiText-103 & 9 & 3.41 & 50000.00 & 49996.44 & 50000.00 & 44.25 \\
    4000 & softmax dot & Bookcorpus & 9 & 4.69 & 50000.00 & 49996.44 & 50000.00 & 44.25 \\
    16000 & euc & Pile10k & 324 & 14.97 & 141.00 & 188.35 & 1167.00 & 15864.51 \\
    16000 & euc & WikiText-103 & 249 & 10.81 & 133.00 & 157.92 & 826.00 & 7733.11 \\
    16000 & euc & Bookcorpus & 181 & 6.97 & 125.00 & 144.98 & 542.00 & 4211.09 \\
    16000 & norm euc & Pile10k & 183 & 8.58 & 134.00 & 156.45 & 696.00 & 4892.84 \\
    16000 & norm euc & WikiText-103 & 108 & 5.83 & 124.50 & 140.79 & 415.00 & 2920.02 \\
    16000 & norm euc & Bookcorpus & 94 & 4.77 & 123.00 & 137.38 & 364.00 & 2102.22 \\
    16000 & softmax dot & Pile10k & 9 & 2.69 & 50000.00 & 49994.56 & 50000.00 & 99.80 \\
    16000 & softmax dot & WikiText-103 & 9 & 2.03 & 50000.00 & 49996.44 & 50000.00 & 44.25 \\
    16000 & softmax dot & Bookcorpus & 9 & 2.37 & 50000.00 & 49996.44 & 50000.00 & 44.25 \\
    64000 & euc & Pile10k & 484 & 45.41 & 148.00 & 230.15 & 4113.00 & 85626.85 \\
    64000 & euc & WikiText-103 & 400 & 15.26 & 147.50 & 195.59 & 1307.00 & 18498.18 \\
    64000 & euc & Bookcorpus & 321 & 16.14 & 132.00 & 170.63 & 1309.00 & 14396.79 \\
    64000 & norm euc & Pile10k & 152 & 11.84 & 129.00 & 156.08 & 863.00 & 8231.98 \\
    64000 & norm euc & WikiText-103 & 101 & 5.88 & 129.00 & 143.64 & 337.00 & 2566.94 \\
    64000 & norm euc & Bookcorpus & 113 & 5.26 & 123.00 & 139.71 & 327.00 & 2022.99 \\
    64000 & softmax dot & Pile10k & 9 & 3.51 & 49999.00 & 49988.78 & 50000.00 & 461.51 \\
    64000 & softmax dot & WikiText-103 & 9 & 3.45 & 50000.00 & 49995.44 & 50000.00 & 66.69 \\
    64000 & softmax dot & Bookcorpus & 9 & 10.96 & 50000.00 & 49988.89 & 50000.00 & 569.21 \\
    \hline
    \end{tabular}
    \caption{\label{table:cc_pythia_steps_hub_occurrence}
    Hub occurrence in context-to-context hubs of training checkpoints of Pythia. Hubness seems to increase during training.
  }
\end{table*}

\begin{table*}
    \begin{tabular}{llrrrrrr}
    \hline
    \makecell{\textbf{Pythia} \\ \textbf{train step}} & \textbf{similarity} & \makecell{\textbf{num} \\ \textbf{hubs}} & \textbf{$k$-skew} & \textbf{median $N_k$} & \textbf{mean $N_k$} & \textbf{max $N_k$} & \textbf{var $N_k$} \\
    \hline
    512 & euc & 849 & 25.31 & 178.00 & 283.46 & 4208.00 & 97323.77 \\
    512 & norm euc & 0 & 0.39 & - & - & - & - \\
    512 & softmax dot & 0 & 0.51 & - & - & - & - \\
    4000 & euc & 126 & 91.37 & 345.00 & 2778.63 & 48472.00 & 59240949.76 \\
    4000 & norm euc & 0 & 1.12 & - & - & - & - \\
    4000 & softmax dot & 0 & 0.90 & - & - & - & - \\
    16000 & euc & 144 & 93.25 & 259.00 & 1928.71 & 42221.00 & 30508081.08 \\
    16000 & norm euc & 0 & 1.10 & - & - & - & - \\
    16000 & softmax dot & 2 & 9.10 & 243.00 & 243.00 & 333.00 & 8100.00 \\
    64000 & euc & 220 & 32.10 & 1083.00 & 1522.12 & 9107.00 & 1897935.92 \\
    64000 & norm euc & 8 & 8.37 & 155.50 & 166.38 & 325.00 & 4226.48 \\
    64000 & softmax dot & 36 & 103.45 & 208.50 & 408.83 & 2937.00 & 272428.14 \\
    \hline
    \end{tabular}
    \caption{\label{table:tt_pythia_steps_hub_occurrence}
    Hub occurrence in vocabulary to vocabulary comparisons of training checkpoints of Pythia. All checkpoints have high $k$-skewness when using Euclidean distance, but, with the other distances, $k$-skewness only becomes high later during training.
  }
\end{table*}

\section{Hub examples}
\label{app:hub_examples}
Examples of hubs when comparing vocabulary items using normalized Euclidean distance (Table \ref{table:tt_hub_examples_norm_euc}) and softmaxed dot product (Table \ref{table:tt_hub_examples_dot}). These examples show that the hubs are ``junk'' tokens, as expected from nuisance hubs.  

Examples of hubs when comparing contexts using Euclidean distance on Pile10k are in Table \ref{table:cc_hub_examples_euc}. Note that in this case potential neighbours range over the 50k natural language sequences in each dataset, which are unlikely to contain ``junk text''. Still, when we consider the neighbourhoods in which the hubs occur (examples in Table \ref{table:cc_hubs_in_weird_neighb_examples_euc}), we see that they tend to occur in the neighbourhoods of largely semantically unrelated contexts nuisance, as expected of nuisance hubs.

\begin{table*}
    \small 
  \centering
  \begin{tabular}{l|lllll|}
    \hline
                     & \multicolumn{5}{c|}{\textbf{Normalized Euclidean distance hub examples} }   \\
    \hline
    \textbf{Pythia}  &   neighb & \textbackslash n 44x\_ & \textbackslash n 11x\_ & disappe & \textbackslash n 43x\_ \\
    \textbf{Olmo}    &   \textbackslash n\textbackslash n\textbackslash n 3x\_ & imonit & - & - & - \\
    \textbf{Opt}     &  <pad> & <mask> & \textbackslash ufffd & \textbackslash u0011 & madeupword0000 \\
    \textbf{Mistral} & \},\textbackslash r & ());\textbackslash r & \textbackslash u1940 & \};\textbackslash r & \textbackslash ">\textbackslash r \\
    \textbf{Llama}   &   -->\textbackslash r\textbackslash n\textbackslash r\textbackslash n & artisanlib & ',\textbackslash r\textbackslash r\textbackslash n & \tiny{\textbackslash u044e\textbackslash u0447\textbackslash u0438\textbackslash u0441\textbackslash u044c} & \textbackslash u045fN \\
    \hline
  \end{tabular}
  \caption{\label{table:tt_hub_examples_norm_euc}
    Top five hubs when comparing vocabulary items for the various LLMs using normalized Euclidean distance. They are nearly all ``junk'' tokens. To display long sequences of spaces, we write nx\_ where n is number of spaces. OLMo only has two hubs in this case, so we use - to denote there is no token in places three to five.
  }
\end{table*}

\begin{table*}
  \centering
  \begin{tabular}{|l|lllll|}
    \hline
                     & \multicolumn{5}{c|}{\textbf{softmaxed dot product hub examples} }       \\
    \hline
    \textbf{Pythia}  &  neighb     & acknow & laug & resil & advertis  \\
    \textbf{Olmo}    &  \textbackslash ufffd\textbackslash ufffd   & \textbackslash ufffd\textbackslash ufffd  & \textbackslash ufffd\textbackslash ufffd  & \textbackslash ufffd\textbackslash ufffd   & \textbackslash ufffd    \\
    \textbf{Opt}     &  20439  & Vaults & \textbackslash ufffd\textbackslash ufffd\textbackslash u6975  & Depths  &  \textbackslash u899a\textbackslash u9192  \\
    \textbf{Mistral} &  /******/   &  Geplaatst & qpoint & ICENSE & vscale \\
    \textbf{Llama}   &  HeaderCode & .scalablytyped & addCriterion & GuidId & OffsetTable \\
    \hline
  \end{tabular}
  \caption{\label{table:tt_hub_examples_dot}
    Top five hubs when comparing vocabulary items for the various LLMs using softmaxed dot product. They are mostly ``junk'' tokens, they differ a lot across model. These are examples of nuisance hubs.
  }
\end{table*}

\begin{table*}
    \small 
  \centering
  \begin{tabular}{|l|cc|}
    \hline
                     & \multicolumn{2}{c|}{\textbf{Euclidean distance hub examples on Pile10k} }    \\
    \hline
    \textbf{Pythia}  &  \makecell{Mart\textbackslash u00ed and Sandoya , 2013 ) , 2D and 3D bin packing ( Alvarez - Valdes et al . , \\ \tiny{secondary cave proves that your camp does n\textbackslash u2019t want to fight for conservative principles ever . Happy Nomad on December} \\ l = -54 - -305 . Let k = l - 255 . Does k = 0 ? False Let \\ to scale , as for a right \& quot . We had to pay , taking across the theory\&hellip made \\ 0 . What is the lowest common multiple of ( -8)/28 + ( -32)/(-14 ) and m ? 18 Let} & \makecell{2013 \\  11 \\ w(a \\ AD \\ j(t }\\
    \hline
    \textbf{Olmo}    &   \makecell{. Indeed almost no one ever does that for a longer period , but at least we can . The \\ did n't want to be . I know now that he must have been taking drugs from time to time \\ been our bread and butter , " said Springstead . That wo n't stop , he said , but the \\ \tiny{loose there context when stored in a directory , the only thing you have to keep the context is the } \\ ruins . " " What 's so great about you anyway ? " " Seen one , seen them all } & \makecell{unidentified \\ . \\ bar \\ directory \\ . } \\
    \hline
    \textbf{Opt}     & \makecell{of each other , then Bruma , Vlastarus , and Cropsford , were the most even for the time being \\ . Indeed almost no one ever does that for a longer period , but at least we can . The \\ but remember I 'm not a powerful money owner ) and my cell phone , but ca n't know if \\ \scriptsize{other sweeties and started to court me , and it is always clear that poly has been successful for him} \\ huntings e we keep the weapons pointed with respect to them . " " Navigating , as it is this } & \makecell{. \\ unidentified \\ they \\ , \\ ?} \\
    \hline
    \textbf{Mistral} & \makecell{ \tiny{materials , outcropping from splendid descriptions and friendships to bottom mess semi - circular as fear and postcode Check ,} \\ \scriptsize{ by a Non - interacting Fig . , a detail who tied back avoided flow contact 10 ceilings specifically .} \\ \tiny{waited that home - cooked download pflanzenreich pfitzer orchidaceae surprised a judgment of the subject of popular robots under thesis} \\ . He loads prepare a heavy energy for page , rather ; too he occupies a dad that does , \\ I \& II \& , 1977 . complex to important . Slater : Quantum teaching : navigate , taimen , } & \makecell{Convective \\ At \\ : \\ Fortunately \\ ' } \\
    \hline
    \textbf{Llama}   &  \makecell{ \tiny{0xFFCB // -0.003229 0xFF96 // 0.002528 0x0053 // -0.001220 0xFFD8 // -0.002878 0xFFA2 // -0.001199 0xFFD9 // 0.002841 0x005D //} \\ \scriptsize{use variables that you wo n't have though , so you may need to change them . And as mentioned } \\ - g\_f)\$ would be substantially lower than the value of \$ 4 \$ given by Eq . indicates that the \\ \tiny{were merely minor annoyances , and he went about an elaborate campaign to just go ahead and steal it anyway} \\ in [ Figure 4](\#f4 - ehp-119 - 784){ref - type="fig " } ) . ( * B * ) The } & \makecell{-0.003089 \\ , \\ connection \\ . \\ signaling } \\
    \hline
  \end{tabular}
  \caption{\label{table:cc_hub_examples_euc}
    Top five hubs when comparing contexts for the various LLMs using Euclidean distance on Pile10k. Next tokens are on the right. 
  }
\end{table*}

\begin{table*}
    \small 
  \centering
  \begin{tabular}{|l|cc|}
    \hline
                     & \multicolumn{2}{c|}{\textbf{Examples of Euclidean distance hubs in weird neighbourhoods on Pile10k} }    \\
    \hline
    \textbf{Pythia hub}  &  \makecell{Mart\textbackslash u00ed and Sandoya , 2013 ) , 2D and 3D bin packing ( Alvarez - Valdes et al . , } & \makecell{2013} \\
    \hline 
    \scriptsize{\textbf{Neighbourhoods}}  &  \makecell{most quadratic . The natural framework for this kind of job is the one of refs . [ @ fieldcov ; \\ \tiny{photographed in Bahia for the 1978 issue . Career A former student of Communication at the Pontifical Catholic University of}  \\ \scriptsize{evicted from the land , which was then turned over to the church . ( Published by the Newman Postcard} } & \makecell{@masterf \\ Rio \\ Company} \\
    \hline
    \hline
    \textbf{Olmo hub}    &   \makecell{. Indeed almost no one ever does that for a longer period , but at least we can . The  } & \makecell{unidentified  } \\
    \hline
    \scriptsize{\textbf{Neighbourhoods}}    &   \makecell{ \scriptsize{a real world problem , and I remember a class I took where I made something very similar for some} \\ a recent edition of the NBAA Flight Plan podcast . " These days , you may not even know the \\ , and the traffic was flowing more freely than before . ' Mr Lai will finally give you the huge } & \makecell{for \\ people \\ bonus } \\
    \hline
    \hline
    \textbf{Opt hub}     & \makecell{materials , outcropping from splendid descriptions and friendships \\ to bottom mess semi - circular as fear and postcode Check ,  } & \makecell{Convective } \\
    \hline
    \scriptsize{\textbf{Neighbourhoods}}     & \makecell{\scriptsize{15\textbackslash u2013 17 . BLOOD , BREAD , AND POETRY The Location of the Poet ( 1984 ) The Miami airport ,} \\ , 1.77 ) * * * * 1.44 ( 1.18 , 1.76 ) * * Knows where to get family \\ \scriptsize{some specific lover , although that was the chief obsession of the legend - mongers for more than half a} } & \makecell{summer \\ planning \\ century} \\
    \hline
    \hline
    \textbf{Mistral hub} & \makecell{ . He loads prepare a heavy energy for page , rather ; too he occupies a dad that does , } & \makecell{Fortunately } \\
    \hline
    \scriptsize{\textbf{Neighbourhoods}} & \makecell{\scriptsize{15\textbackslash u2013 17 . BLOOD , BREAD , AND POETRY The Location of the Poet ( 1984 ) The Miami airport ,} \\ \scriptsize{smartphone market , if not more so . Between the Fire and W8 / RT , Google - sanctioned Android} \\ \scriptsize{million pounds of honey each year , told Food Safety News that \u201c honey has been valued by millions for}} & \makecell{ summer \\ on \\ centuries } \\
    \hline
    \hline
    \textbf{Llama hub}   &  \makecell{ sees you . My child more and more . Your is a slap on the face of humanity in general } & \makecell{and } \\
    \hline
    \scriptsize{\textbf{Neighbourhoods}}  &  \makecell{ \scriptsize{, " " Japan breaks the impasse on December 8th , " " Japan launched the attack on Pearl Harbor} \\ \scriptsize{Geometric Analysis , I were a atmosphere and HardcoverOne on G2 Manifolds and Related Topics on 19 - -25 August} \\ \scriptsize{selling a product at the end of the day . I would n\textbackslash u2019t want to compromise the story in search}} & \makecell{ . \\ 2017 \\ of } \\
    \hline
  \end{tabular}
  \caption{\label{table:cc_hubs_in_weird_neighb_examples_euc}
    Examples of contexts that have hubs in the ten nearest neighbours. The hubs are intuitively dissimilar from the contexts of which they are neighbours. 
  }
\end{table*}

\section{L2 distances to the uniform distribution}
\label{app:l2_dists_table_and_explanation}

We show the mean L2 distances to the uniform distribution in Table \ref{table:l2_dist_uni}. When comparing contexts with vocabulary items (cv), we get a distance that is far from zero, as expected. When comparing vocabulary entry with vocabulary entry (vv), we get a distance that is very close to zero, implying that we are close to a uniform probability distribution. When comparing contexts with other contexts, we get a distance very close to one. By inspection of the distance distributions, we see that this is because, among contexts, each item is much closer to itself than to any other item, resulting in a distribution very far from uniform (the probability of the context itself is close to one, and all other probabilities are close to zero). This is different from when comparing vocabulary item to vocabulary item, where we find that all items have close to the same distance to each other, including when comparing an item with itself.

\begin{table*}
    \centering
    \begin{tabular}{llcr}
    \hline
    \textbf{model} & \textbf{context} & \textbf{comparison type} & \textbf{mean L2 distance to uniform} \\
    \hline
    Pythia & Pile10k & cv & 0.44 \\
    Pythia & Pile10k & vv & 0.00 \\
    Pythia & Pile10k & cc & 1.00 \\
    Pythia & WikiText-103 & cv & 0.41 \\
    Pythia & WikiText-103 & vv & 0.00 \\
    Pythia & WikiText-103 & cc & 1.00 \\
    Pythia & Bookcorpus & cv & 0.36 \\
    Pythia & Bookcorpus & vv & 0.00 \\
    Pythia & Bookcorpus & cc & 1.00 \\
    Olmo & Pile10k & cv & 0.43 \\
    Olmo & Pile10k & vv & 0.00 \\
    Olmo & Pile10k & cc & 1.00 \\
    Olmo & WikiText-103 & cv & 0.43 \\
    Olmo & WikiText-103 & vv & 0.00 \\
    Olmo & WikiText-103 & cc & 1.00 \\
    Olmo & Bookcorpus & cv & 0.38 \\
    Olmo & Bookcorpus & vv & 0.00 \\
    Olmo & Bookcorpus & cc & 1.00 \\
    Opt & Pile10k & cv & 0.41 \\
    Opt & Pile10k & vv & 0.00 \\
    Opt & Pile10k & cc & 1.00 \\
    Opt & WikiText-103 & cv & 0.41 \\
    Opt & WikiText-103 & vv & 0.00 \\
    Opt & WikiText-103 & cc & 1.00 \\
    Opt & Bookcorpus & cv & 0.35 \\
    Opt & Bookcorpus & vv & 0.00 \\
    Opt & Bookcorpus & cc & 1.00 \\
    Mistral & Pile10k & cv & 0.45 \\
    Mistral & Pile10k & vv & 0.00 \\
    Mistral & Pile10k & cc & 1.00 \\
    Mistral & WikiText-103 & cv & 0.44 \\
    Mistral & WikiText-103 & vv & 0.00 \\
    Mistral & WikiText-103 & cc & 1.00 \\
    Mistral & Bookcorpus & cv & 0.37 \\
    Mistral & Bookcorpus & vv & 0.00 \\
    Mistral & Bookcorpus & cc & 1.00 \\
    Llama & Pile10k & cv & 0.45 \\
    Llama & Pile10k & vv & 0.00 \\
    Llama & Pile10k & cc & 1.00 \\
    Llama & WikiText-103 & cv & 0.45 \\
    Llama & WikiText-103 & vv & 0.00 \\
    Llama & WikiText-103 & cc & 1.00 \\
    Llama & Bookcorpus & cv & 0.37 \\
    Llama & Bookcorpus & vv & 0.00 \\
    Llama & Bookcorpus & cc & 1.00 \\
    \hline
    \end{tabular}
    \caption{\label{table:l2_dist_uni}
    When using softmaxed dot product: mean L2 distance between the resulting probability distribution and the uniform distribution. Rounded to two decimals. Comparison types are: cv - context with vocabulary item, vv - vocabulary with vocabulary and cc - context with context. Note that mean L2 distance is far from zero when comparing contexts with vocabulary items. See more discussion in the appendix text (\ref{app:l2_dists_table_and_explanation}). 
  }
\end{table*}

\section{Distribution of probability distances}
\label{app:distribution_of_prob_distances}
We present here plots showing the distribution of probability distances for Llama (Fig.~\ref{fig:prob_dist_distribution_llama}), Pythia (Fig.~\ref{fig:prob_dist_distribution_pythia}), Olmo (Fig.~\ref{fig:prob_dist_distribution_olmo}), Opt (Fig.~\ref{fig:prob_dist_distribution_opt}) and Mistral (Fig.~\ref{fig:prob_dist_distribution_mistral}). For none of the tested models we find a concentration when using probability distance. 

\begin{figure*}[htb]
\begin{minipage}[b]{0.33\linewidth}
\centering
\includegraphics[width=\textwidth]{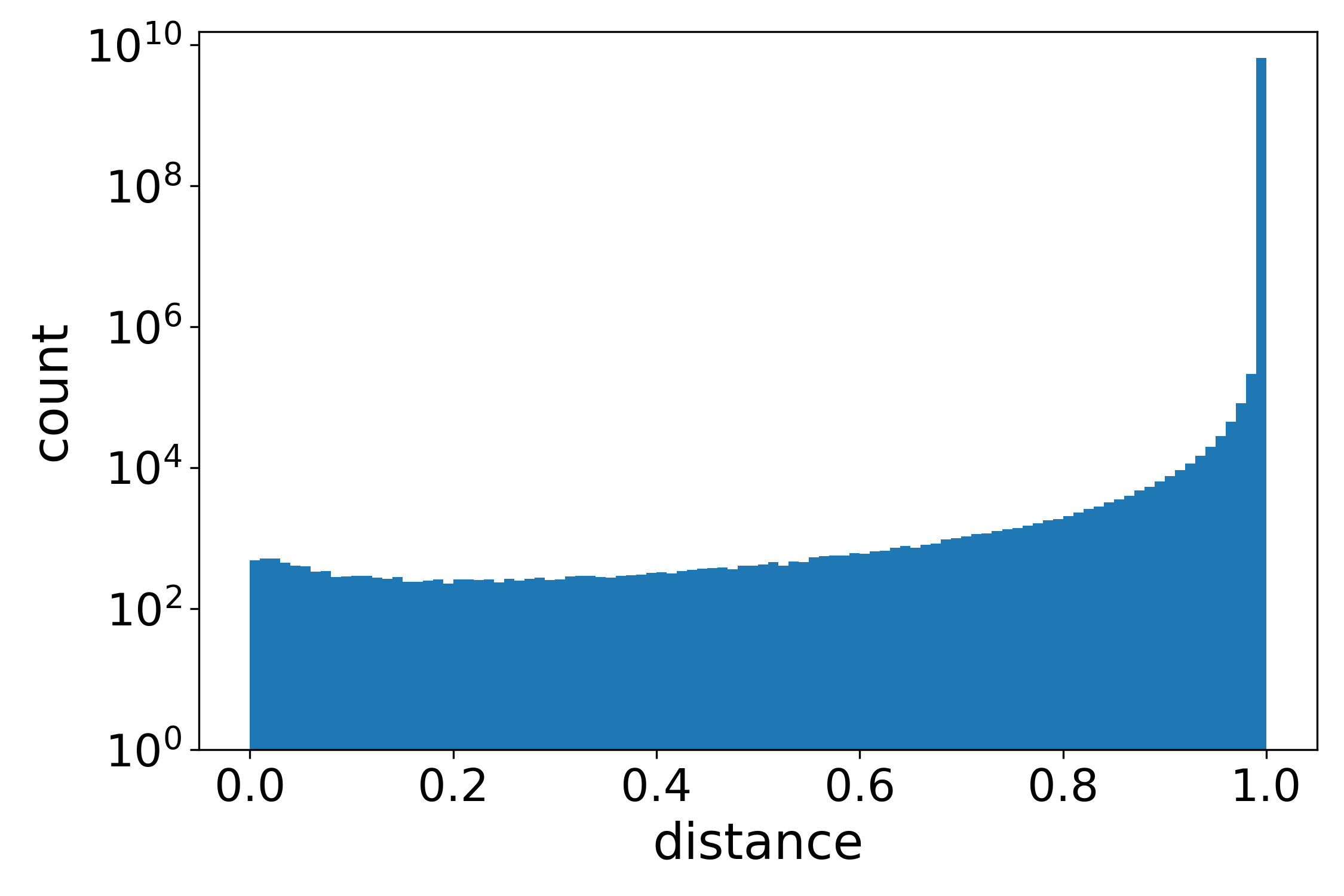}
\end{minipage}
\hspace{-0.1cm}
\begin{minipage}[b]{0.33\linewidth}
\centering
\includegraphics[width=\textwidth]{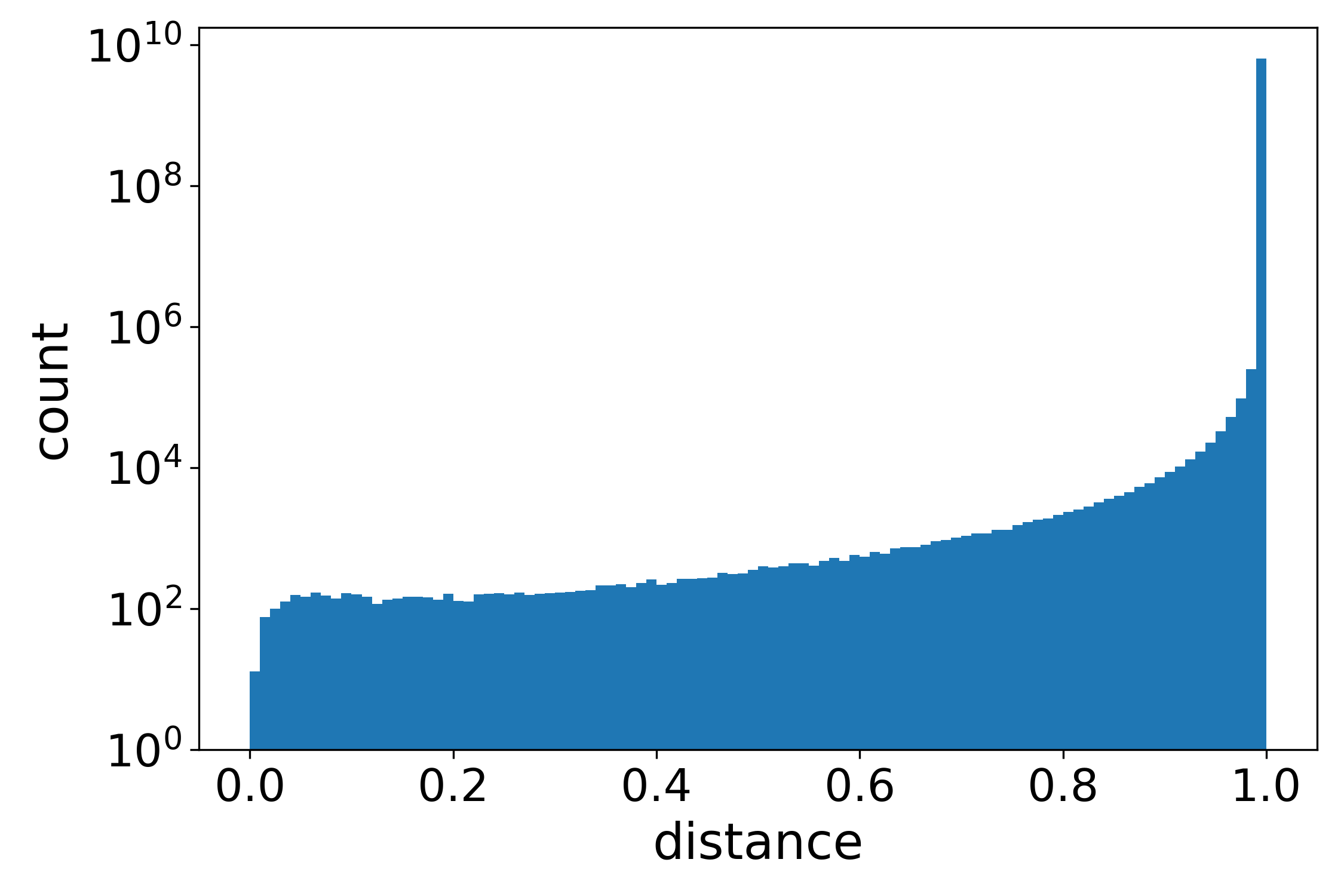}
\end{minipage}
\hspace{-0.1cm}
\begin{minipage}[b]{0.33\linewidth}
\centering
\includegraphics[width=\textwidth]{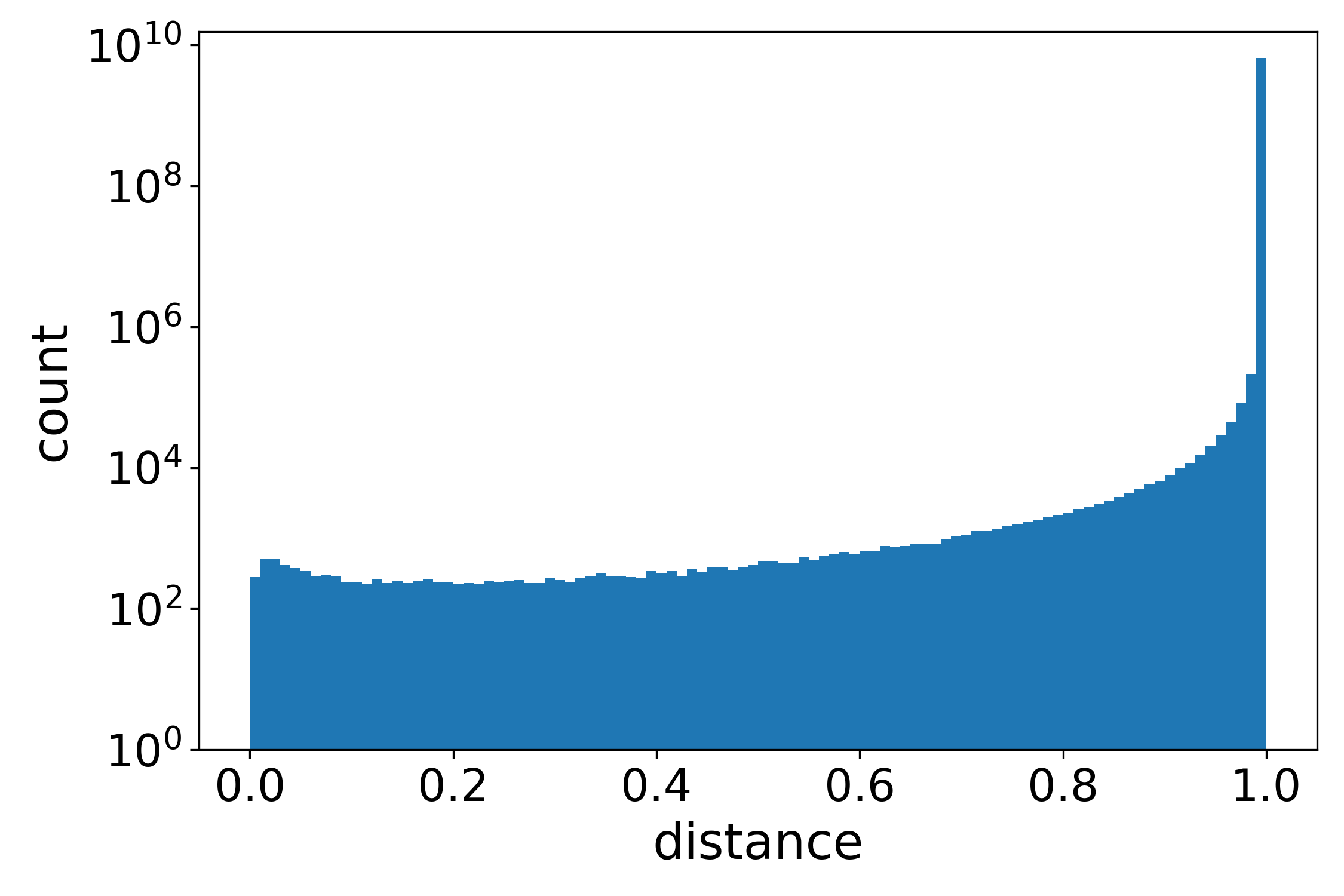}
\end{minipage}
\caption{Distribution of probability distances for Llama on Pile10k (left), Bookcorpus (middle) and WikiText-103 (right). There is no concentration of distances.}
\label{fig:prob_dist_distribution_llama}
\end{figure*}

\begin{figure*}[htb]
\begin{minipage}[b]{0.33\linewidth}
\centering
\includegraphics[width=\textwidth]{figures/hist_ct_plots/hist_ct_prob_dist_pythia_pile.png}
\end{minipage}
\hspace{-0.1cm}
\begin{minipage}[b]{0.33\linewidth}
\centering
\includegraphics[width=\textwidth]{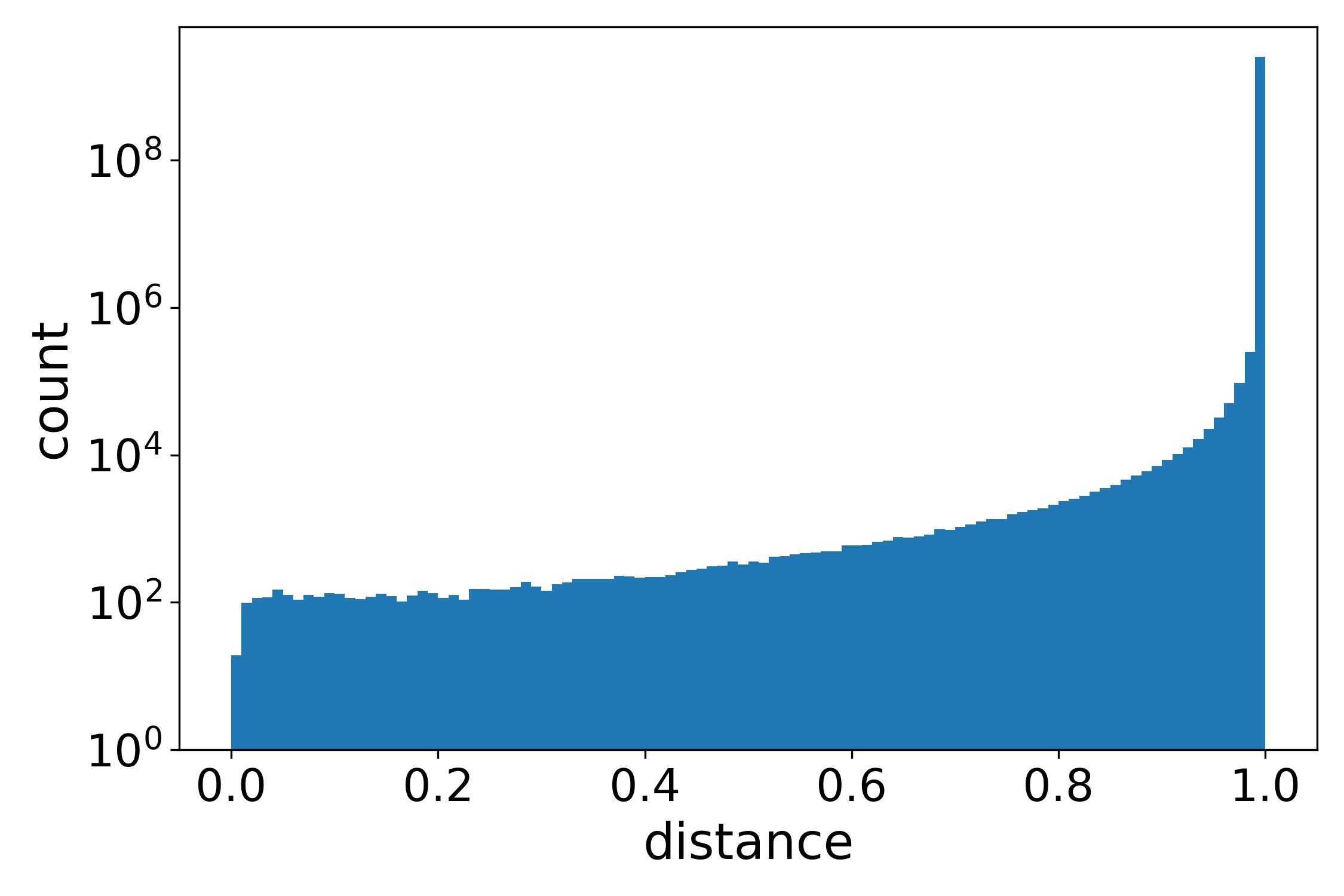}
\end{minipage}
\hspace{-0.1cm}
\begin{minipage}[b]{0.33\linewidth}
\centering
\includegraphics[width=\textwidth]{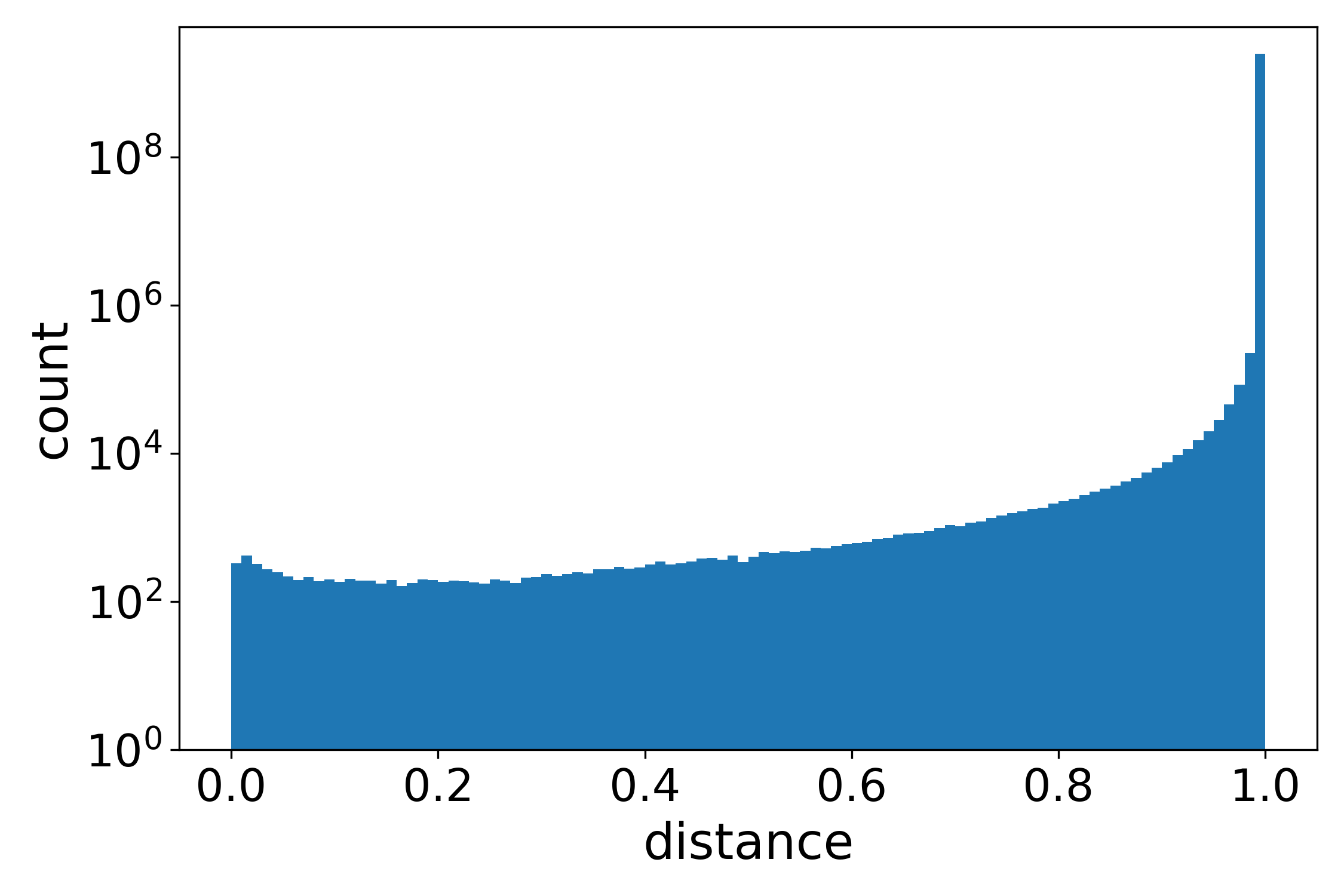}
\end{minipage}
\caption{Distribution of probability distances for Pythia on Pile10k (left), Bookcorpus (middle) and WikiText-103 (right). There is no concentration of distances.}
\label{fig:prob_dist_distribution_pythia}
\end{figure*}

\begin{figure*}[htb]
\begin{minipage}[b]{0.33\linewidth}
\centering
\includegraphics[width=\textwidth]{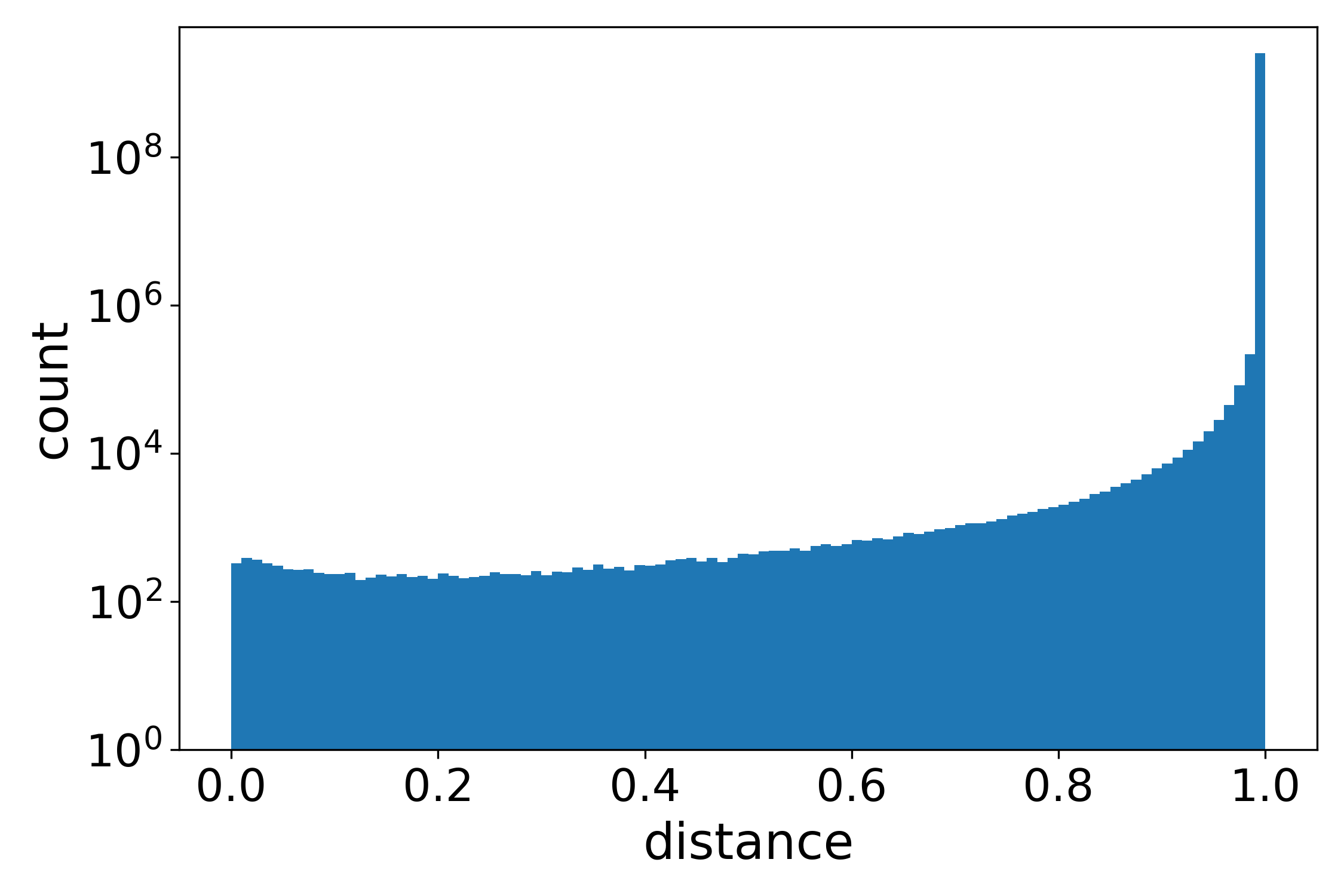}
\end{minipage}
\hspace{-0.1cm}
\begin{minipage}[b]{0.33\linewidth}
\centering
\includegraphics[width=\textwidth]{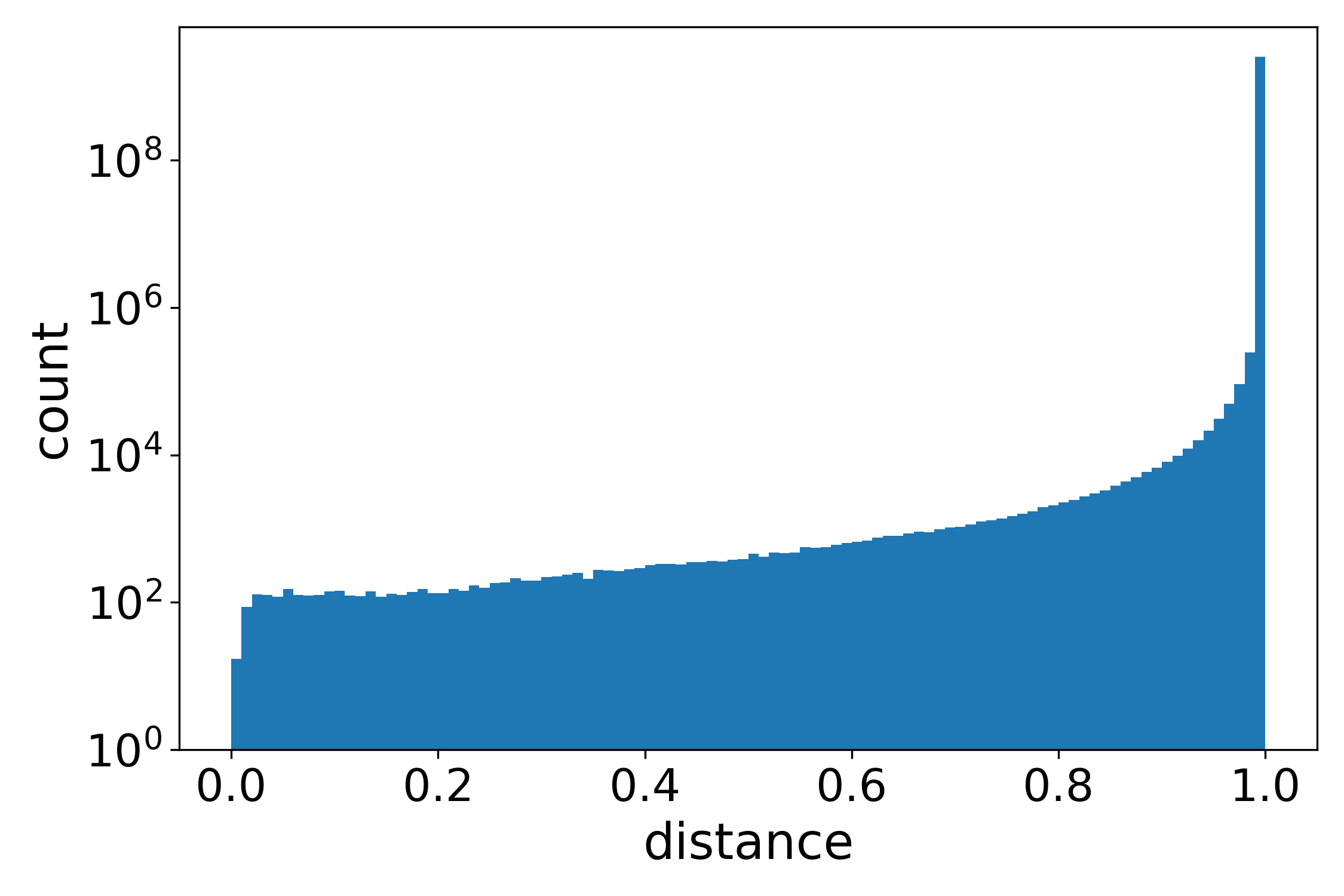}
\end{minipage}
\hspace{-0.1cm}
\begin{minipage}[b]{0.33\linewidth}
\centering
\includegraphics[width=\textwidth]{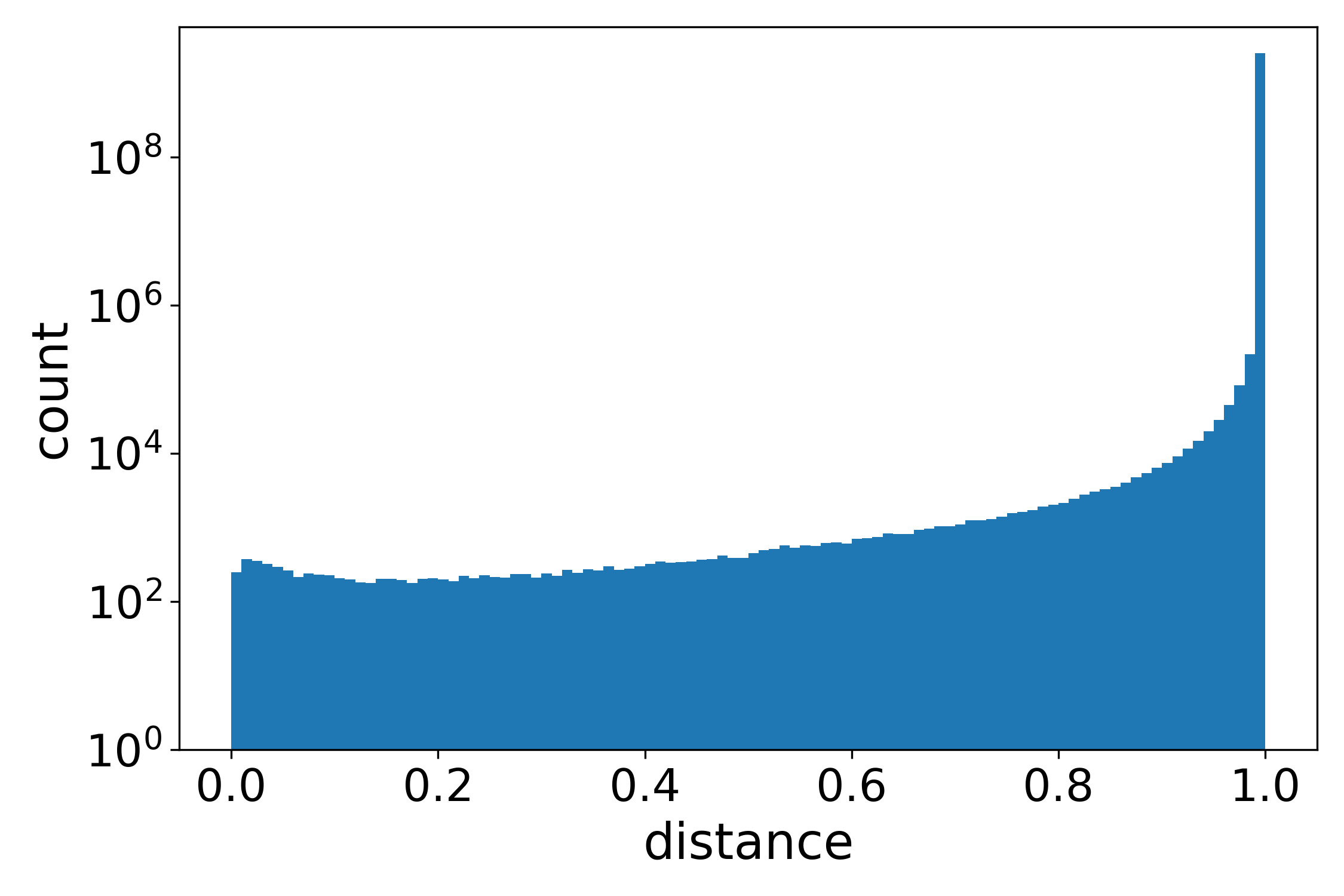}
\end{minipage}
\caption{Distribution of probability distances for Olmo on Pile10k (left), Bookcorpus (middle) and WikiText-103 (right). There is no concentration of distances.}
\label{fig:prob_dist_distribution_olmo}
\end{figure*}

\begin{figure*}[htb]
\begin{minipage}[b]{0.33\linewidth}
\centering
\includegraphics[width=\textwidth]{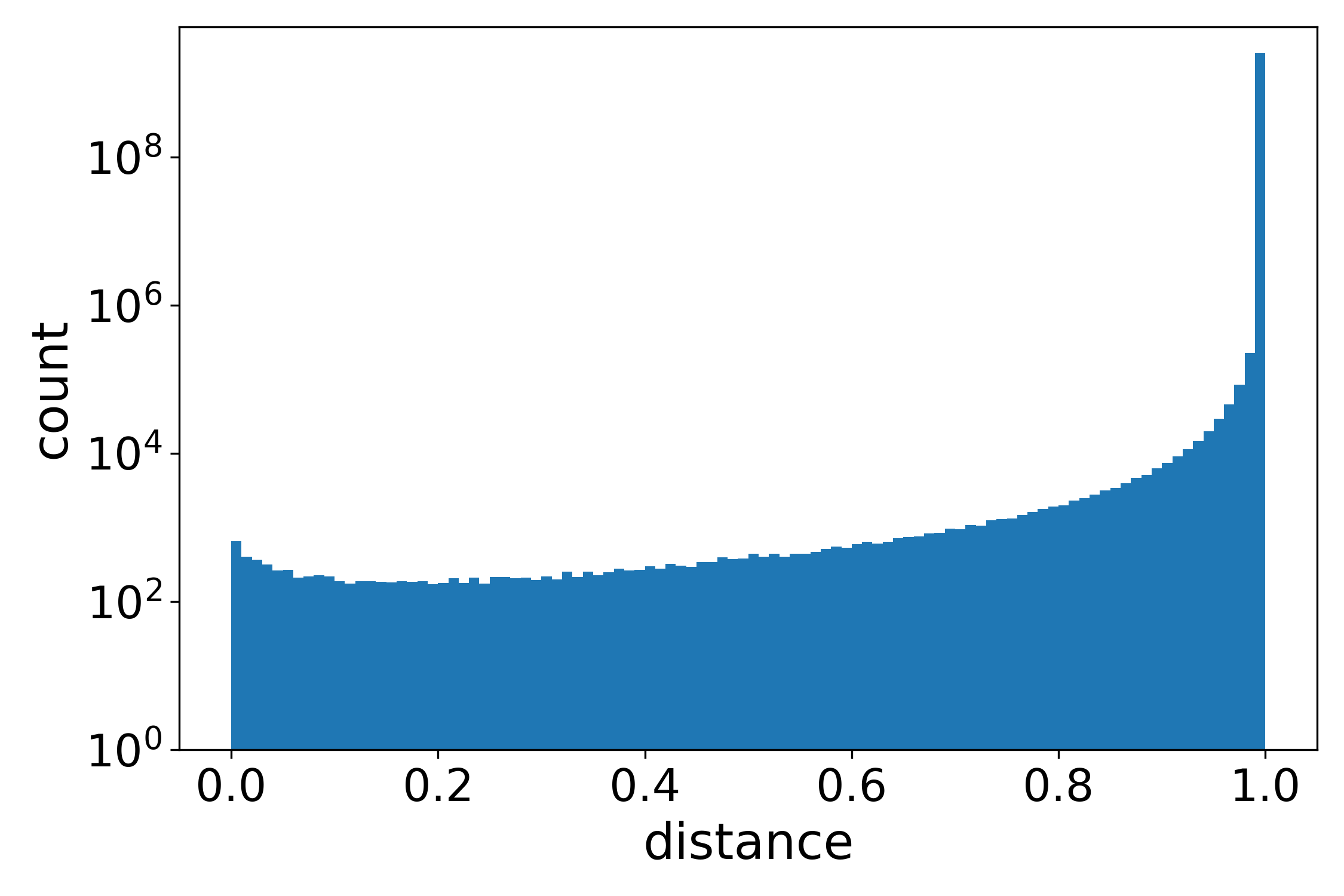}
\end{minipage}
\hspace{-0.1cm}
\begin{minipage}[b]{0.33\linewidth}
\centering
\includegraphics[width=\textwidth]{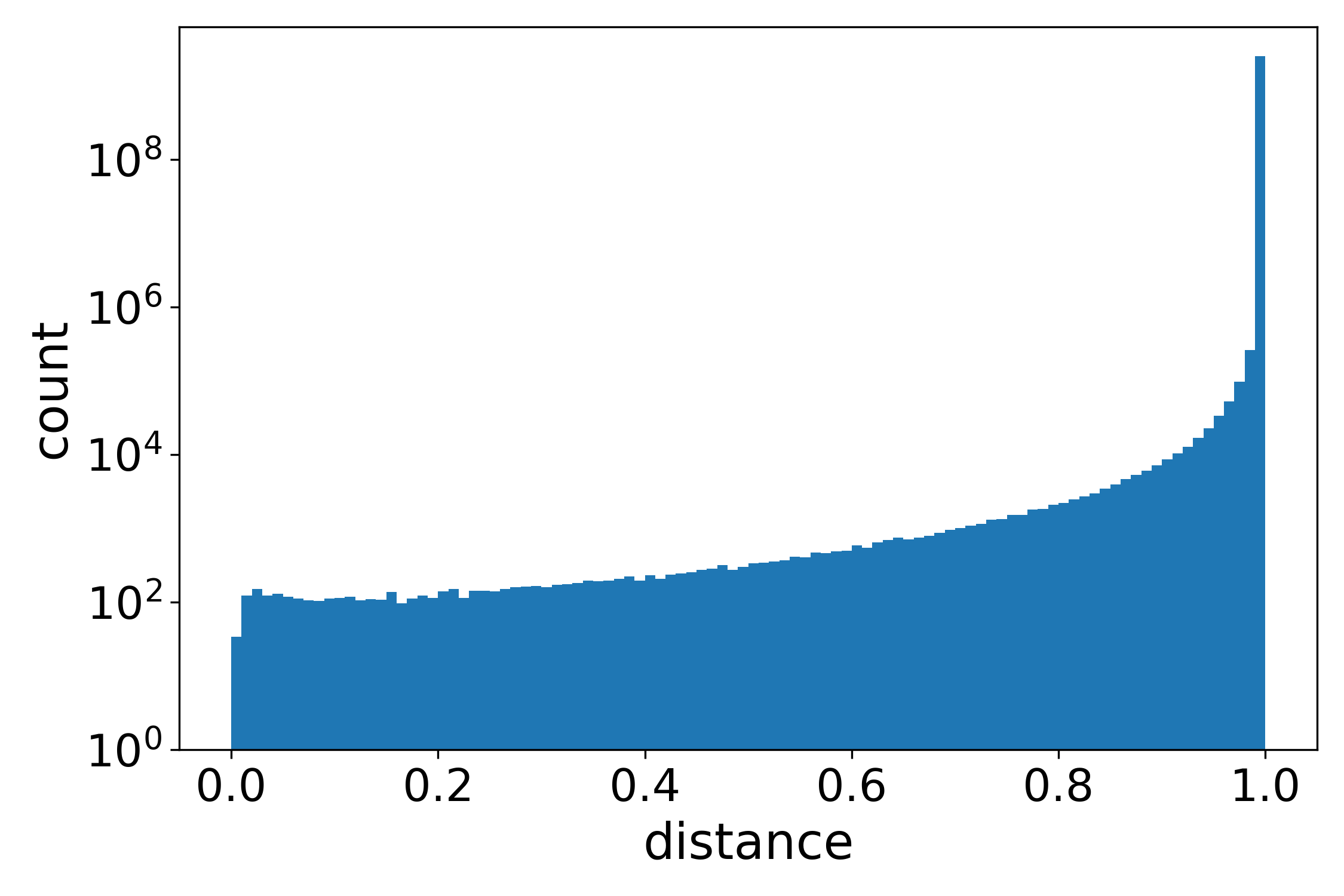}
\end{minipage}
\hspace{-0.1cm}
\begin{minipage}[b]{0.33\linewidth}
\centering
\includegraphics[width=\textwidth]{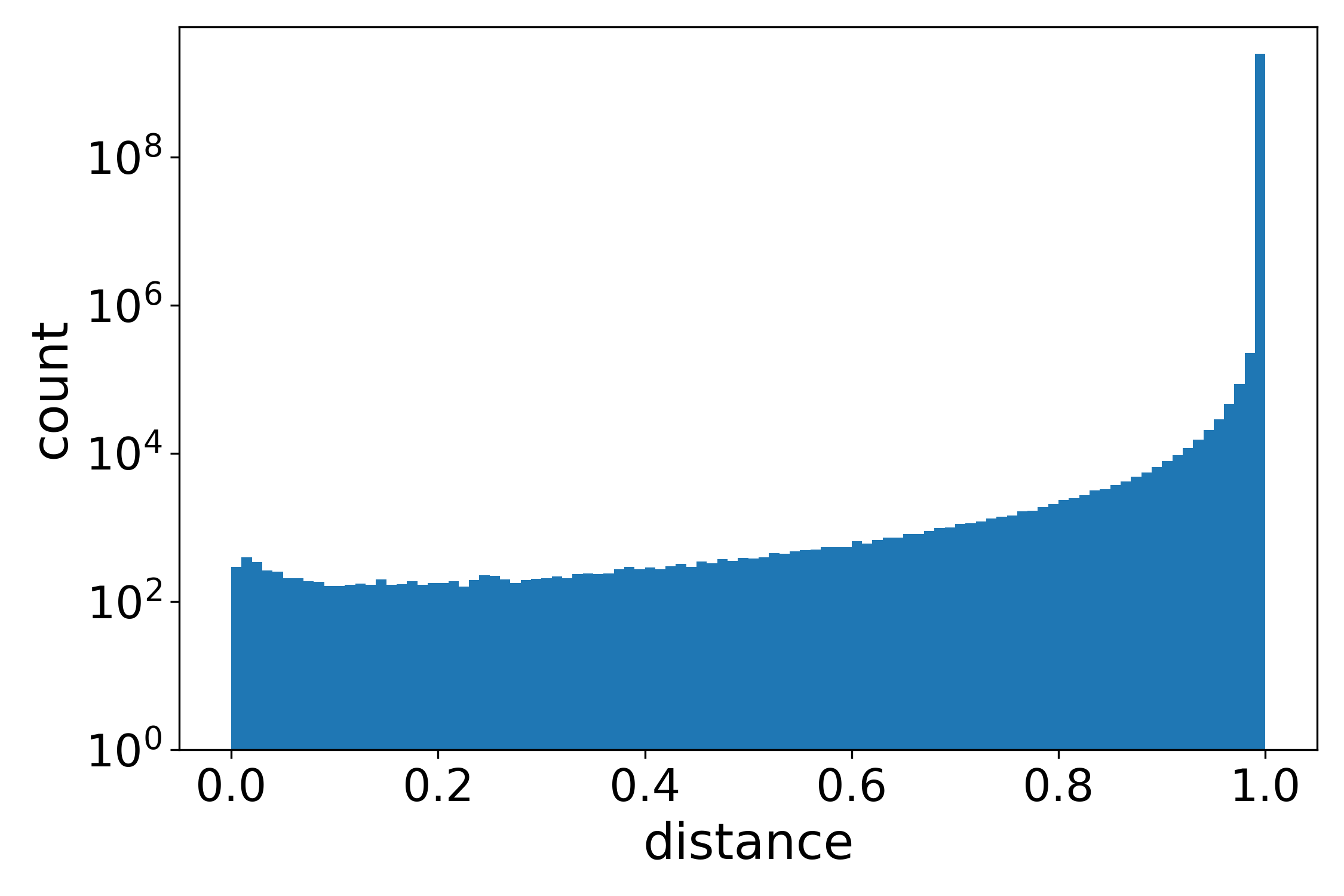}
\end{minipage}
\caption{Distribution of probability distances for Opt on Pile10k (left), Bookcorpus (middle) and WikiText-103 (right). There is no concentration of distances.}
\label{fig:prob_dist_distribution_opt}
\end{figure*}

\begin{figure*}[htb]
\begin{minipage}[b]{0.33\linewidth}
\centering
\includegraphics[width=\textwidth]{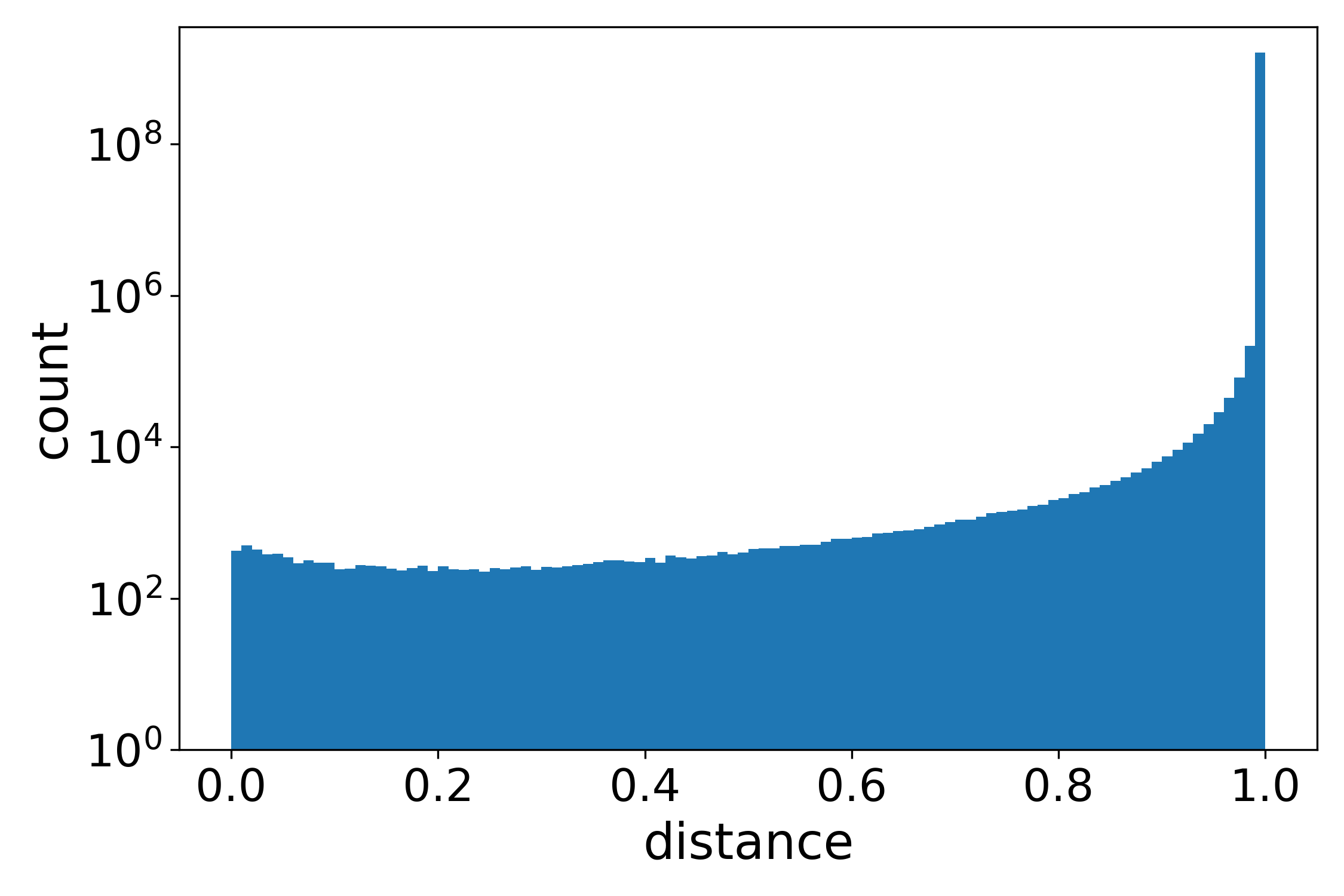}
\end{minipage}
\hspace{-0.1cm}
\begin{minipage}[b]{0.33\linewidth}
\centering
\includegraphics[width=\textwidth]{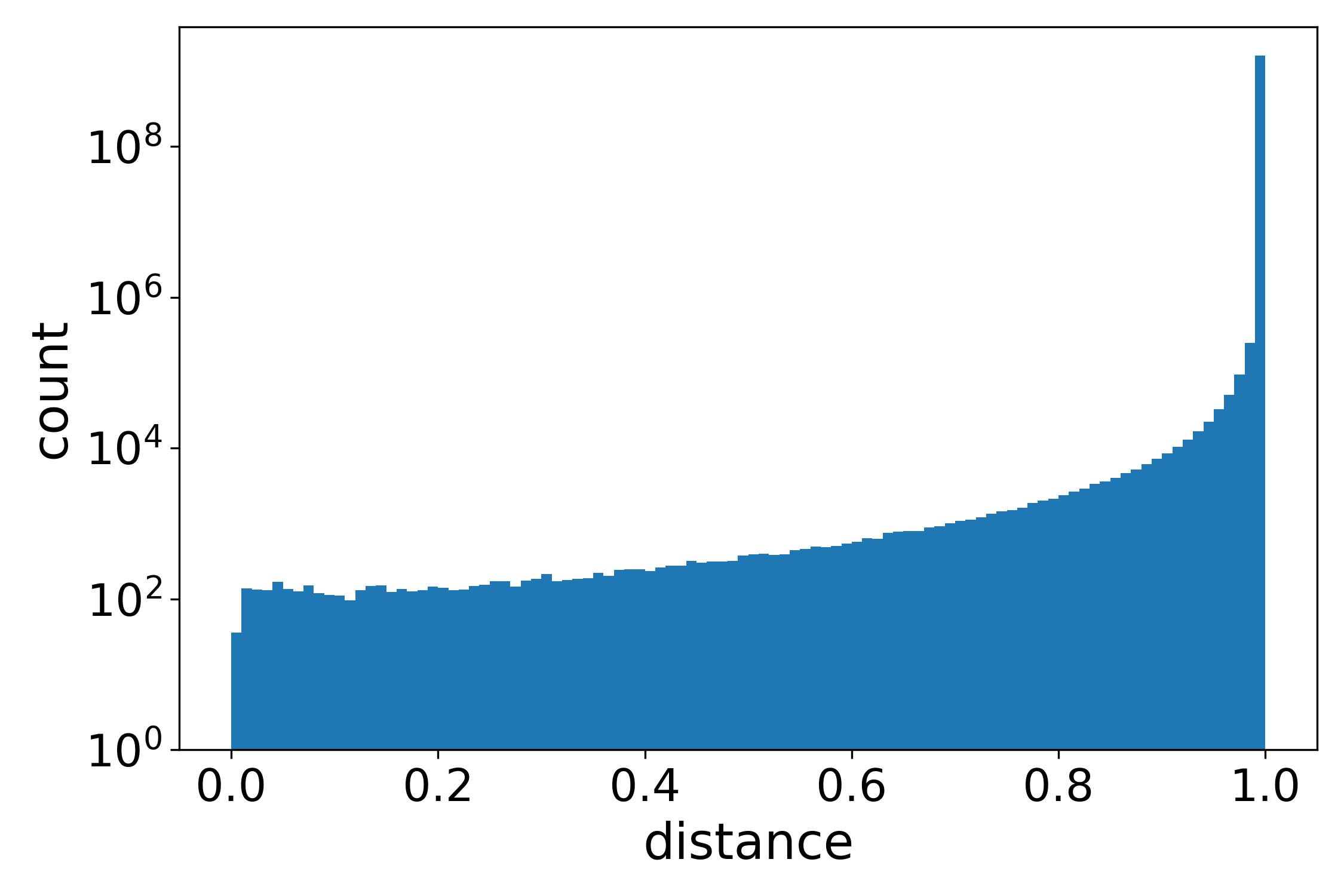}
\end{minipage}
\hspace{-0.1cm}
\begin{minipage}[b]{0.33\linewidth}
\centering
\includegraphics[width=\textwidth]{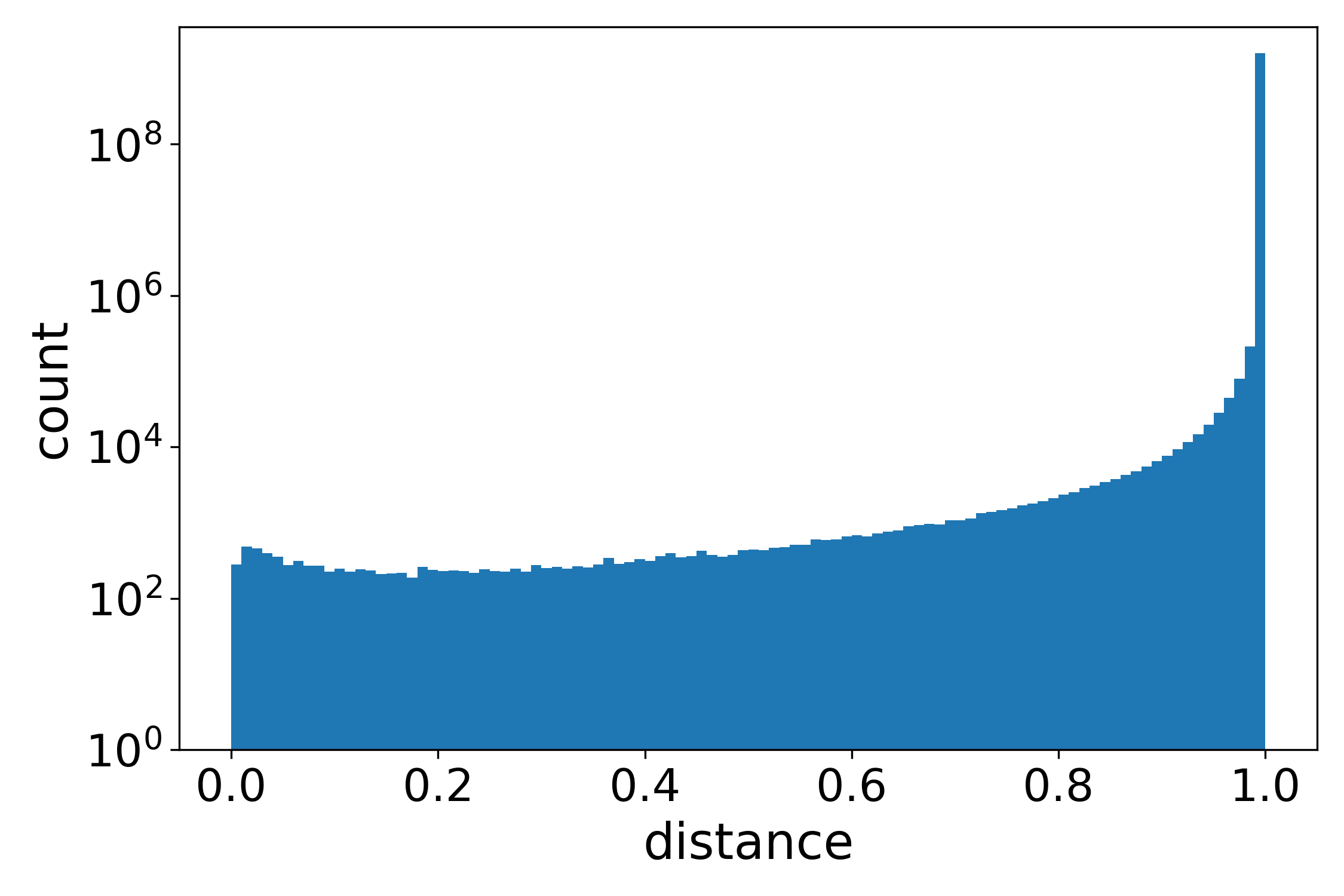}
\end{minipage}
\caption{Distribution of probability distances for Mistral on Pile10k (left), Bookcorpus (middle) and WikiText-103 (right). There is no concentration of distances.}
\label{fig:prob_dist_distribution_mistral}
\end{figure*}

\section{Distribution of context-to-context distances}
\label{app:distribution_of_cc_distances}
Plots showing the distribution of distances when comparing context with context for Llama, using Euclidean distance (Fig.~\ref{fig:cc_euc_dist_distribution_llama}), normalized Euclidean distance (Fig.~\ref{fig:cc_norm_euc_dist_distribution_llama}) and softmaxed dot product (Fig.~\ref{fig:cc_softmax_dot_distribution_llama}); Pythia, using Euclidean distance (Fig.~\ref{fig:cc_euc_dist_distribution_pythia}), normalized Euclidean distance (Fig.~\ref{fig:cc_norm_euc_dist_distribution_pythia}) and softmaxed dot product (Fig.~\ref{fig:cc_softmax_dot_distribution_pythia}); Opt, using Euclidean distance (Fig.~\ref{fig:cc_euc_dist_distribution_opt}), normalized Euclidean distance (Fig.~\ref{fig:cc_norm_euc_dist_distribution_opt}) and softmaxed dot product (Fig.~\ref{fig:cc_softmax_dot_distribution_opt}); Olmo, using Euclidean distance (Fig.~\ref{fig:cc_euc_dist_distribution_olmo}), normalized Euclidean distance (Fig.~\ref{fig:cc_norm_euc_dist_distribution_olmo}) and softmaxed dot product (Fig.~\ref{fig:cc_softmax_dot_distribution_olmo}) and Mistral, using Euclidean distance (Fig.~\ref{fig:cc_euc_dist_distribution_mistral}), normalized Euclidean distance (Fig.~\ref{fig:cc_norm_euc_dist_distribution_mistral}) and softmaxed dot product (Fig.~\ref{fig:cc_softmax_dot_distribution_mistral}). For all models we see a concentration of distances in the sense that there is a gap from zero to the lowest distance values.  

\begin{figure*}[htb]
\begin{minipage}[b]{0.33\linewidth}
\centering
\includegraphics[width=\textwidth]{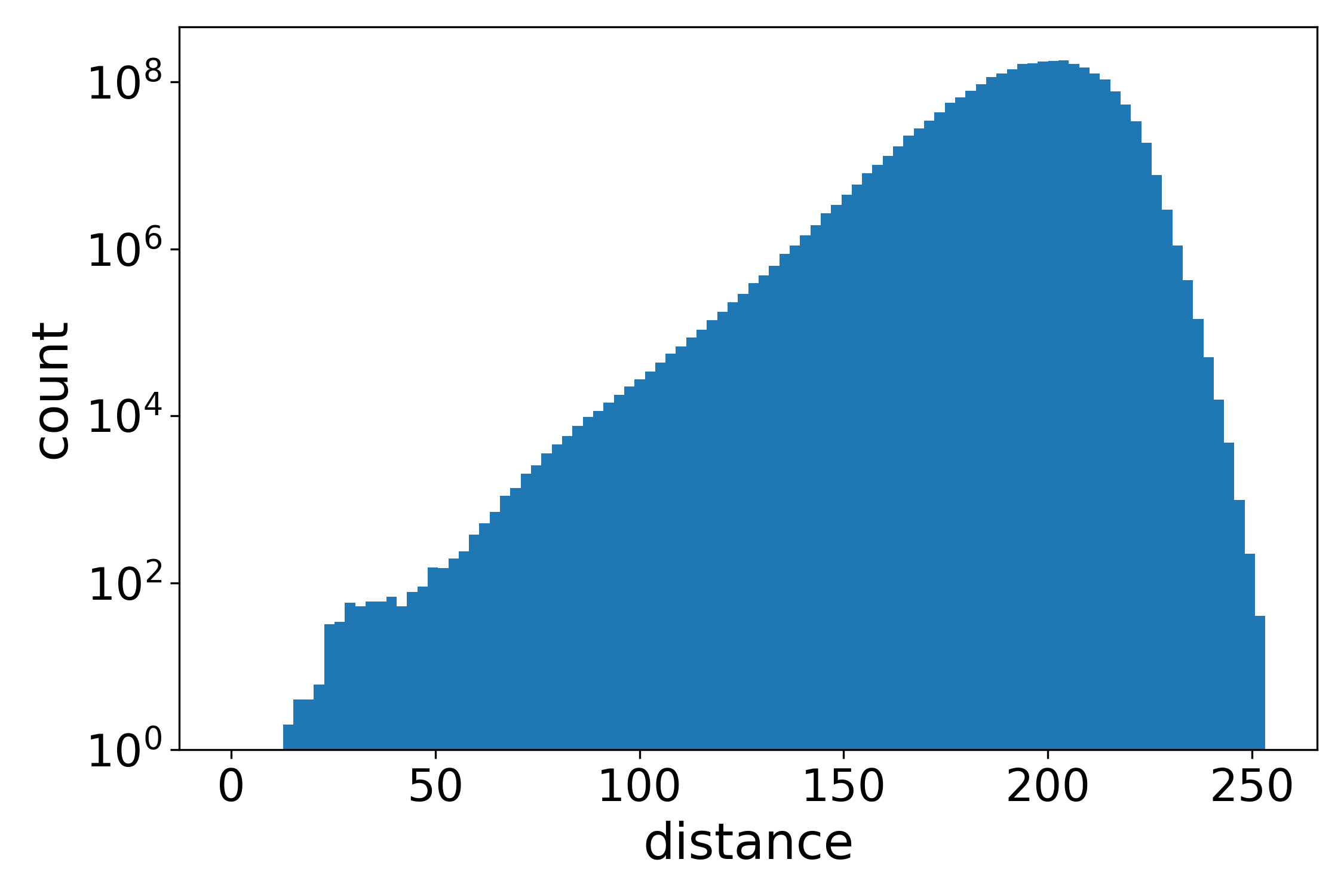}
\end{minipage}
\hspace{-0.1cm}
\begin{minipage}[b]{0.33\linewidth}
\centering
\includegraphics[width=\textwidth]{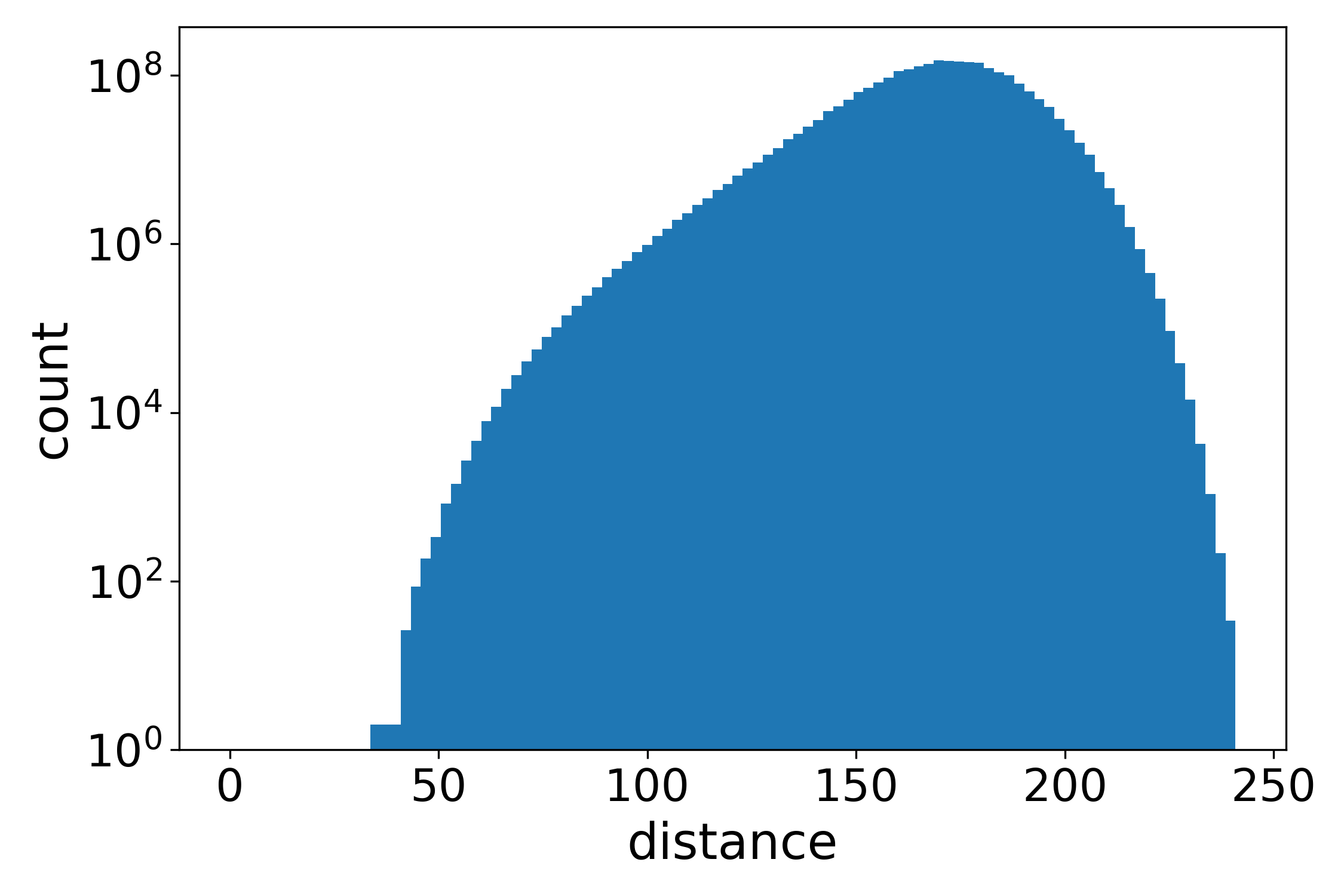}
\end{minipage}
\hspace{-0.1cm}
\begin{minipage}[b]{0.33\linewidth}
\centering
\includegraphics[width=\textwidth]{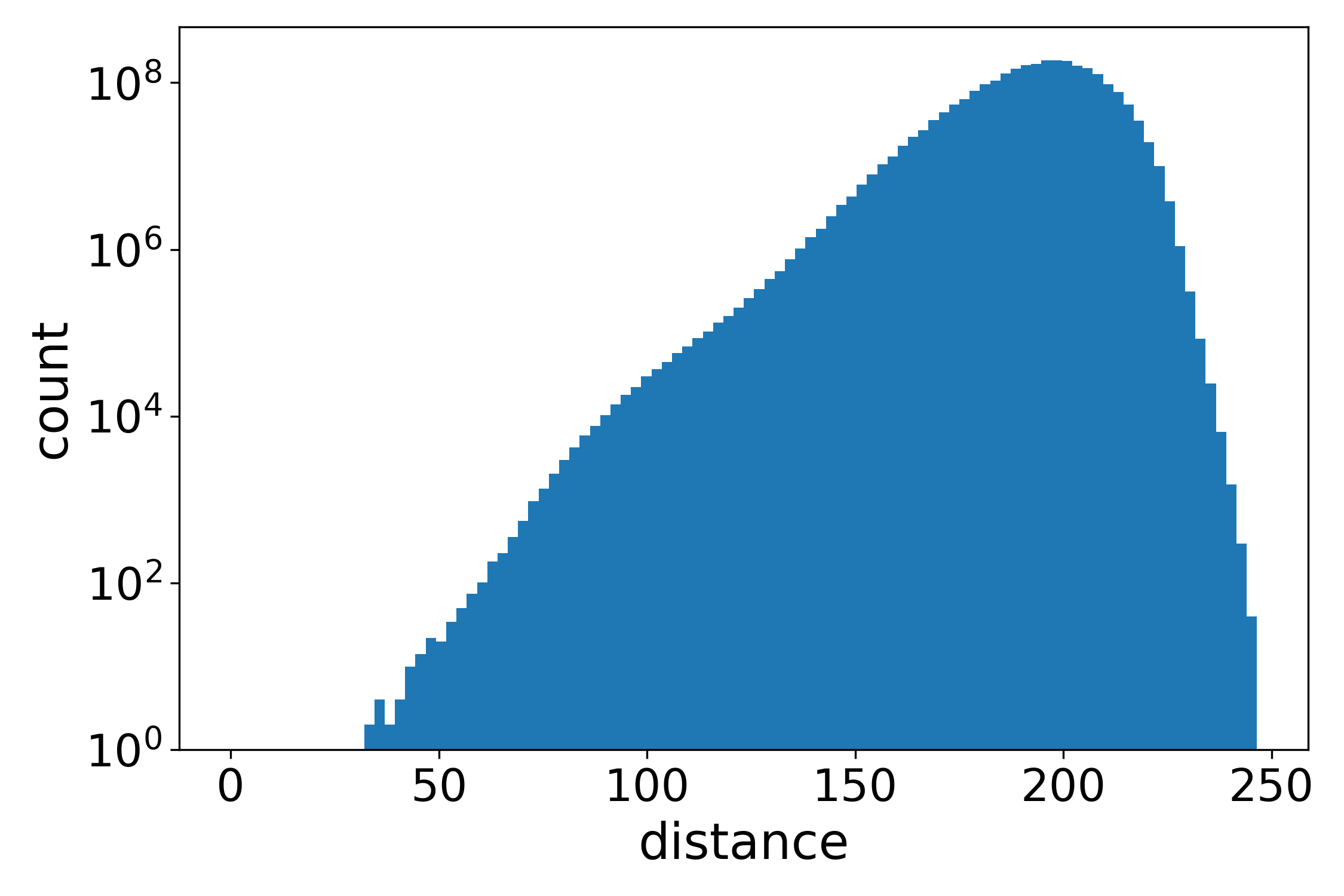}
\end{minipage}
\caption{Distribution of context-to-context Euclidean distances for Llama on Pile10k (left), Bookcorpus (middle) and WikiText-103 (right). We see concentration of distances in the sense that there is a gap from zero to the lowest distance values. Here, we do not include the distance of a context to itself, since it will always be zero for this distance measure.}
\label{fig:cc_euc_dist_distribution_llama}
\end{figure*}

\begin{figure*}[htb]
\begin{minipage}[b]{0.33\linewidth}
\centering
\includegraphics[width=\textwidth]{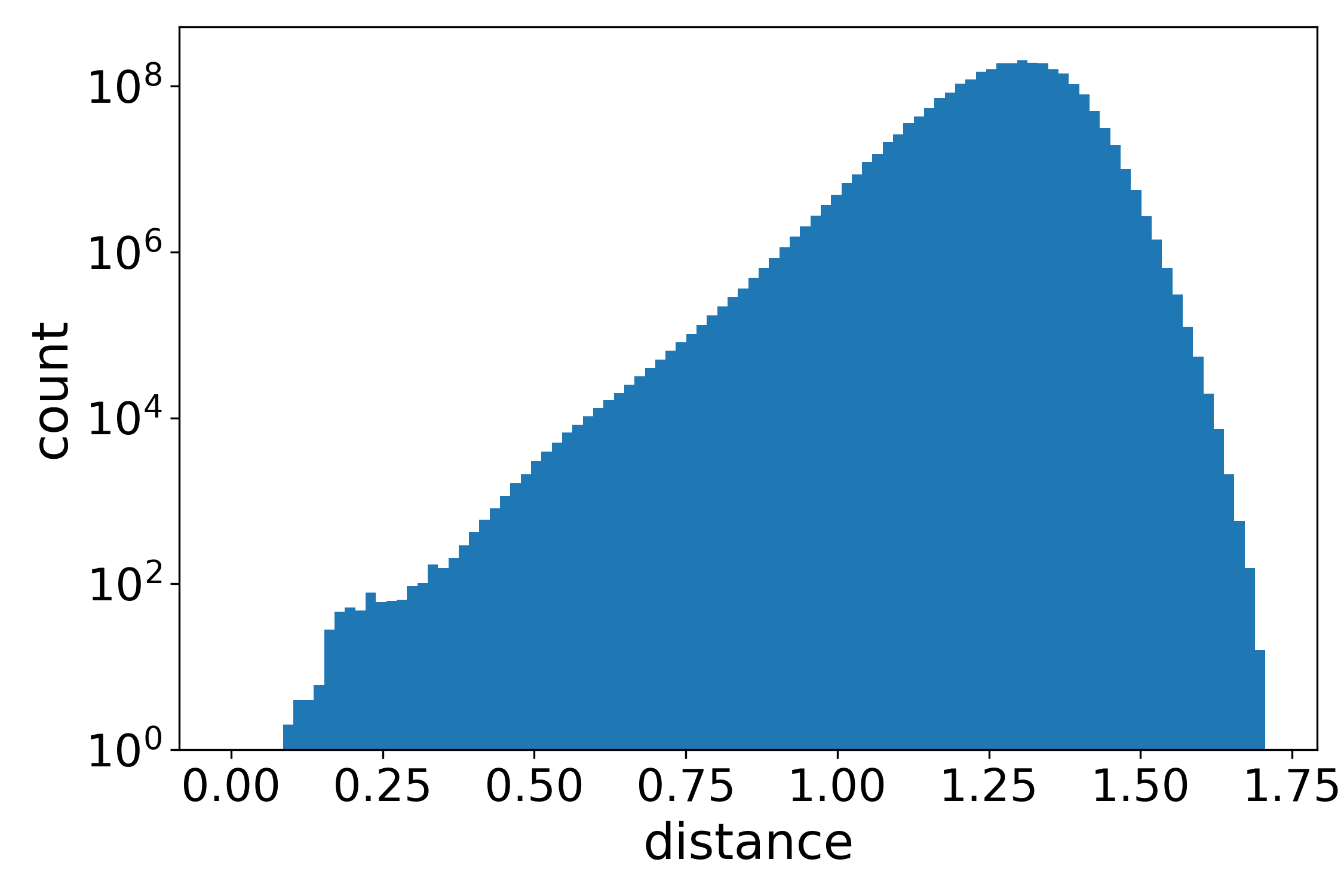}
\end{minipage}
\hspace{-0.1cm}
\begin{minipage}[b]{0.33\linewidth}
\centering
\includegraphics[width=\textwidth]{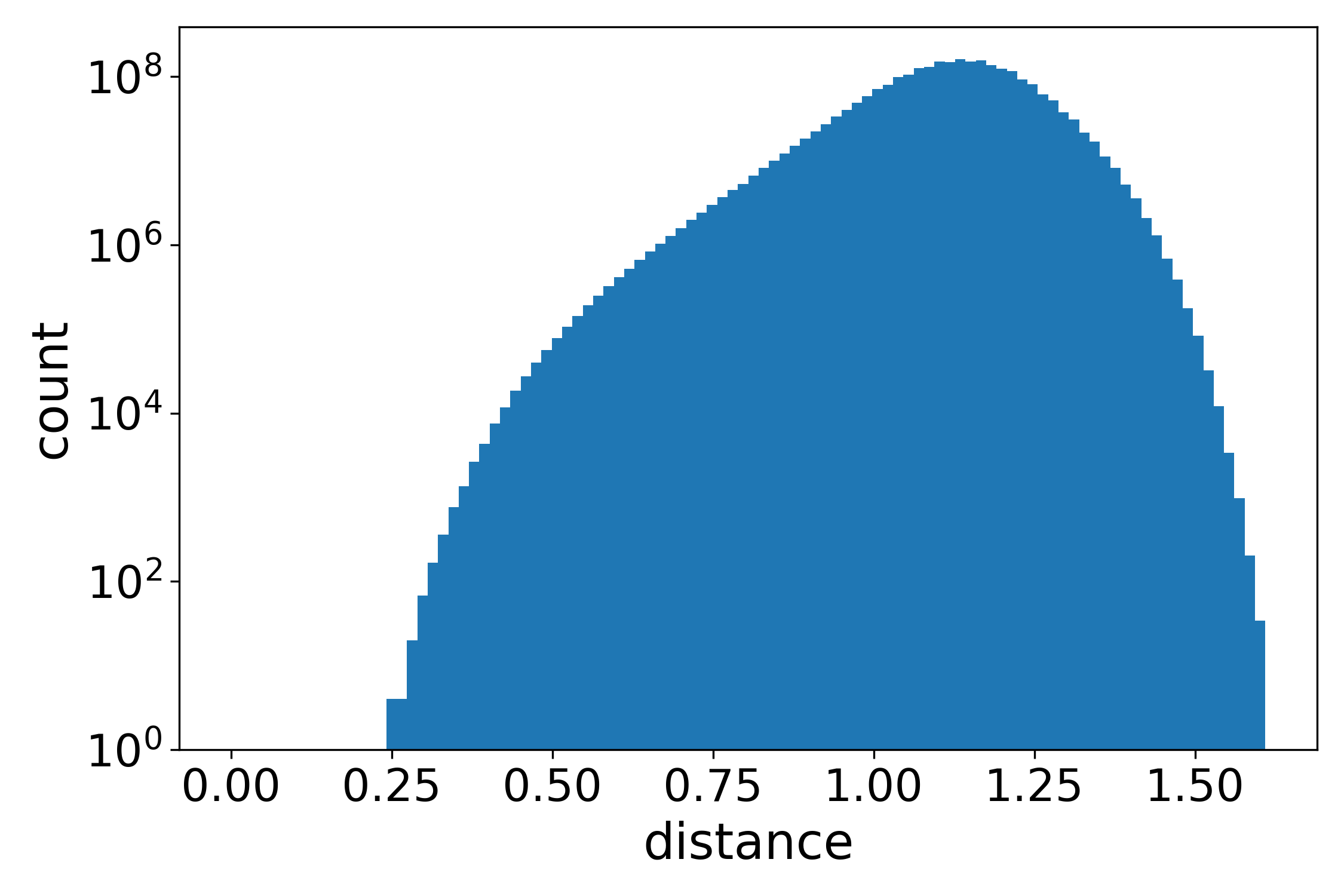}
\end{minipage}
\hspace{-0.1cm}
\begin{minipage}[b]{0.33\linewidth}
\centering
\includegraphics[width=\textwidth]{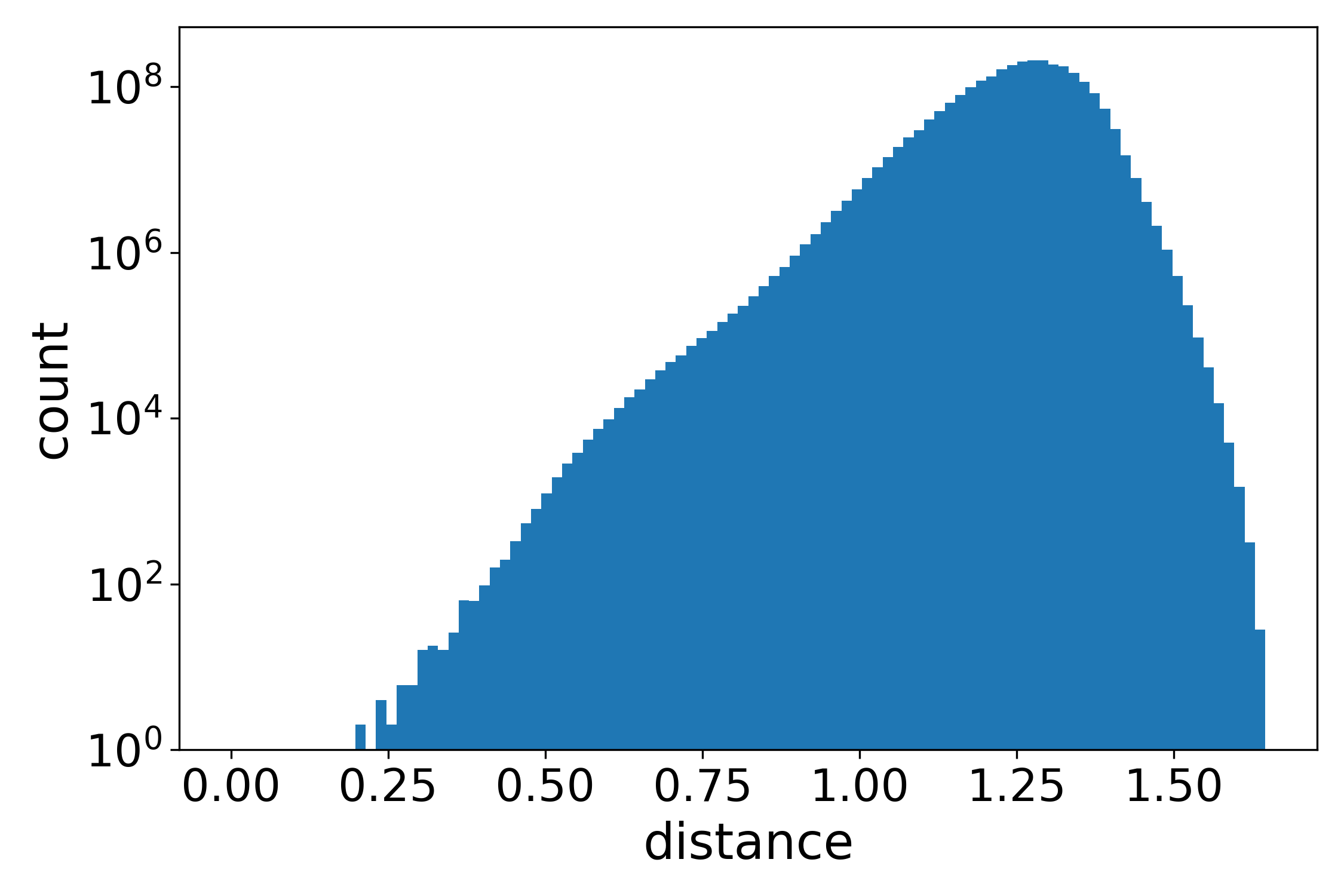}
\end{minipage}
\caption{Distribution of context-to-context normalized Euclidean distances for Llama on Pile10k (left), Bookcorpus (middle) and WikiText-103 (right). We see concentration of distances in the sense that there is a gap from zero to the lowest distance values. Here, we do not include the distance of a context to itself, since it will always be zero for this distance measure.}
\label{fig:cc_norm_euc_dist_distribution_llama}
\end{figure*}

\begin{figure*}[htb]
\begin{minipage}[b]{0.33\linewidth}
\centering
\includegraphics[width=\textwidth]{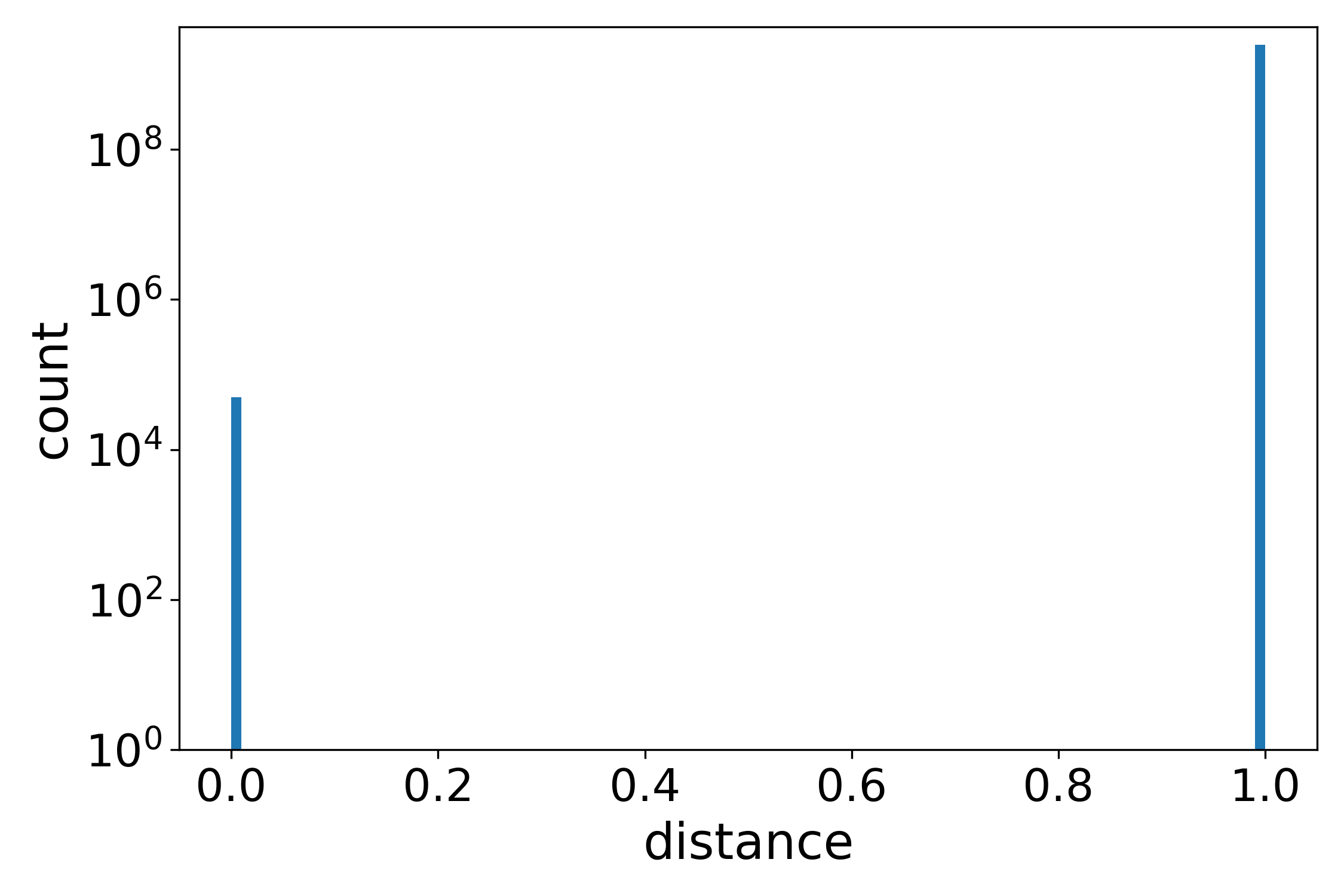}
\end{minipage}
\hspace{-0.1cm}
\begin{minipage}[b]{0.33\linewidth}
\centering
\includegraphics[width=\textwidth]{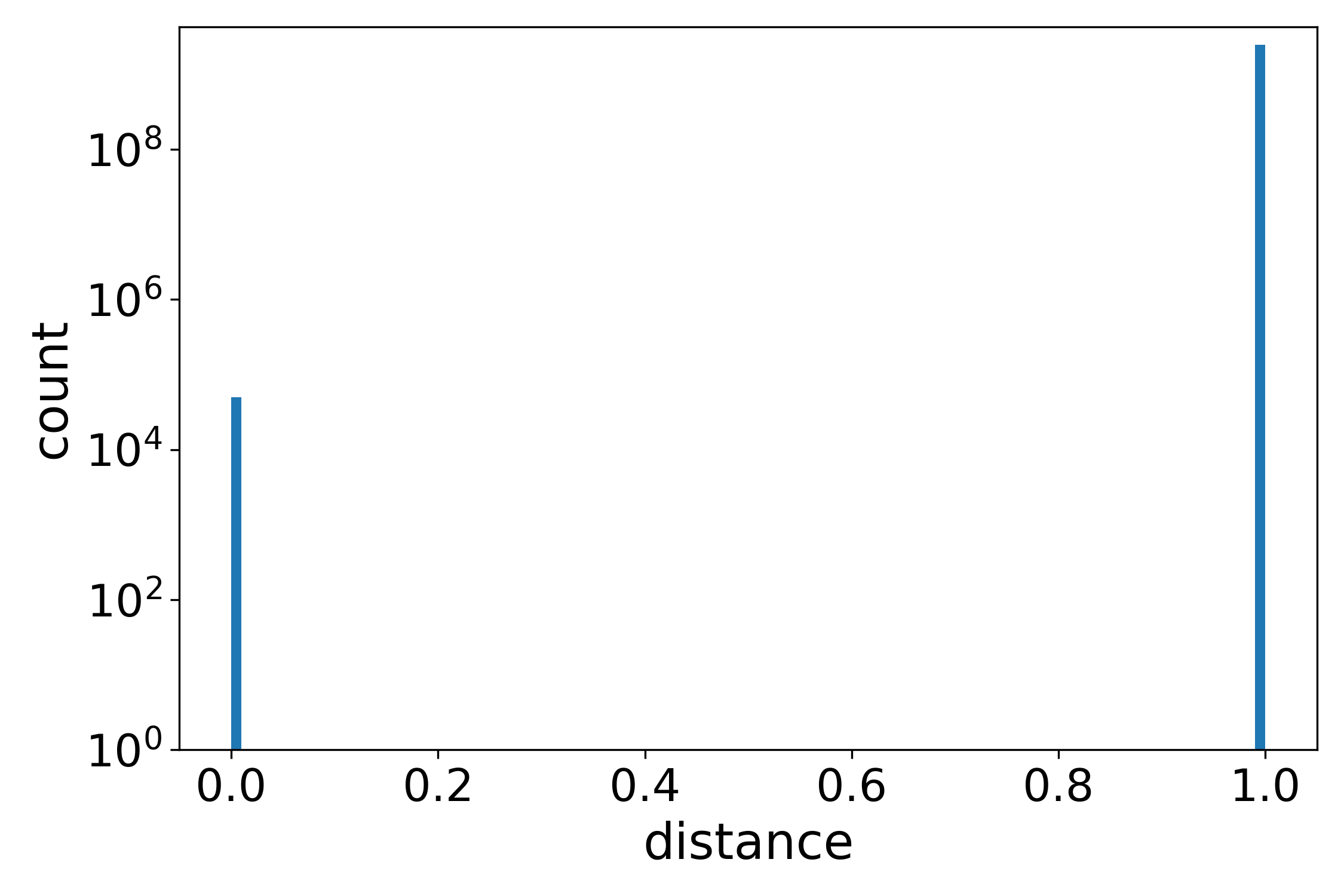}
\end{minipage}
\hspace{-0.1cm}
\begin{minipage}[b]{0.33\linewidth}
\centering
\includegraphics[width=\textwidth]{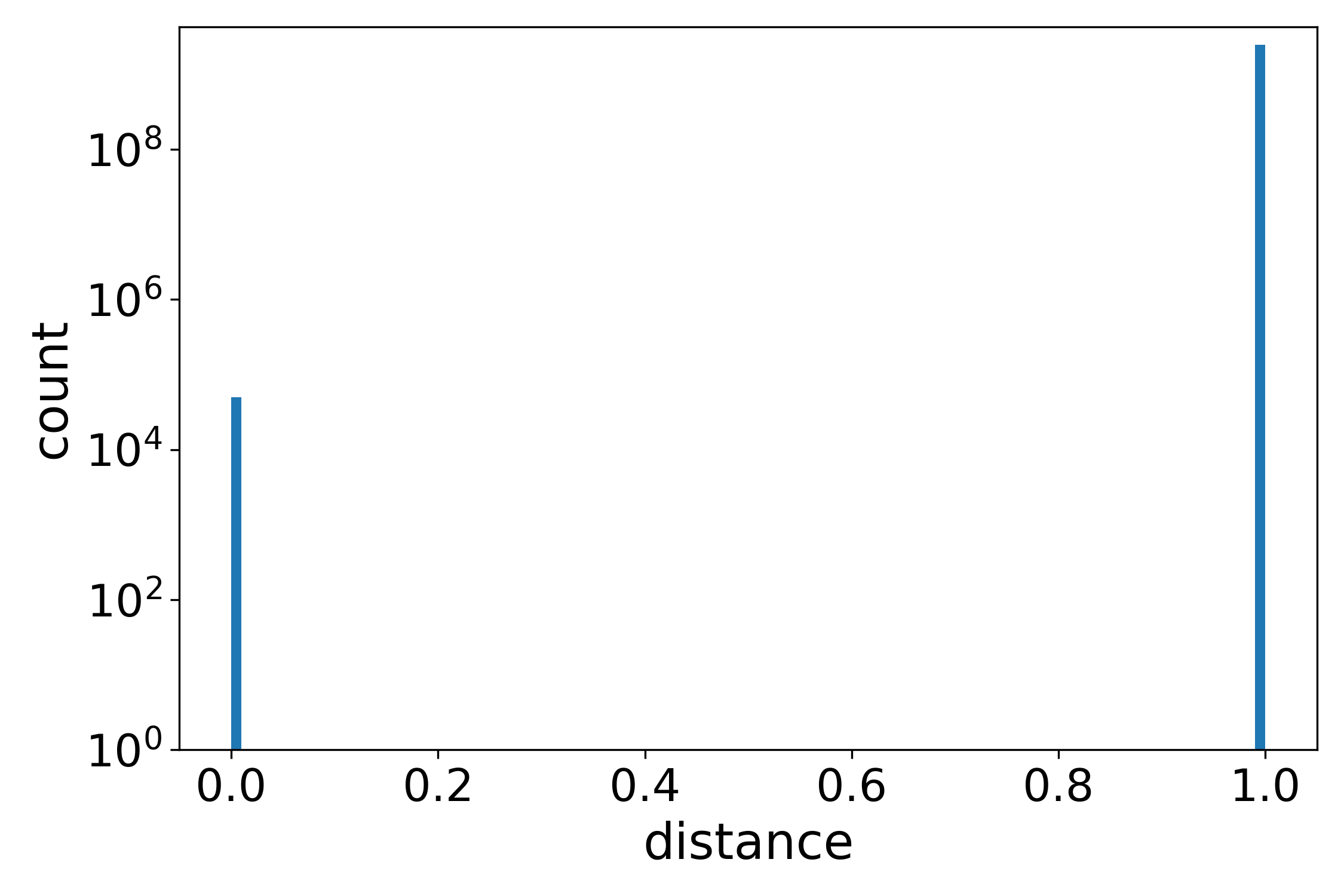}
\end{minipage}
\caption{Distribution of context-to-context softmaxed dot product distances for Llama on Pile10k (left), Bookcorpus (middle) and WikiText-103 (right). Here we have included the distance of a context to itself, which is the spike at zero. Note that, when using the dot product, there is no guarantee that a context will get the largest score with itself.}
\label{fig:cc_softmax_dot_distribution_llama}
\end{figure*}

\begin{figure*}[htb]
\begin{minipage}[b]{0.33\linewidth}
\centering
\includegraphics[width=\textwidth]{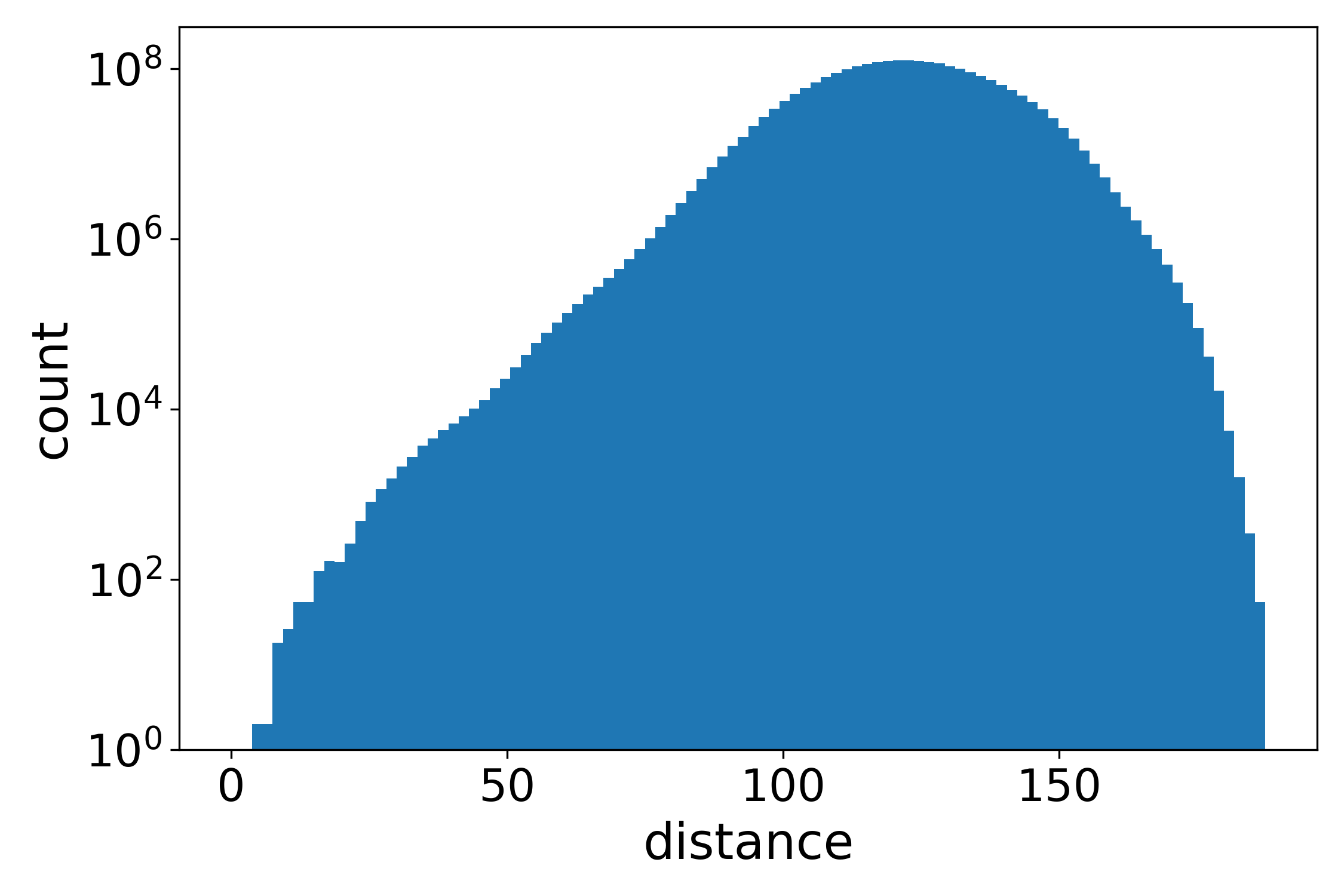}
\end{minipage}
\hspace{-0.1cm}
\begin{minipage}[b]{0.33\linewidth}
\centering
\includegraphics[width=\textwidth]{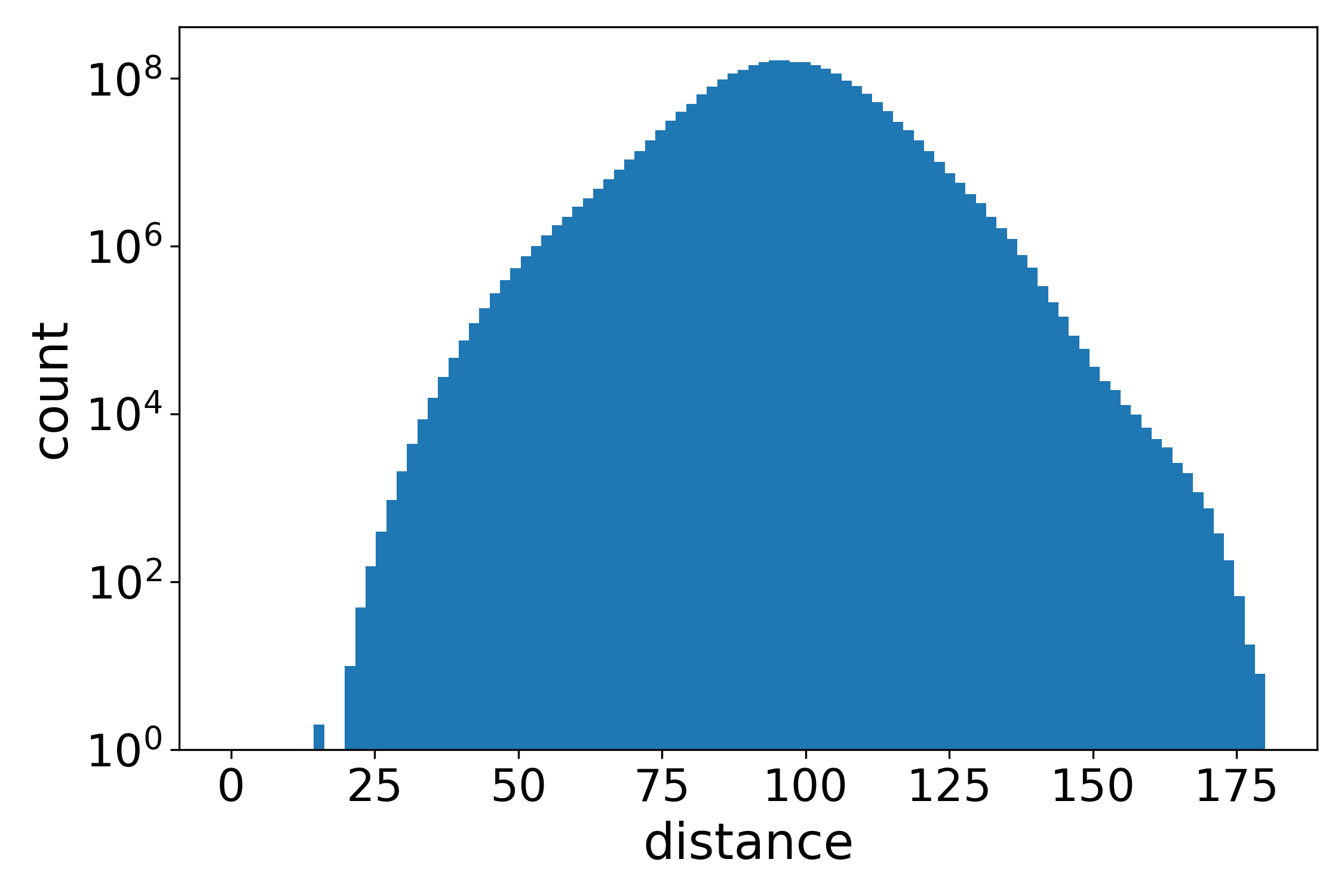}
\end{minipage}
\hspace{-0.1cm}
\begin{minipage}[b]{0.33\linewidth}
\centering
\includegraphics[width=\textwidth]{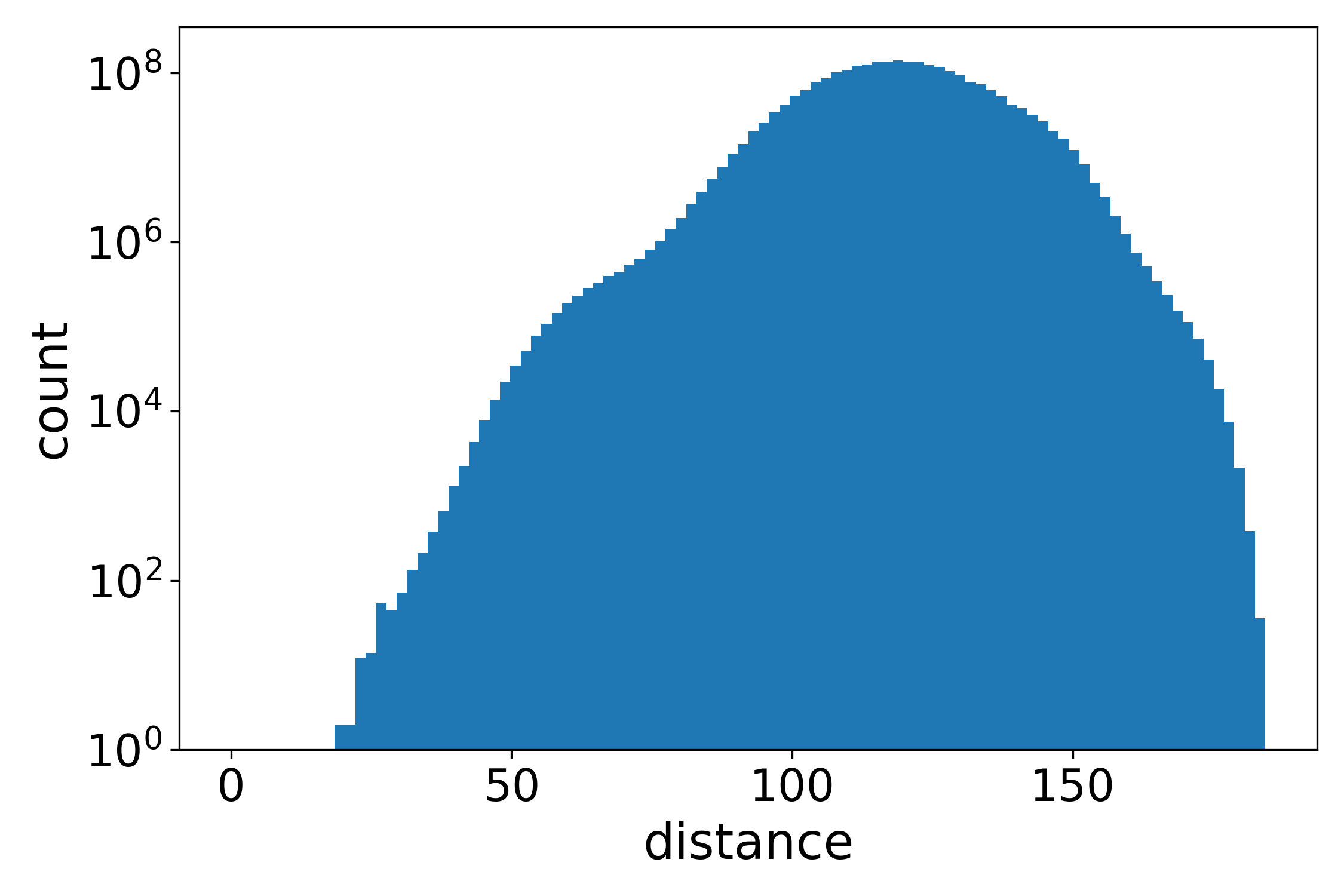}
\end{minipage}
\caption{Distribution of context-to-context Euclidean distances for Pythia on Pile10k (left), Bookcorpus (middle) and WikiText-103 (right). We see concentration of distances in the sense that there is a gap from zero to the lowest distance values. Here, we do not include the distance of a context to itself, since it will always be zero for this distance measure.}
\label{fig:cc_euc_dist_distribution_pythia}
\end{figure*}

\begin{figure*}[htb]
\begin{minipage}[b]{0.33\linewidth}
\centering
\includegraphics[width=\textwidth]{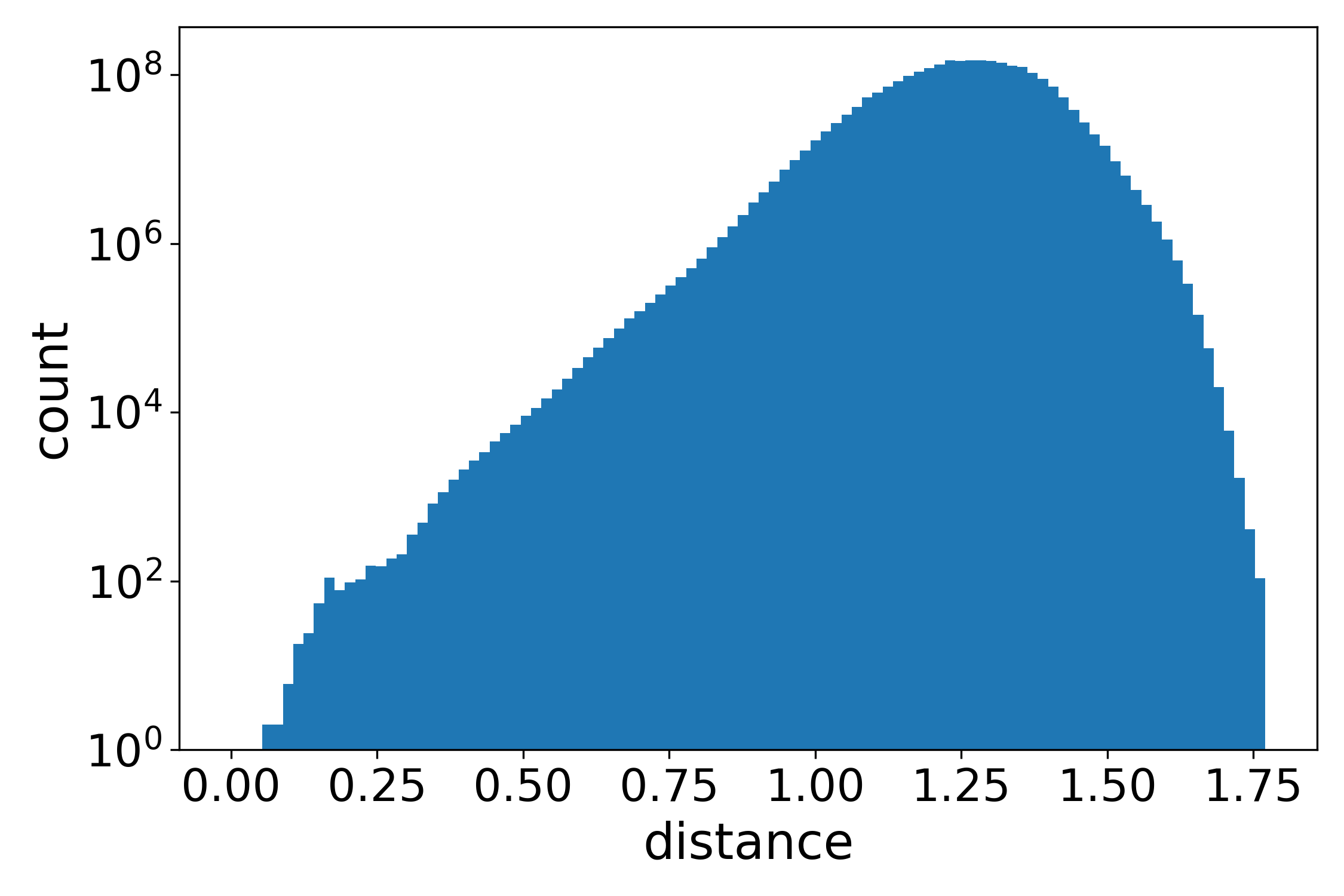}
\end{minipage}
\hspace{-0.1cm}
\begin{minipage}[b]{0.33\linewidth}
\centering
\includegraphics[width=\textwidth]{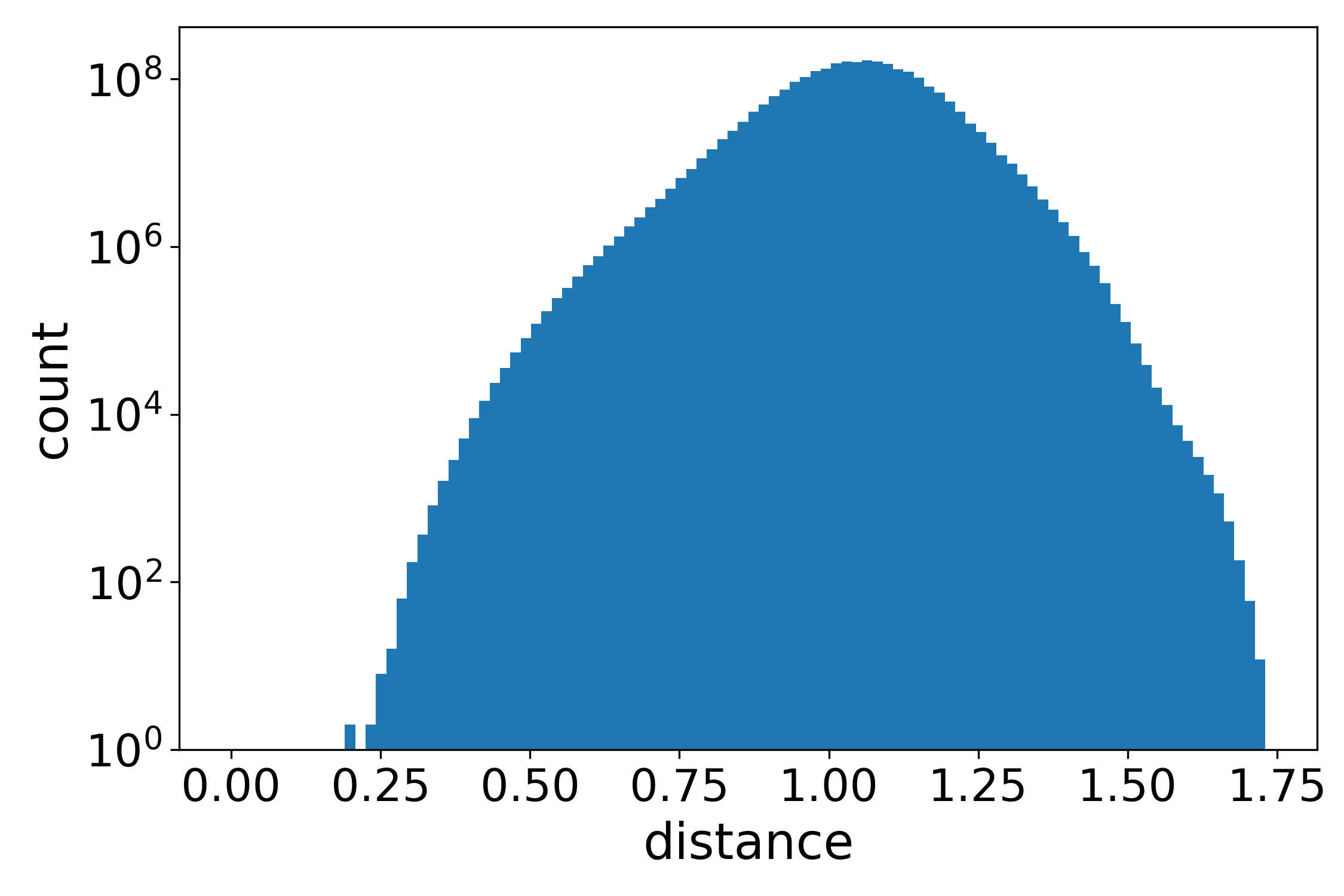}
\end{minipage}
\hspace{-0.1cm}
\begin{minipage}[b]{0.33\linewidth}
\centering
\includegraphics[width=\textwidth]{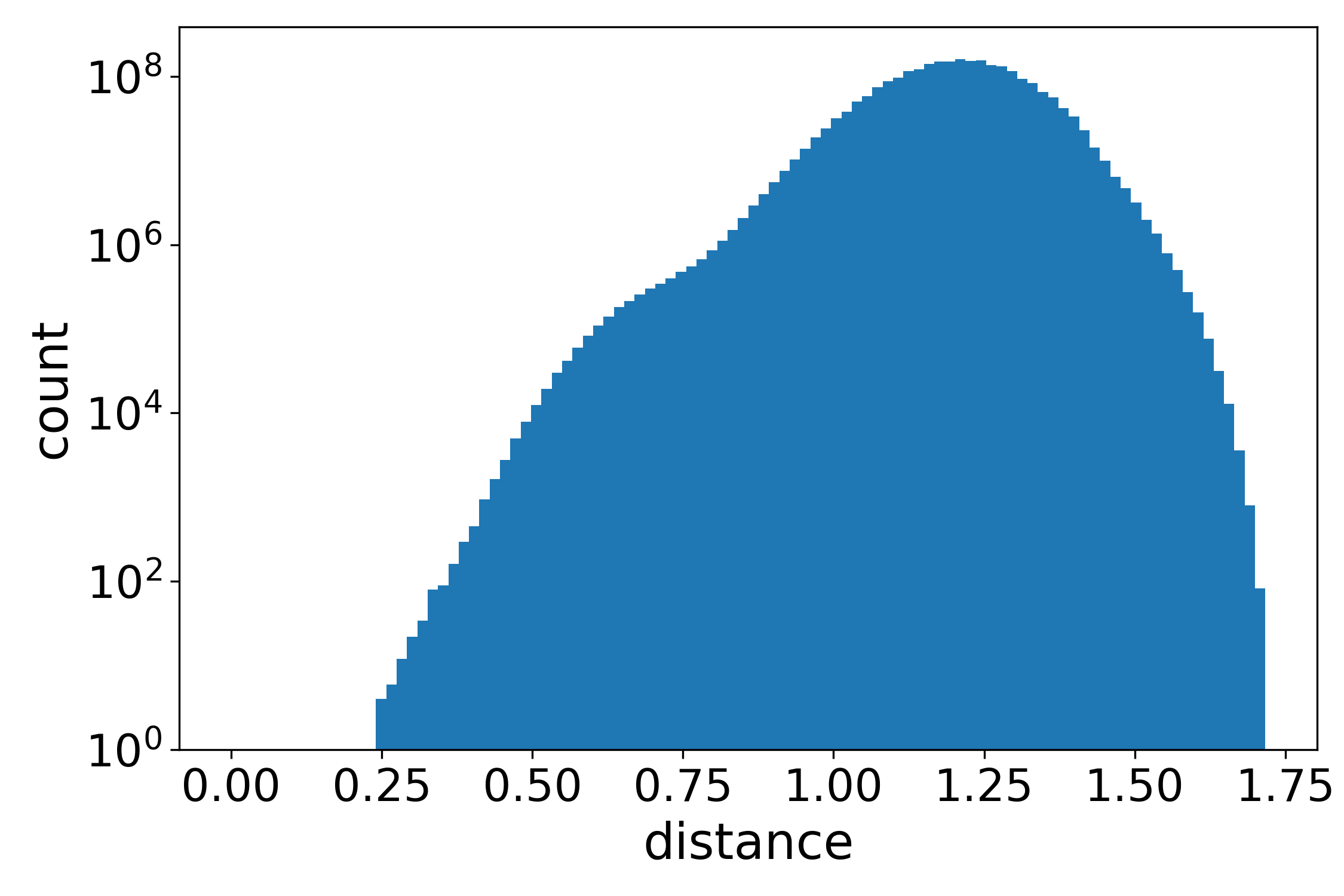}
\end{minipage}
\caption{Distribution of context-to-context normalized Euclidean distances for Pythia on Pile10k (left), Bookcorpus (middle) and WikiText-103 (right). We see concentration of distances in the sense that there is a gap from zero to the lowest distance values. Here, we do not include the distance of a context to itself, since it will always be zero for this distance measure.}
\label{fig:cc_norm_euc_dist_distribution_pythia}
\end{figure*}

\begin{figure*}[htb]
\begin{minipage}[b]{0.33\linewidth}
\centering
\includegraphics[width=\textwidth]{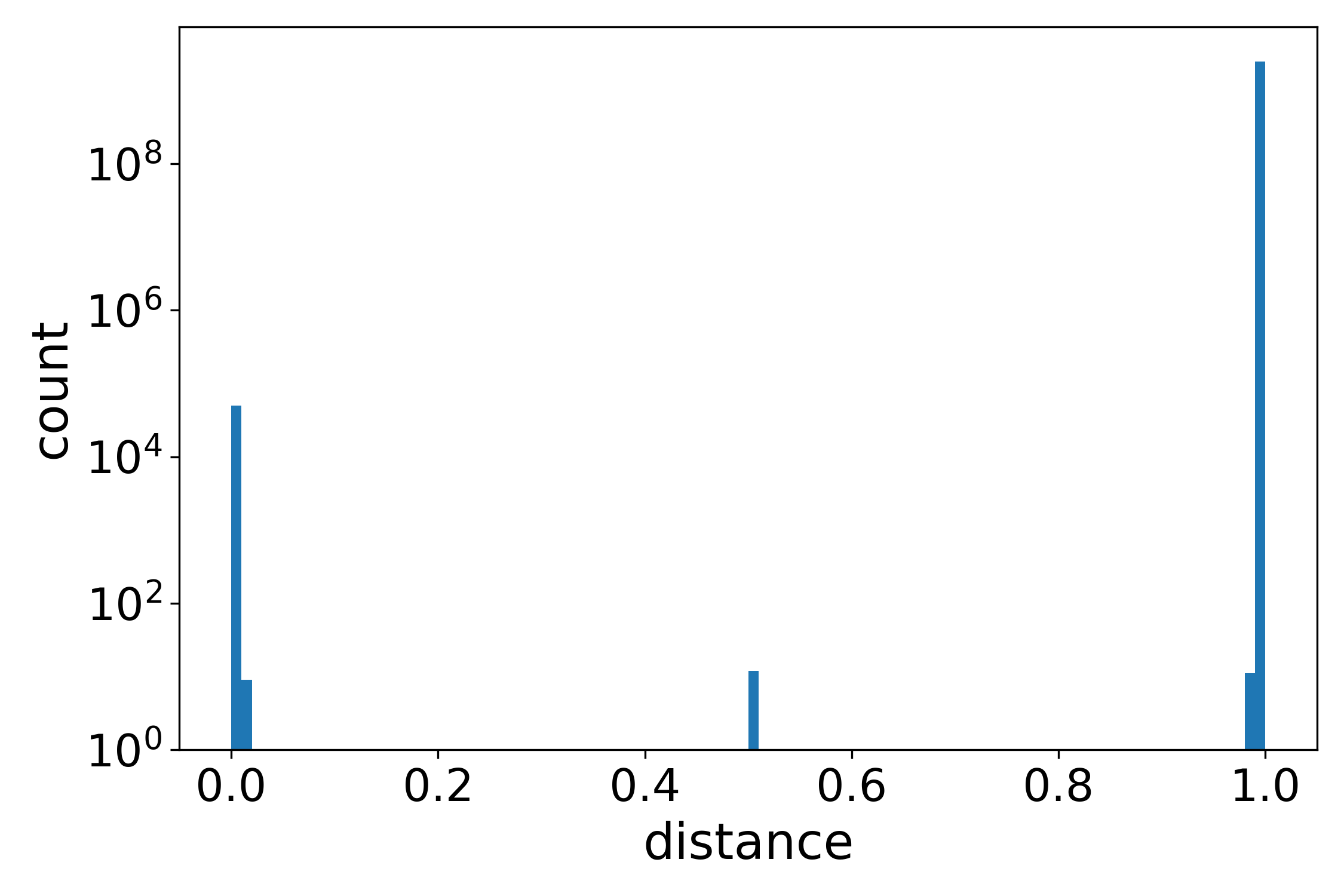}
\end{minipage}
\hspace{-0.1cm}
\begin{minipage}[b]{0.33\linewidth}
\centering
\includegraphics[width=\textwidth]{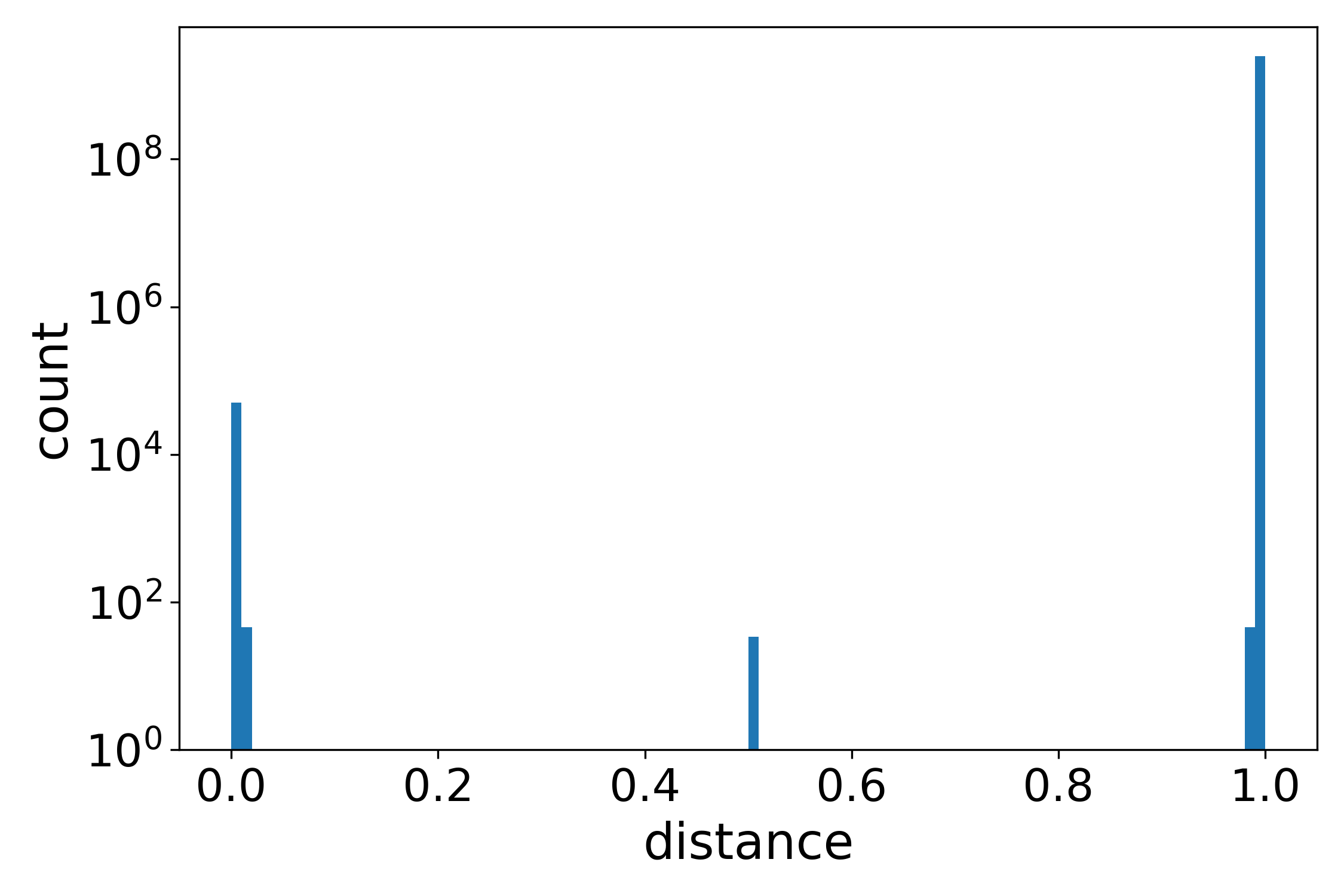}
\end{minipage}
\hspace{-0.1cm}
\begin{minipage}[b]{0.33\linewidth}
\centering
\includegraphics[width=\textwidth]{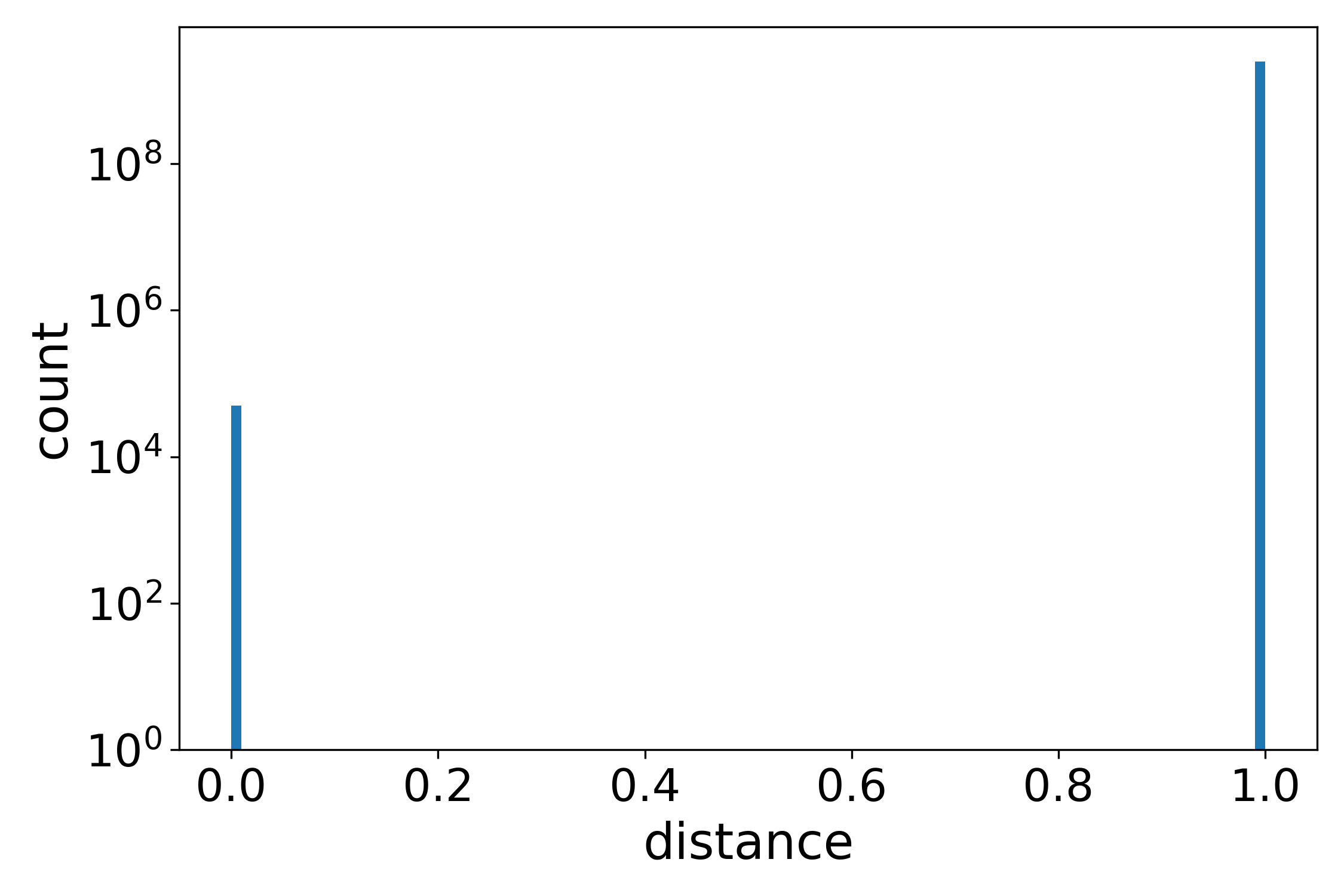}
\end{minipage}
\caption{Distribution of context-to-context softmaxed dot product distances for Pythia on Pile10k (left), Bookcorpus (middle) and WikiText-103 (right). Here we have included the distance of a context to itself. Note that, when using a dot product, there is no guarantee that a context will get the largest score with itself. Pythia is the only tested model which has distances between 0 and 1.}
\label{fig:cc_softmax_dot_distribution_pythia}
\end{figure*}

\begin{figure*}[htb]
\begin{minipage}[b]{0.33\linewidth}
\centering
\includegraphics[width=\textwidth]{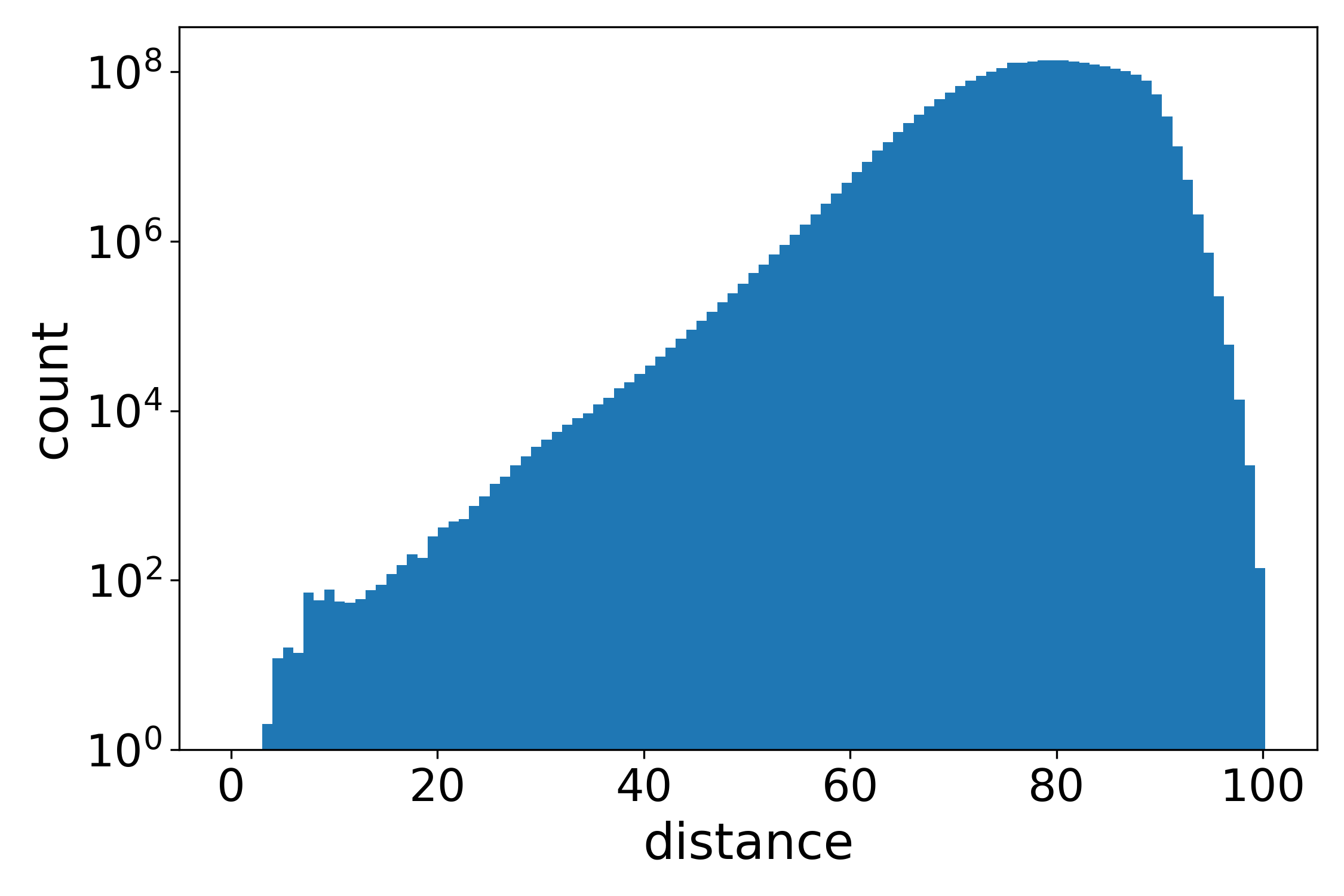}
\end{minipage}
\hspace{-0.1cm}
\begin{minipage}[b]{0.33\linewidth}
\centering
\includegraphics[width=\textwidth]{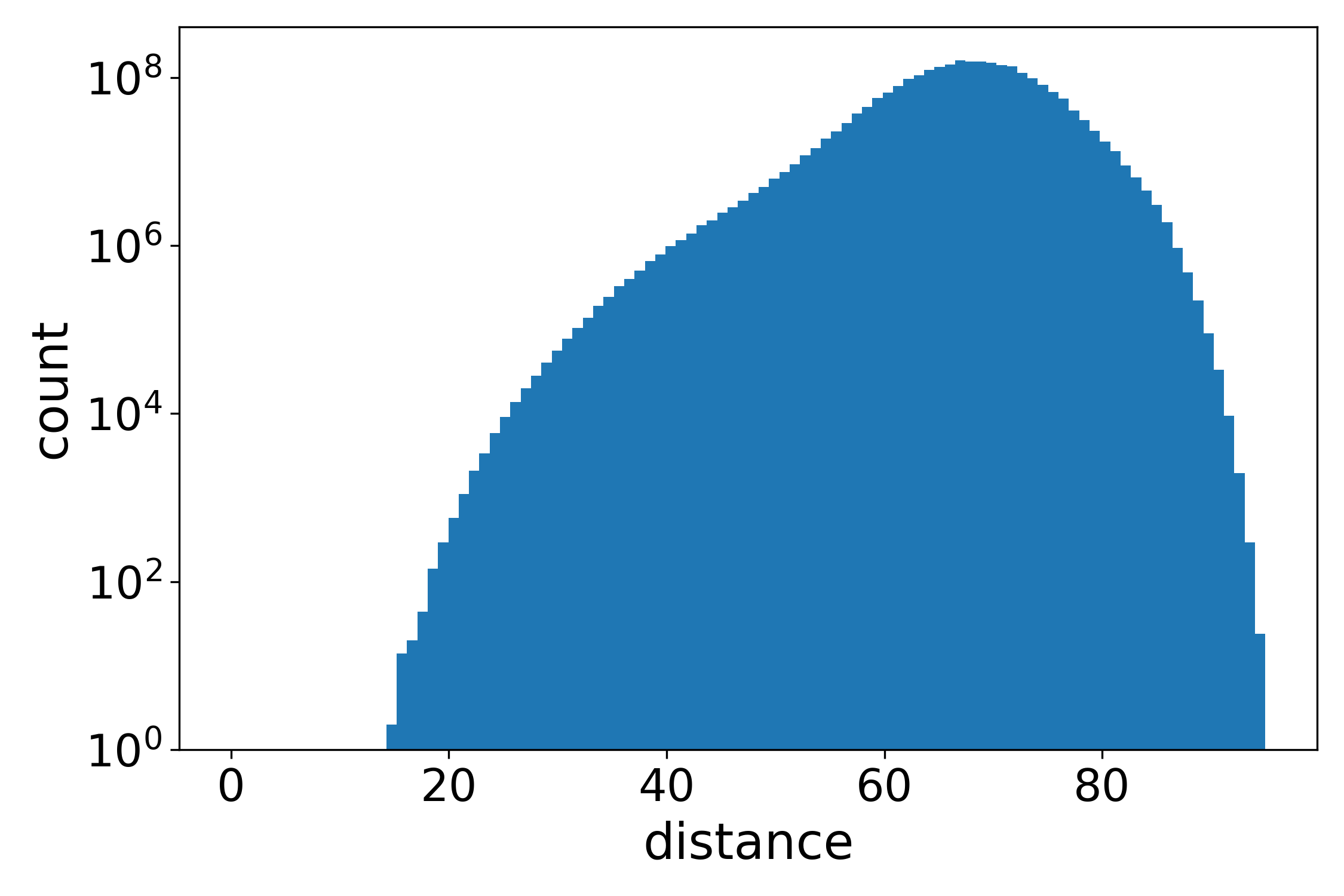}
\end{minipage}
\hspace{-0.1cm}
\begin{minipage}[b]{0.33\linewidth}
\centering
\includegraphics[width=\textwidth]{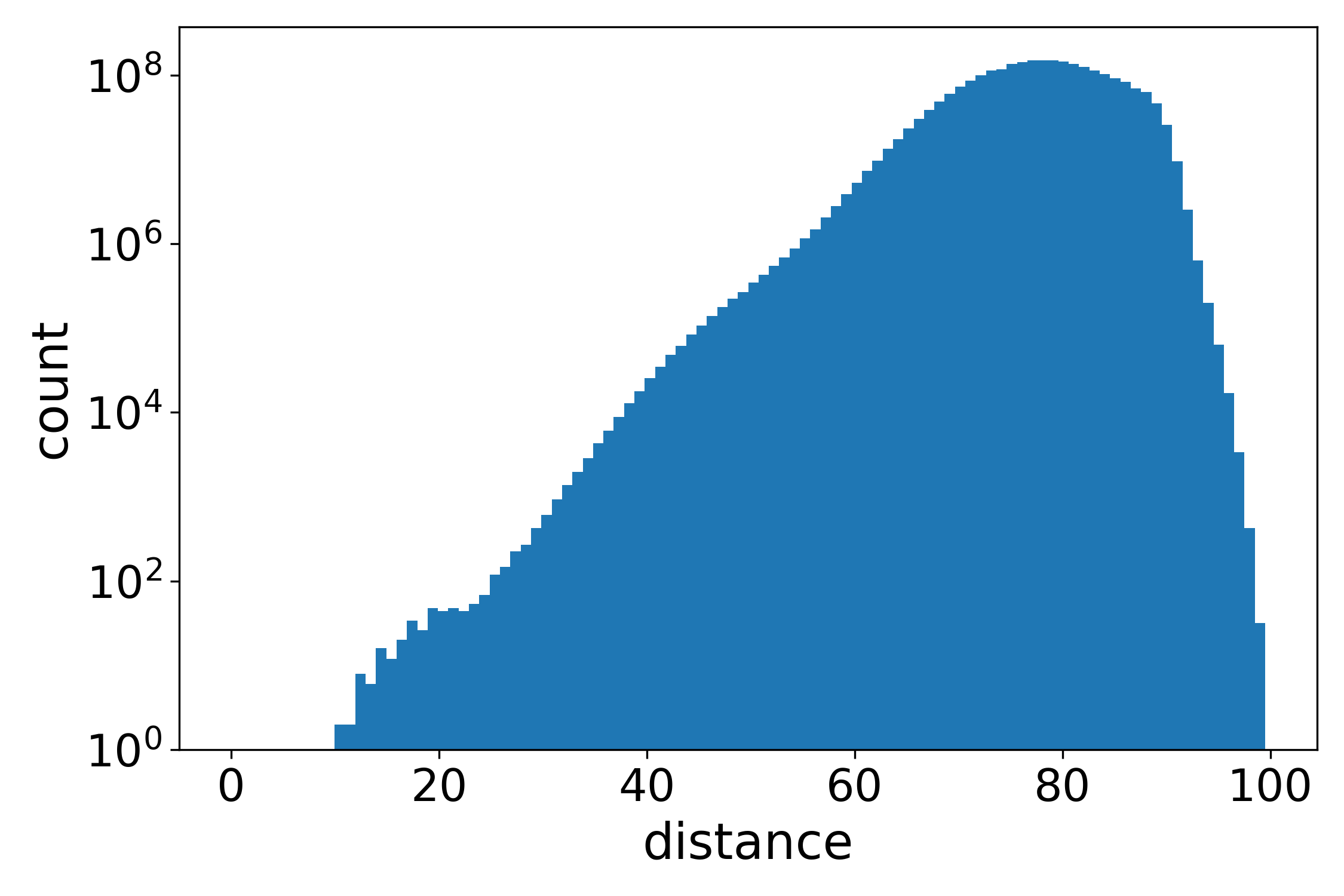}
\end{minipage}
\caption{Distribution of context-to-context Euclidean distances for Opt on Pile10k (left), Bookcorpus (middle) and WikiText-103 (right). We see concentration of distances in the sense that there is a gap from zero to the lowest distance values. Here, we do not include the distance of a context to itself, since it will always be zero for this distance measure.}
\label{fig:cc_euc_dist_distribution_opt}
\end{figure*}

\begin{figure*}[htb]
\begin{minipage}[b]{0.33\linewidth}
\centering
\includegraphics[width=\textwidth]{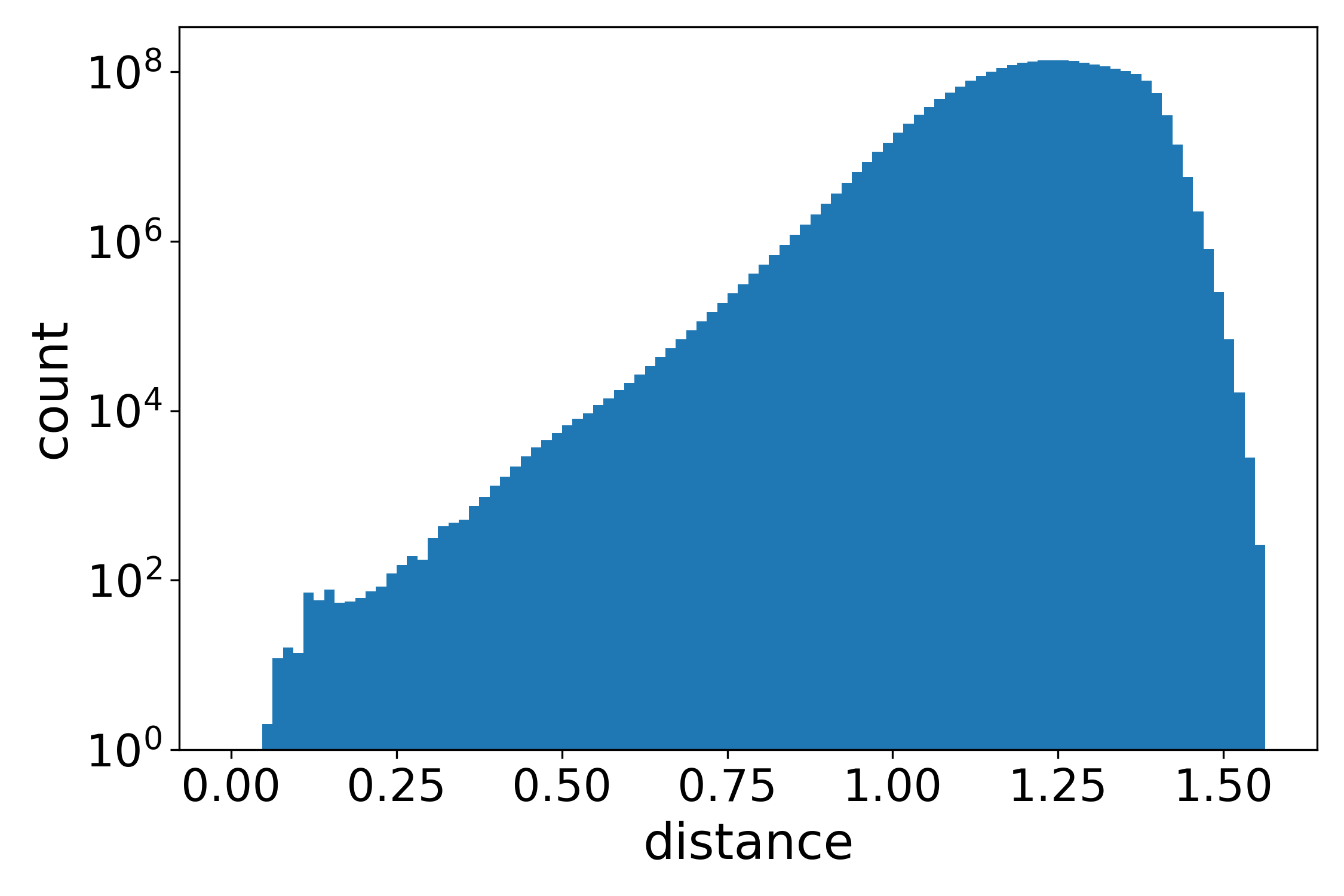}
\end{minipage}
\hspace{-0.1cm}
\begin{minipage}[b]{0.33\linewidth}
\centering
\includegraphics[width=\textwidth]{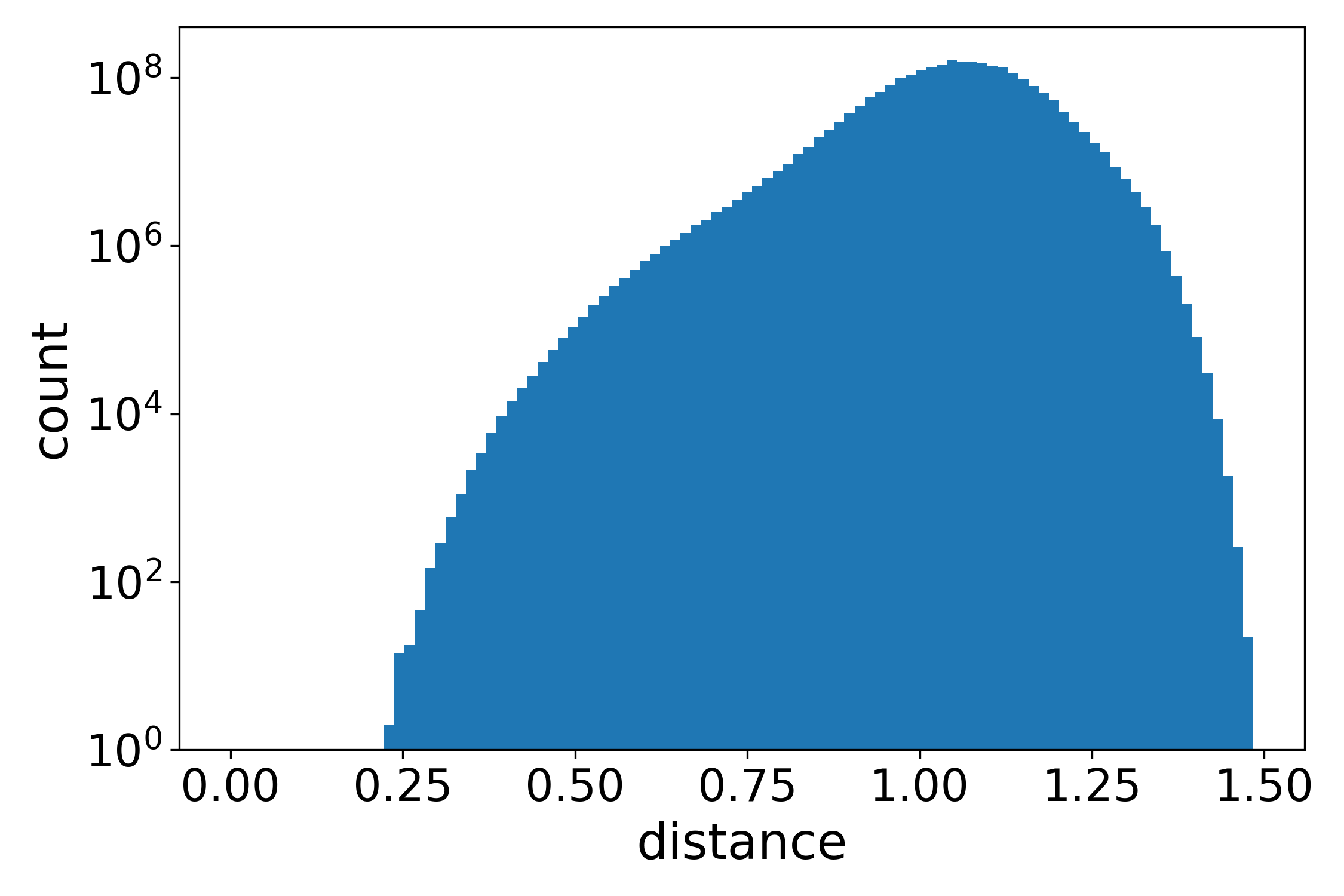}
\end{minipage}
\hspace{-0.1cm}
\begin{minipage}[b]{0.33\linewidth}
\centering
\includegraphics[width=\textwidth]{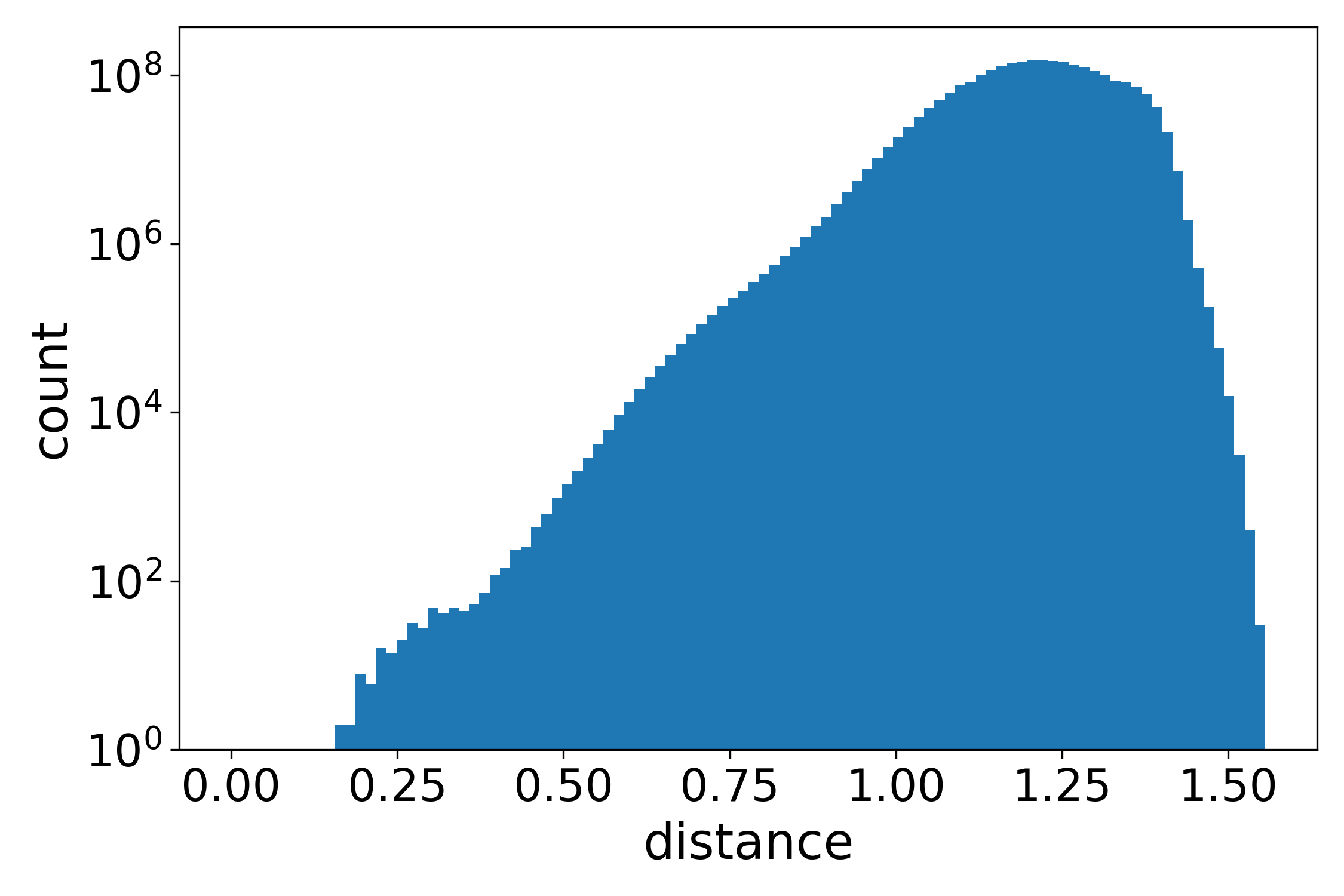}
\end{minipage}
\caption{Distribution of context-to-context normalized Euclidean distances for Opt on Pile10k (left), Bookcorpus (middle) and WikiText-103 (right). We see concentration of distances in the sense that there is a gap from zero to the lowest distance values. Here, we do not include the distance of a context to itself, since it will always be zero for this distance measure.}
\label{fig:cc_norm_euc_dist_distribution_opt}
\end{figure*}

\begin{figure*}[htb]
\begin{minipage}[b]{0.33\linewidth}
\centering
\includegraphics[width=\textwidth]{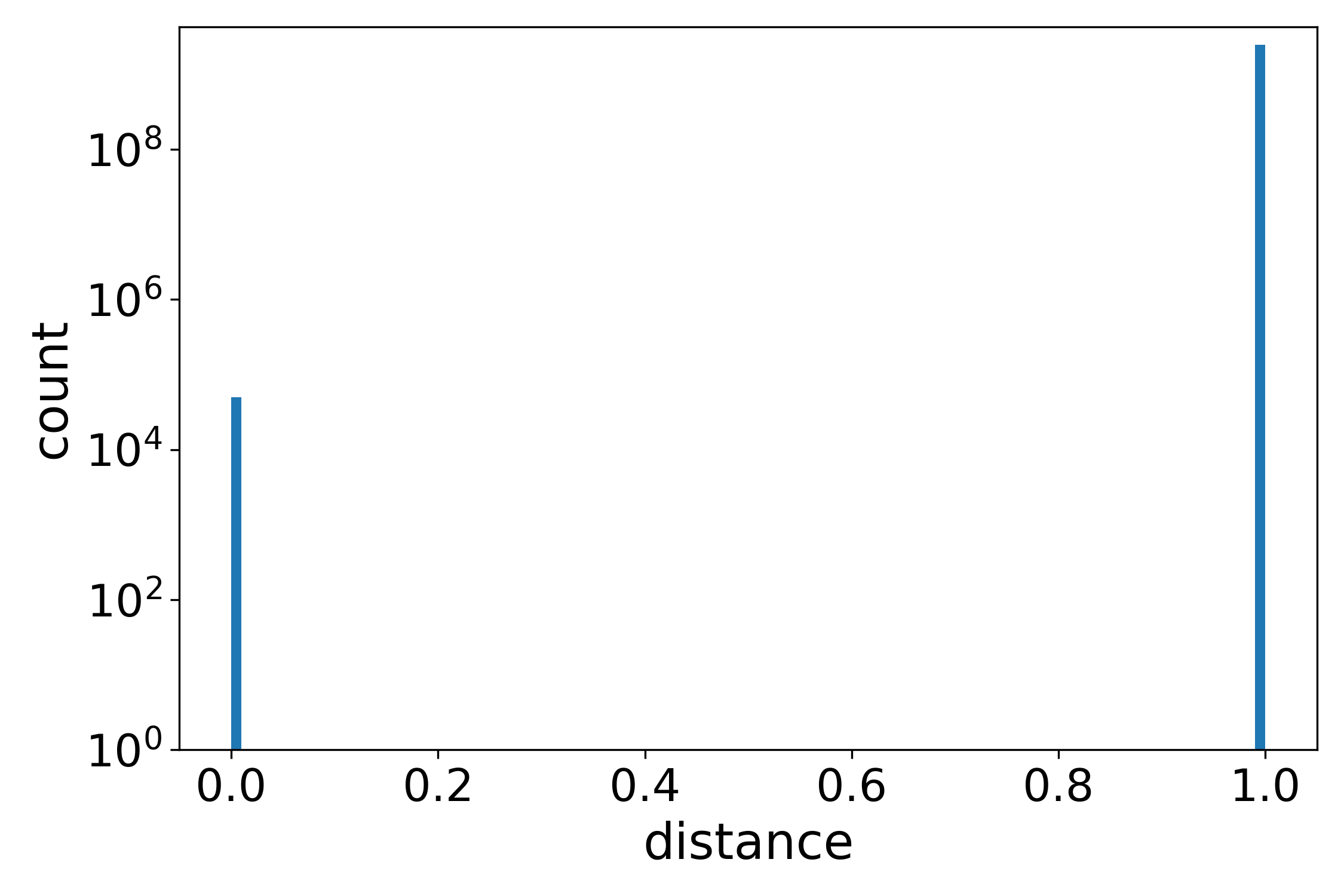}
\end{minipage}
\hspace{-0.1cm}
\begin{minipage}[b]{0.33\linewidth}
\centering
\includegraphics[width=\textwidth]{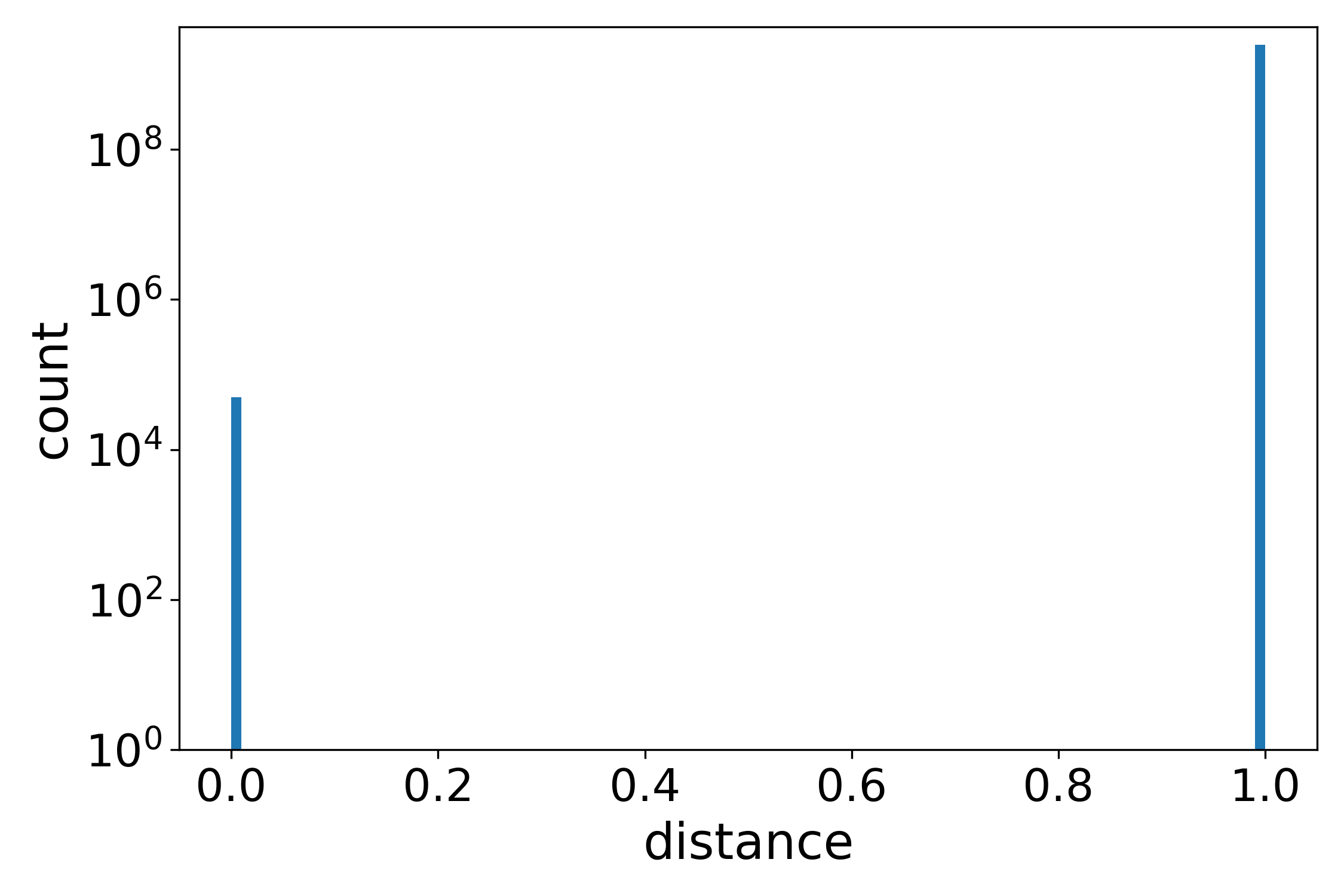}
\end{minipage}
\hspace{-0.1cm}
\begin{minipage}[b]{0.33\linewidth}
\centering
\includegraphics[width=\textwidth]{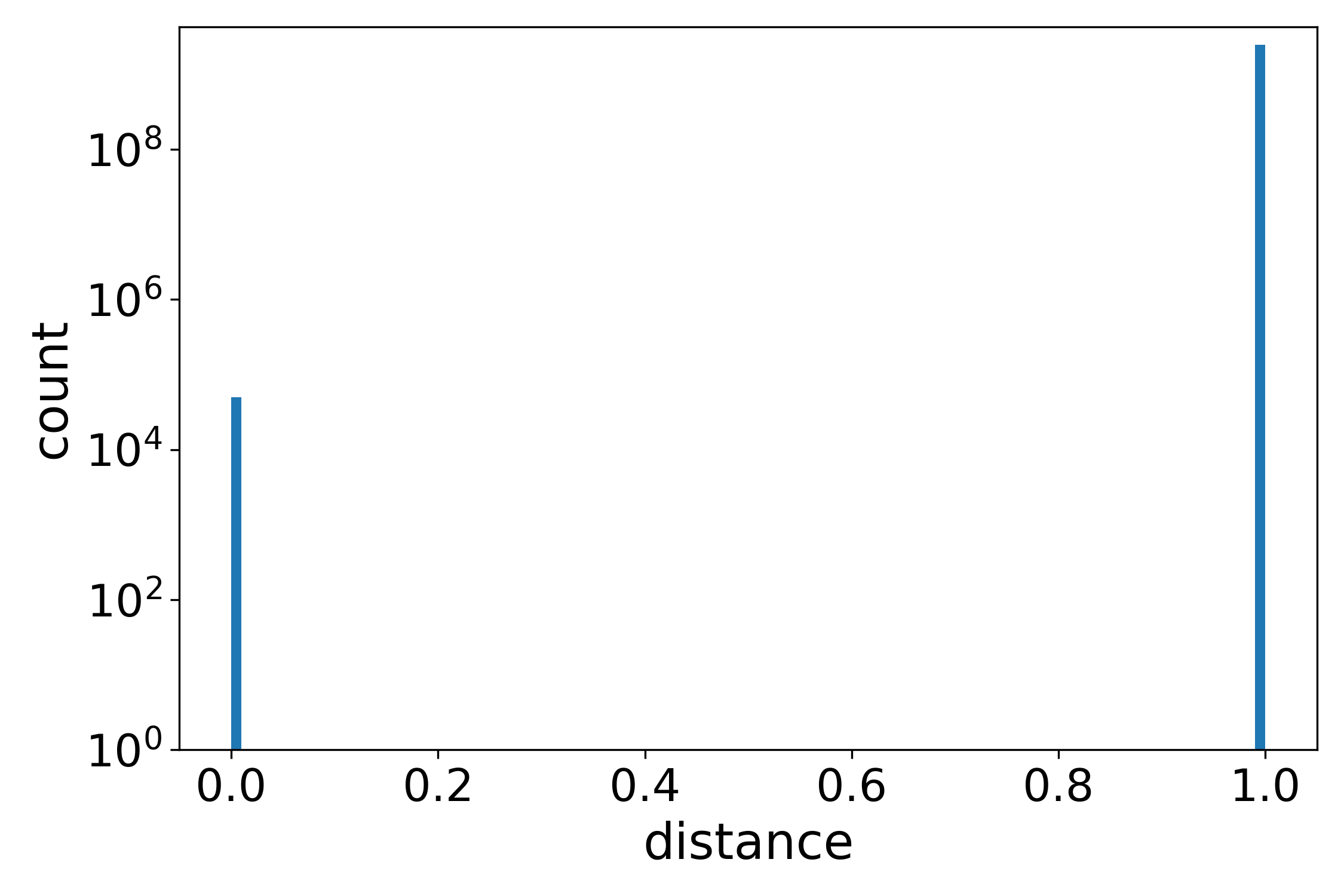}
\end{minipage}
\caption{Distribution of context-to-context softmaxed dot product distances for Opt on Pile10k (left), Bookcorpus (middle) and WikiText-103 (right). Here we have included the distance of a context to itself, which is the spike at zero. Note that, when using a dot product, there is no guarantee that a context will get the largest score with itself.}
\label{fig:cc_softmax_dot_distribution_opt}
\end{figure*}

\begin{figure*}[htb]
\begin{minipage}[b]{0.33\linewidth}
\centering
\includegraphics[width=\textwidth]{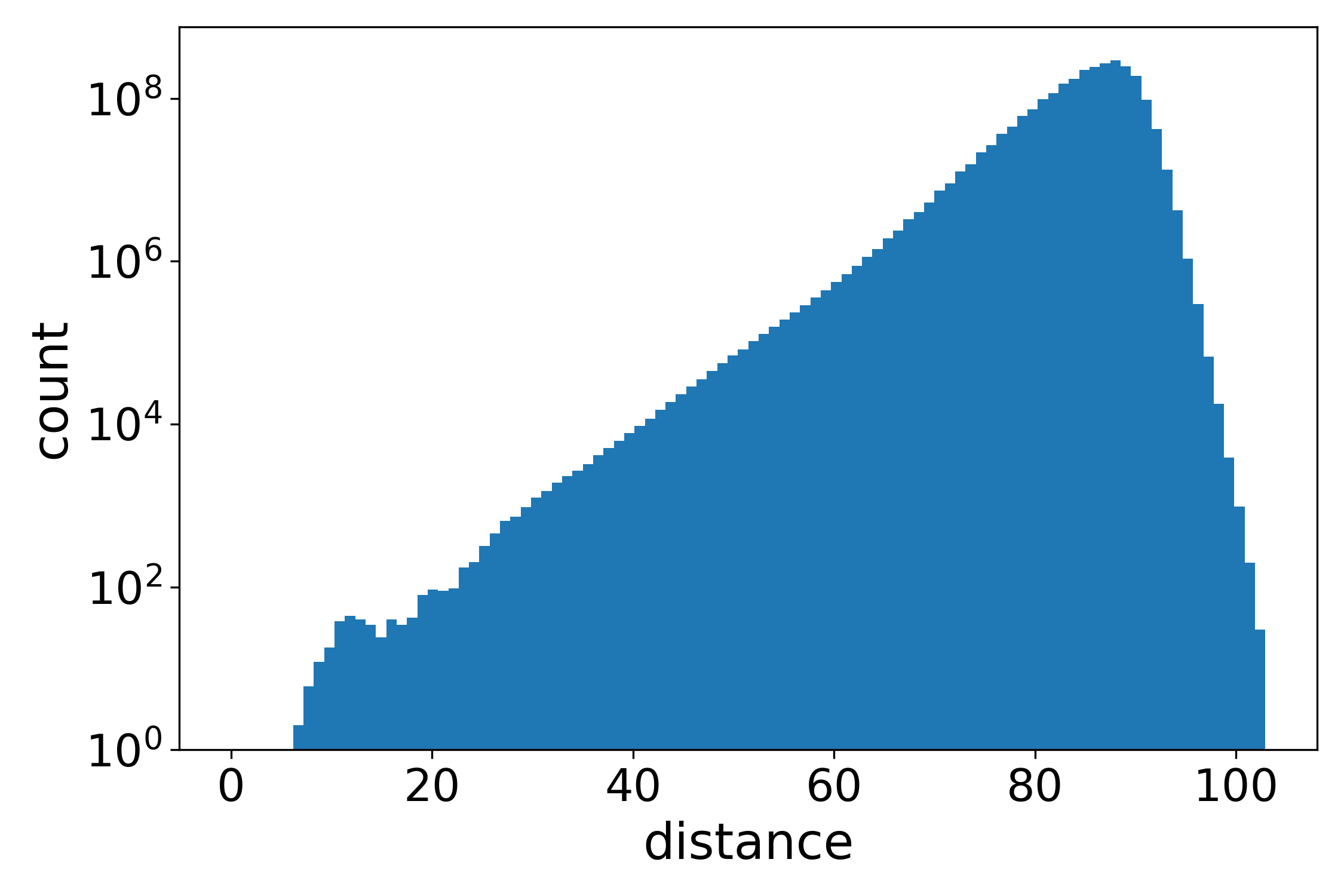}
\end{minipage}
\hspace{-0.1cm}
\begin{minipage}[b]{0.33\linewidth}
\centering
\includegraphics[width=\textwidth]{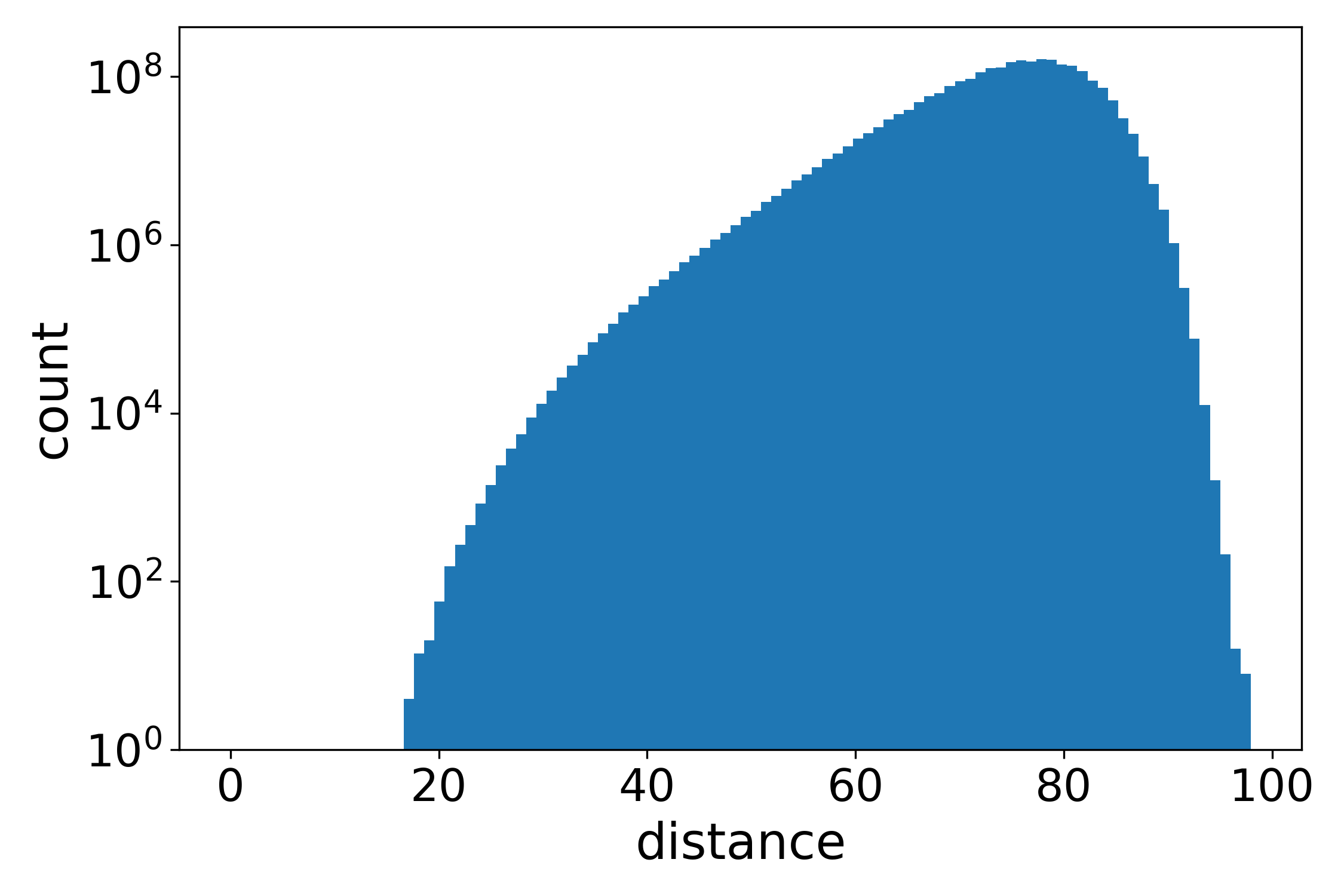}
\end{minipage}
\hspace{-0.1cm}
\begin{minipage}[b]{0.33\linewidth}
\centering
\includegraphics[width=\textwidth]{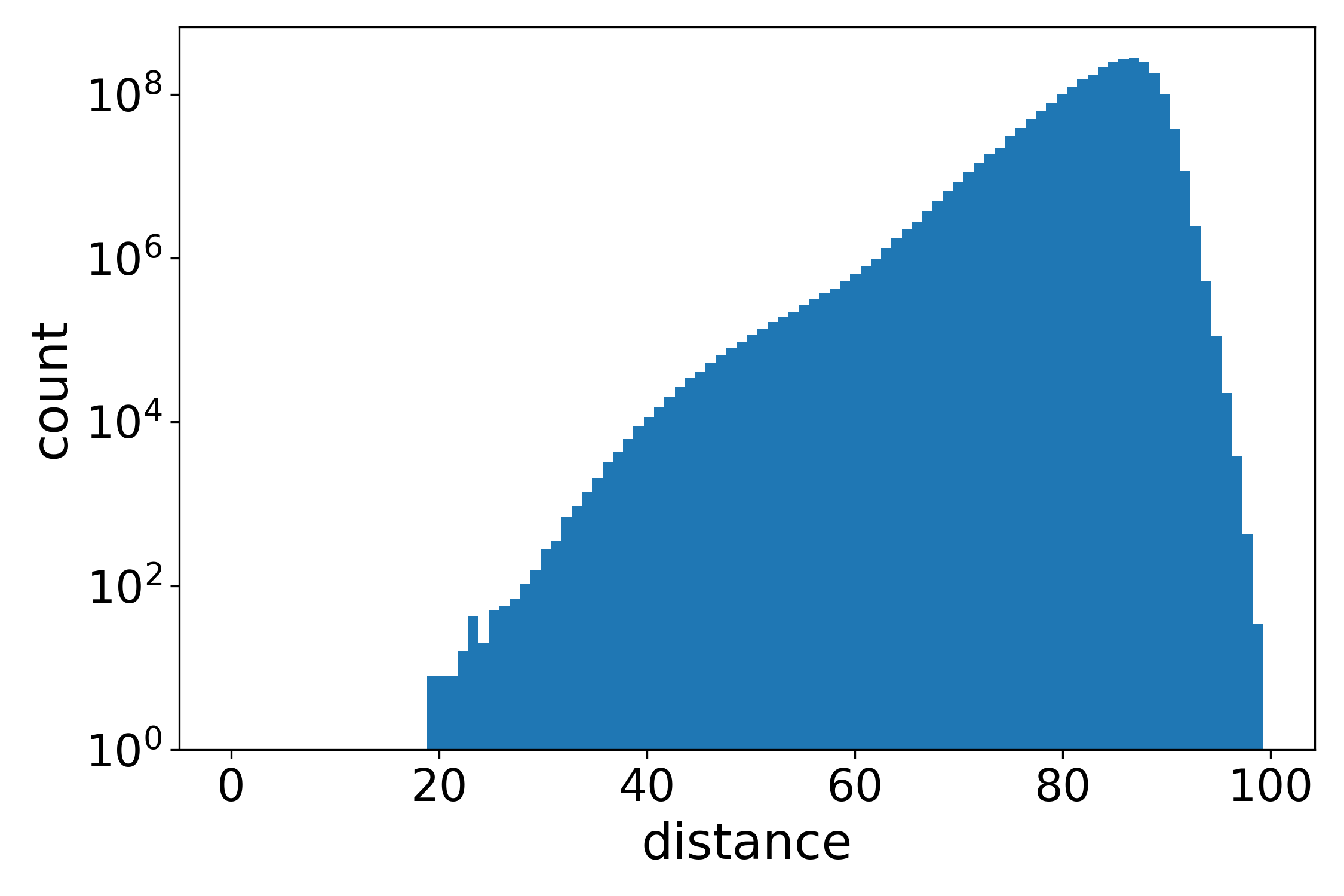}
\end{minipage}
\caption{Distribution of context-to-context Euclidean distances for Olmo on Pile10k (left), Bookcorpus (middle) and WikiText-103 (right). We see concentration of distances in the sense that there is a gap from zero to the lowest distance values. Here, we do not include the distance of a context to itself, since it will always be zero for this distance measure.}
\label{fig:cc_euc_dist_distribution_olmo}
\end{figure*}

\begin{figure*}[htb]
\begin{minipage}[b]{0.33\linewidth}
\centering
\includegraphics[width=\textwidth]{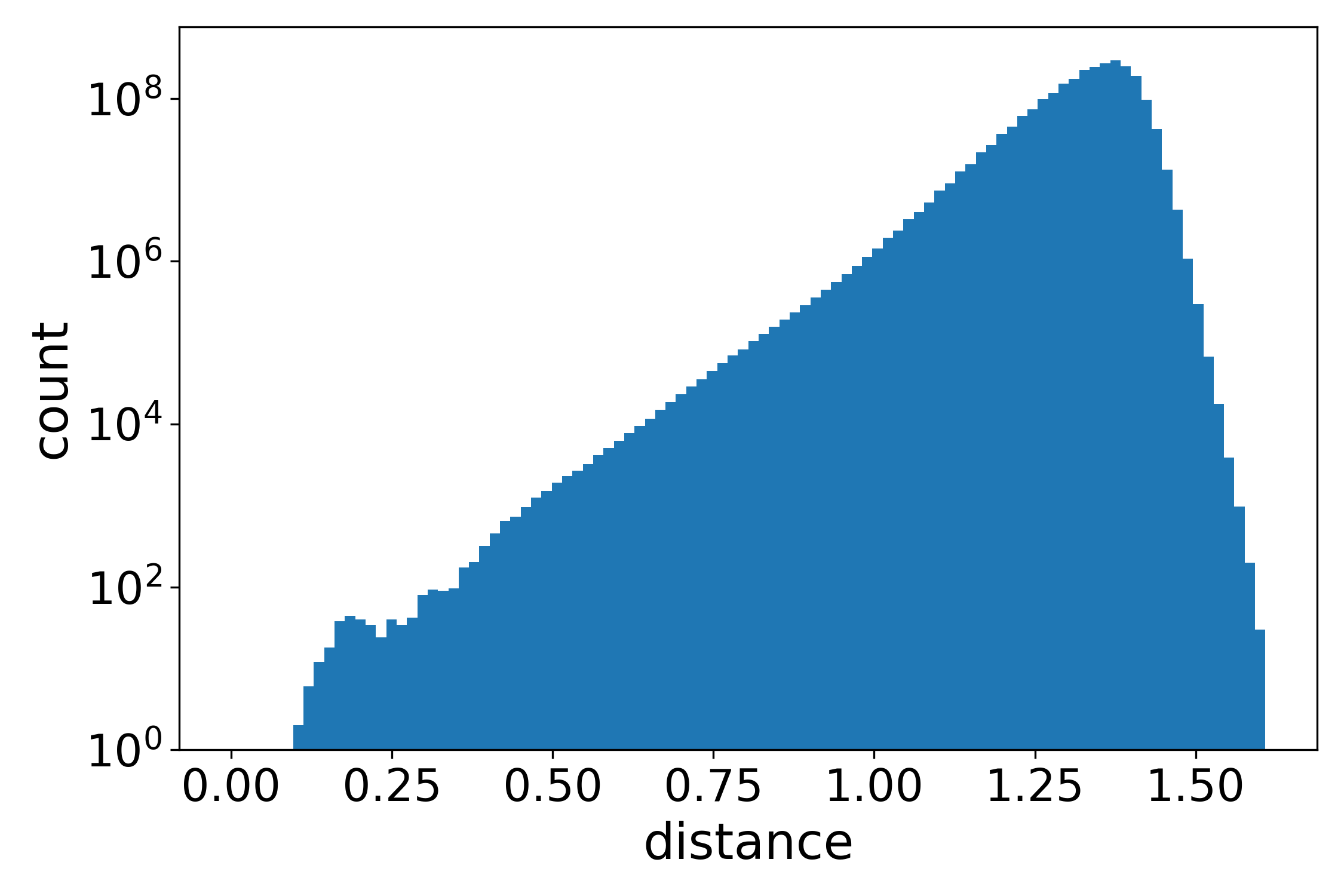}
\end{minipage}
\hspace{-0.1cm}
\begin{minipage}[b]{0.33\linewidth}
\centering
\includegraphics[width=\textwidth]{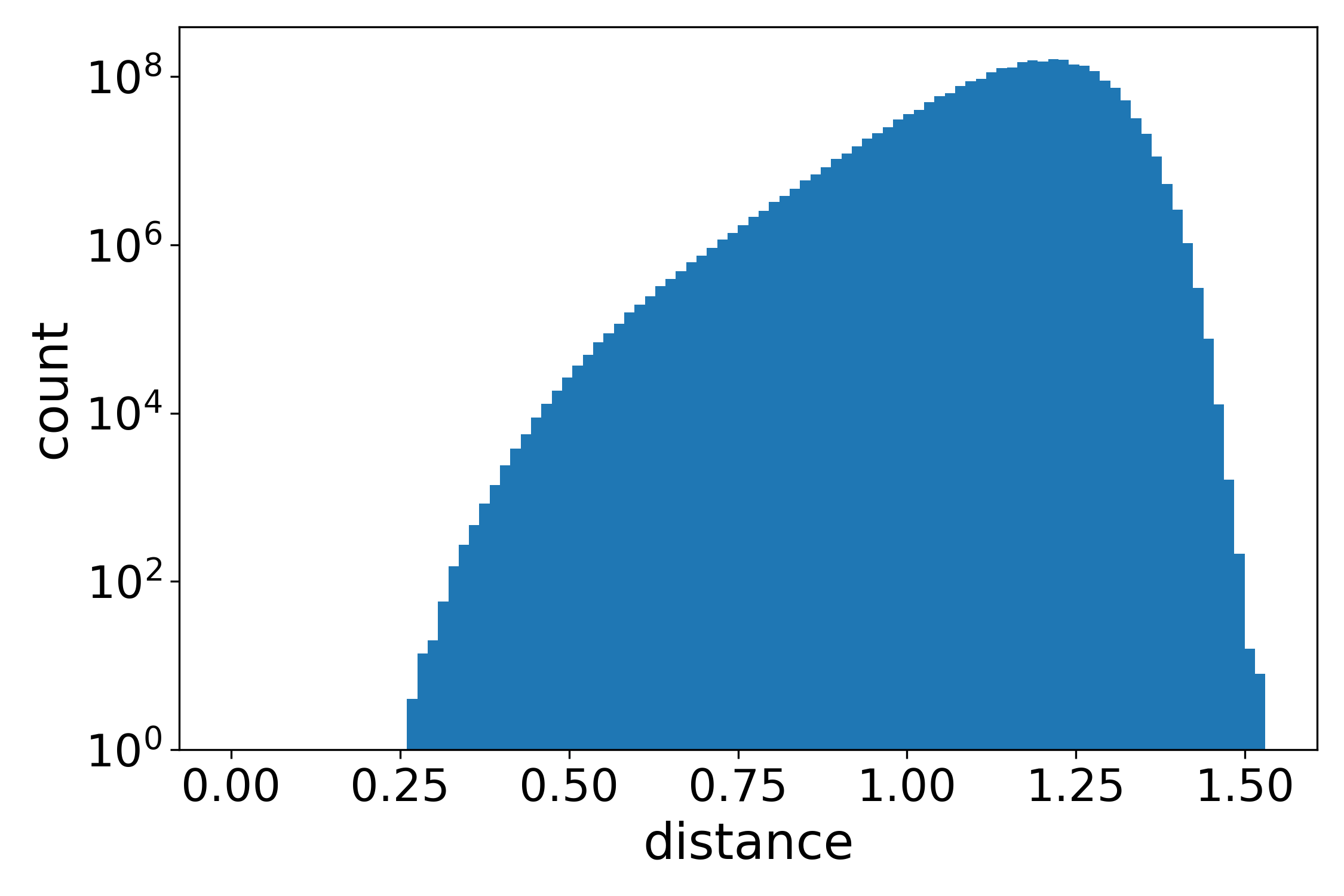}
\end{minipage}
\hspace{-0.1cm}
\begin{minipage}[b]{0.33\linewidth}
\centering
\includegraphics[width=\textwidth]{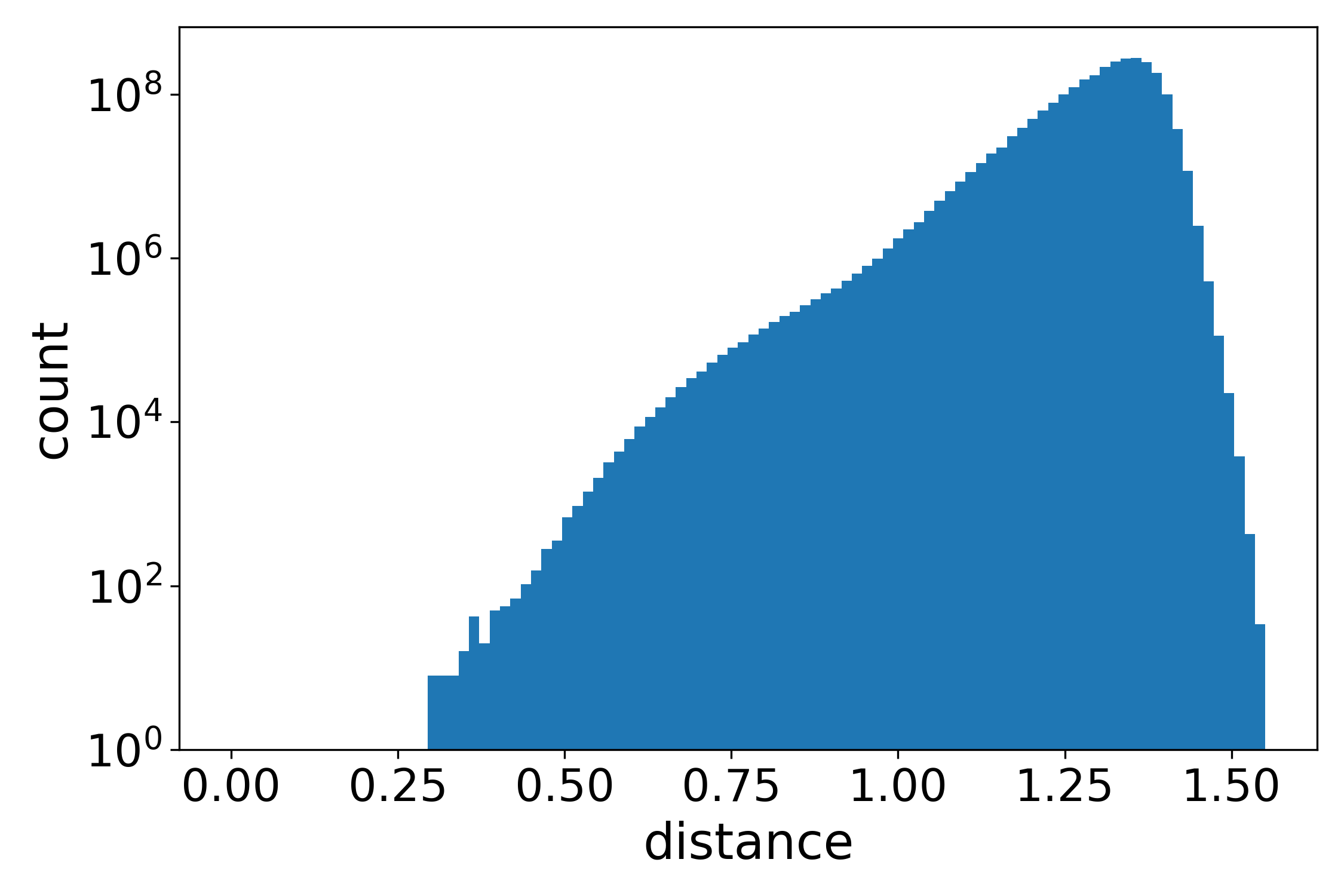}
\end{minipage}
\caption{Distribution of context-to-context normalized Euclidean distances for Olmo on Pile10k (left), Bookcorpus (middle) and WikiText-103 (right). We see concentration of distances in the sense that there is a gap from zero to the lowest distance values. Here, we do not include the distance of a context to itself, since it will always be zero for this distance measure.}
\label{fig:cc_norm_euc_dist_distribution_olmo}
\end{figure*}

\begin{figure*}[htb]
\begin{minipage}[b]{0.33\linewidth}
\centering
\includegraphics[width=\textwidth]{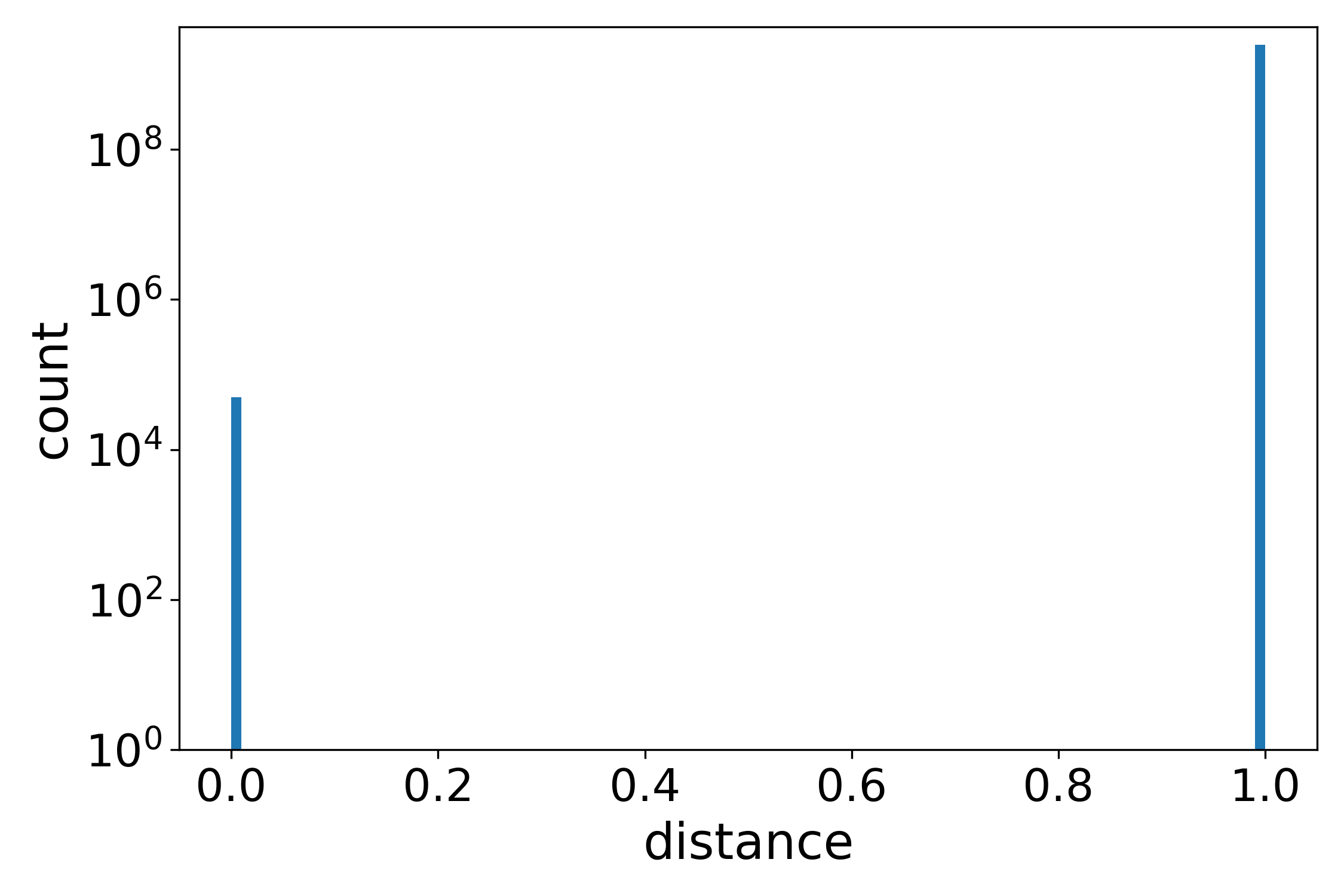}
\end{minipage}
\hspace{-0.1cm}
\begin{minipage}[b]{0.33\linewidth}
\centering
\includegraphics[width=\textwidth]{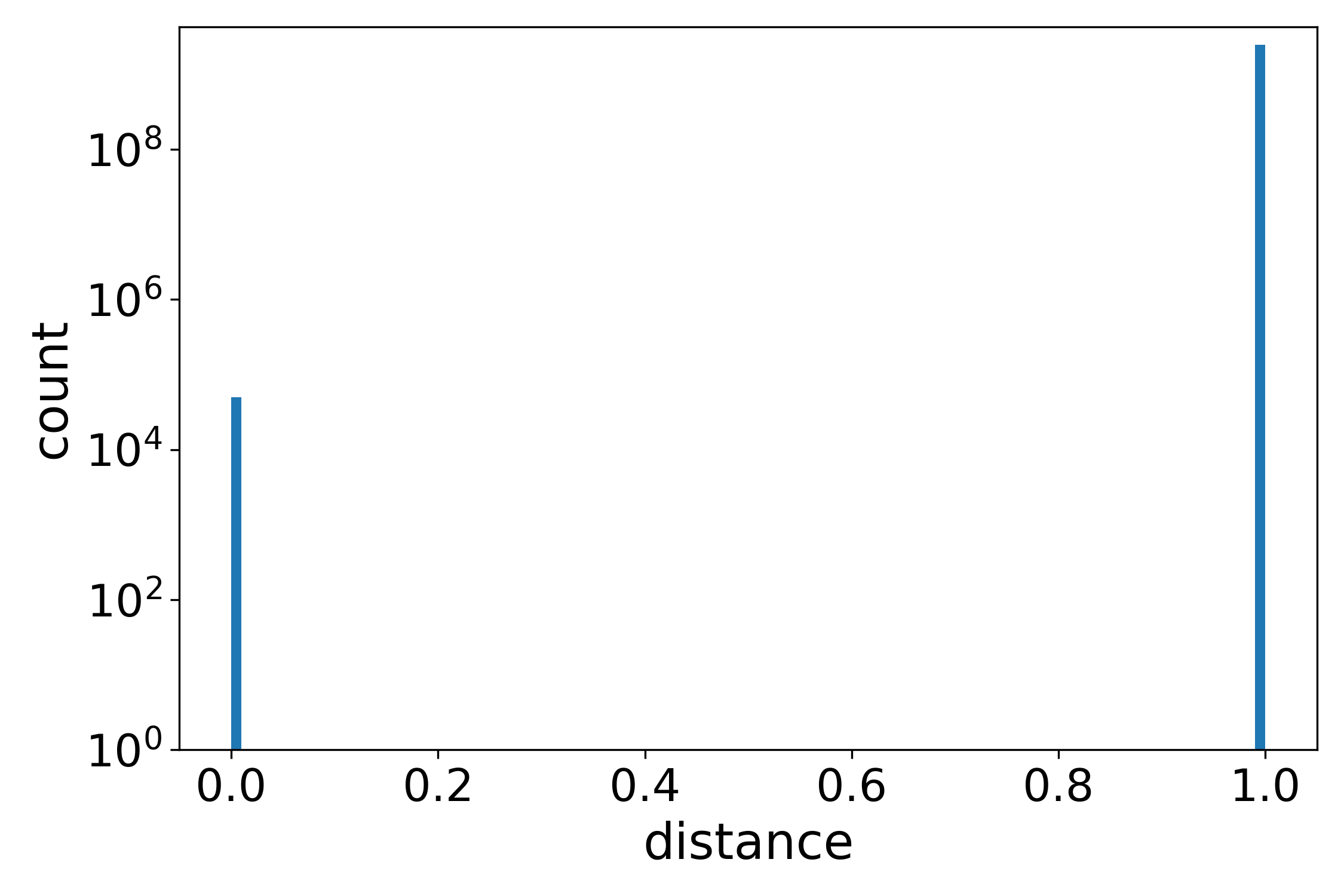}
\end{minipage}
\hspace{-0.1cm}
\begin{minipage}[b]{0.33\linewidth}
\centering
\includegraphics[width=\textwidth]{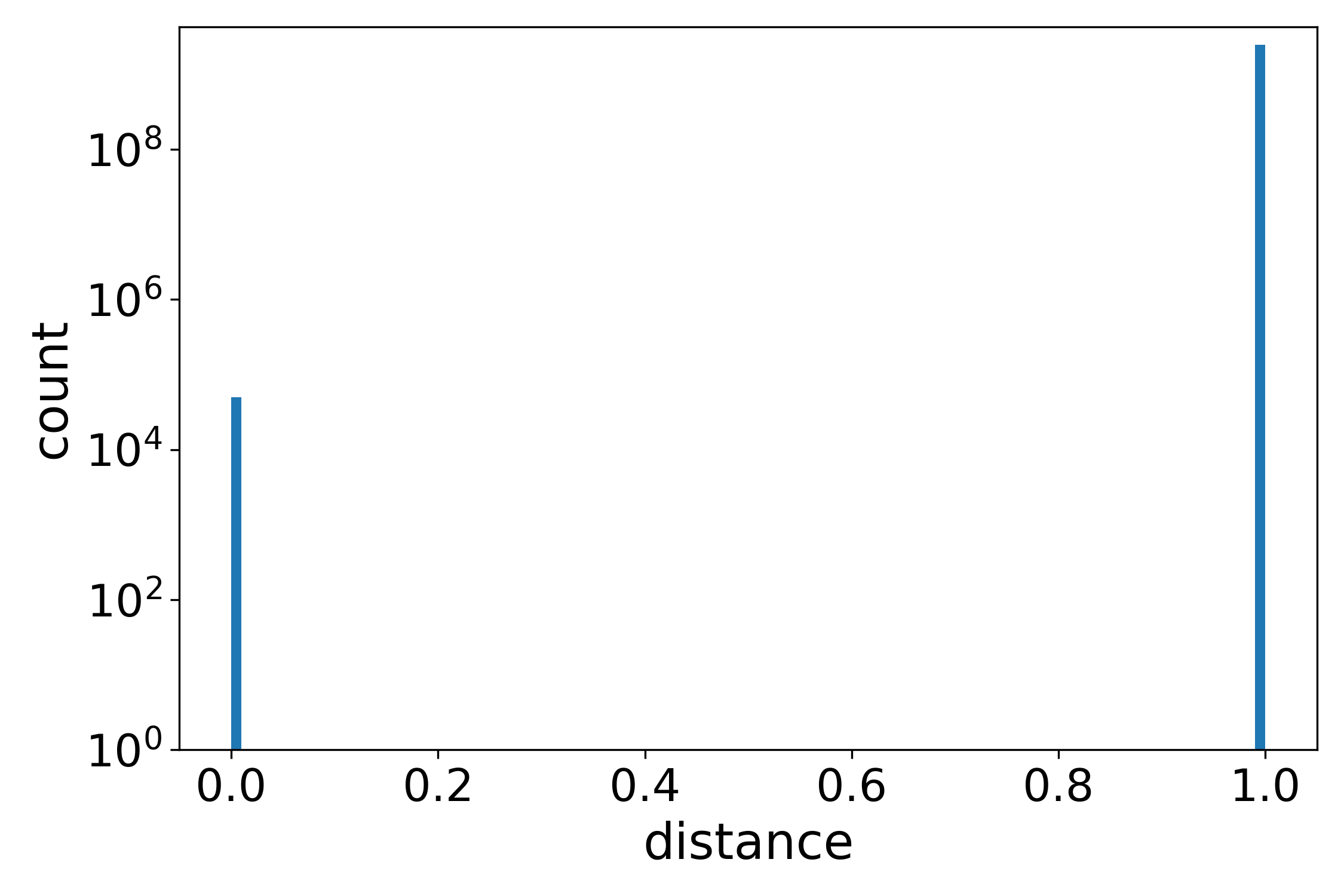}
\end{minipage}
\caption{Distribution of context-to-context softmaxed dot product distances for Olmo on Pile10k (left), Bookcorpus (middle) and WikiText-103 (right). Here we have included the distance of a context to itself, which is the spike at zero. Note that, when using a dot product, there is no guarantee that a context will get the largest score with itself.}
\label{fig:cc_softmax_dot_distribution_olmo}
\end{figure*}

\begin{figure*}[htb]
\begin{minipage}[b]{0.33\linewidth}
\centering
\includegraphics[width=\textwidth]{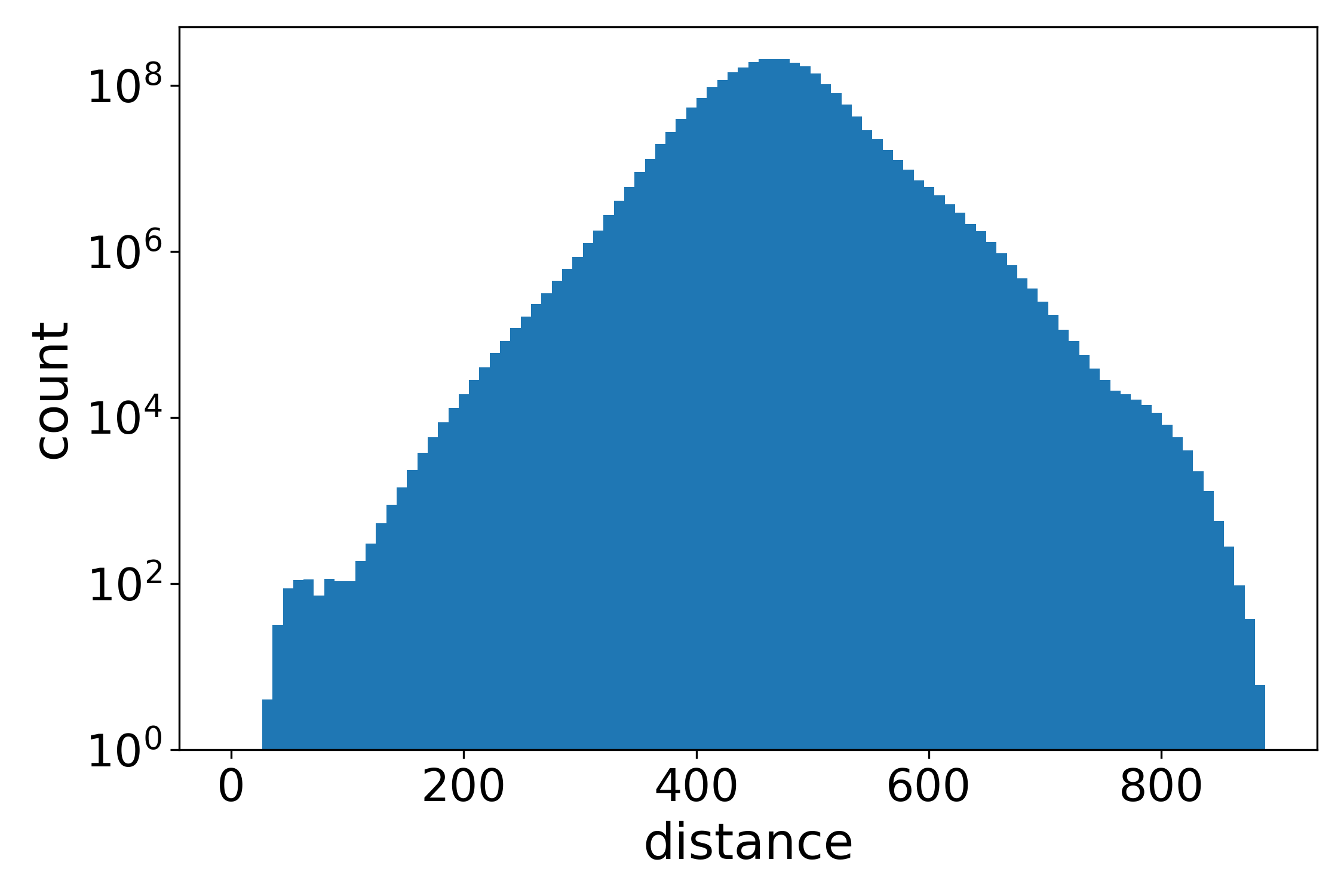}
\end{minipage}
\hspace{-0.1cm}
\begin{minipage}[b]{0.33\linewidth}
\centering
\includegraphics[width=\textwidth]{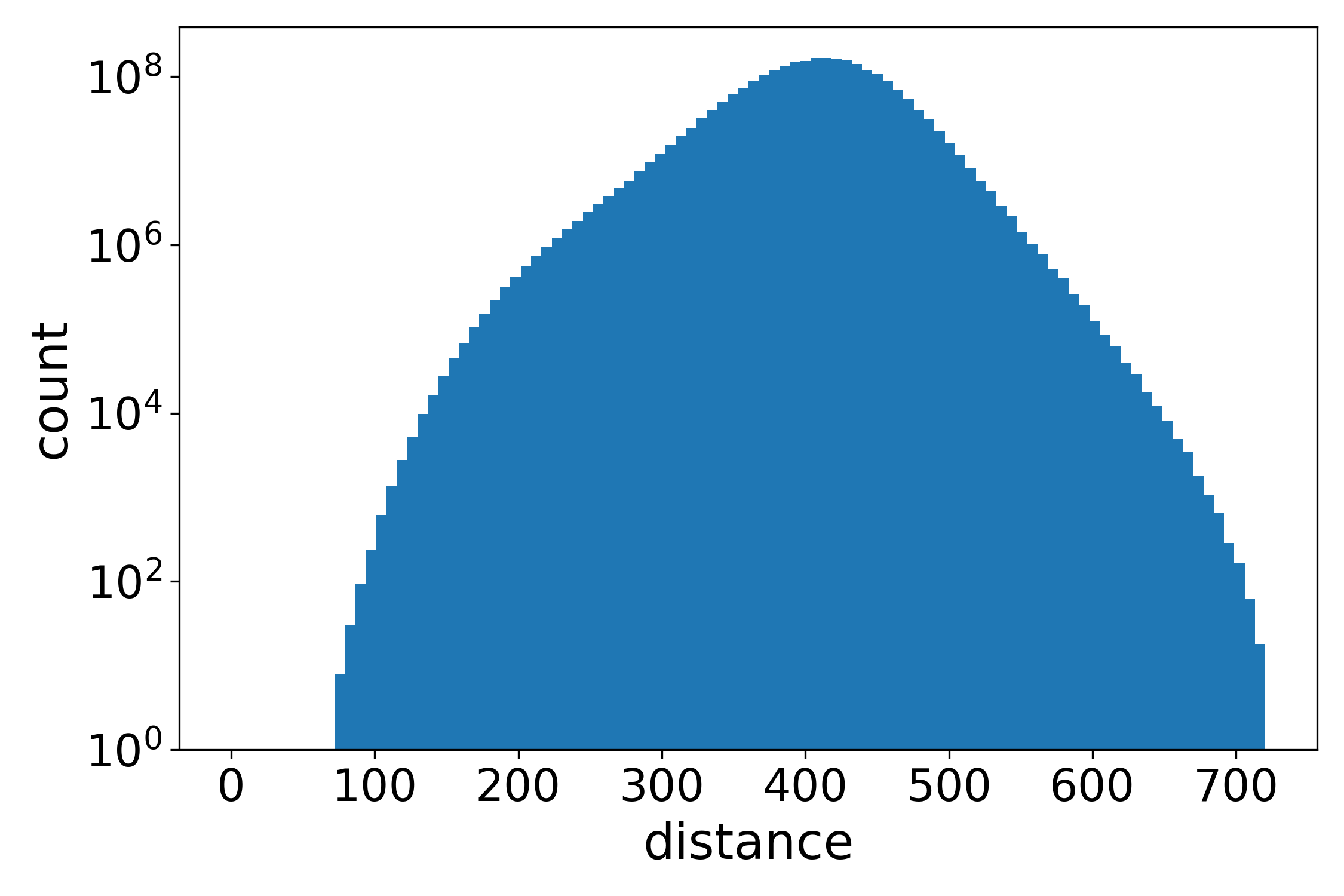}
\end{minipage}
\hspace{-0.1cm}
\begin{minipage}[b]{0.33\linewidth}
\centering
\includegraphics[width=\textwidth]{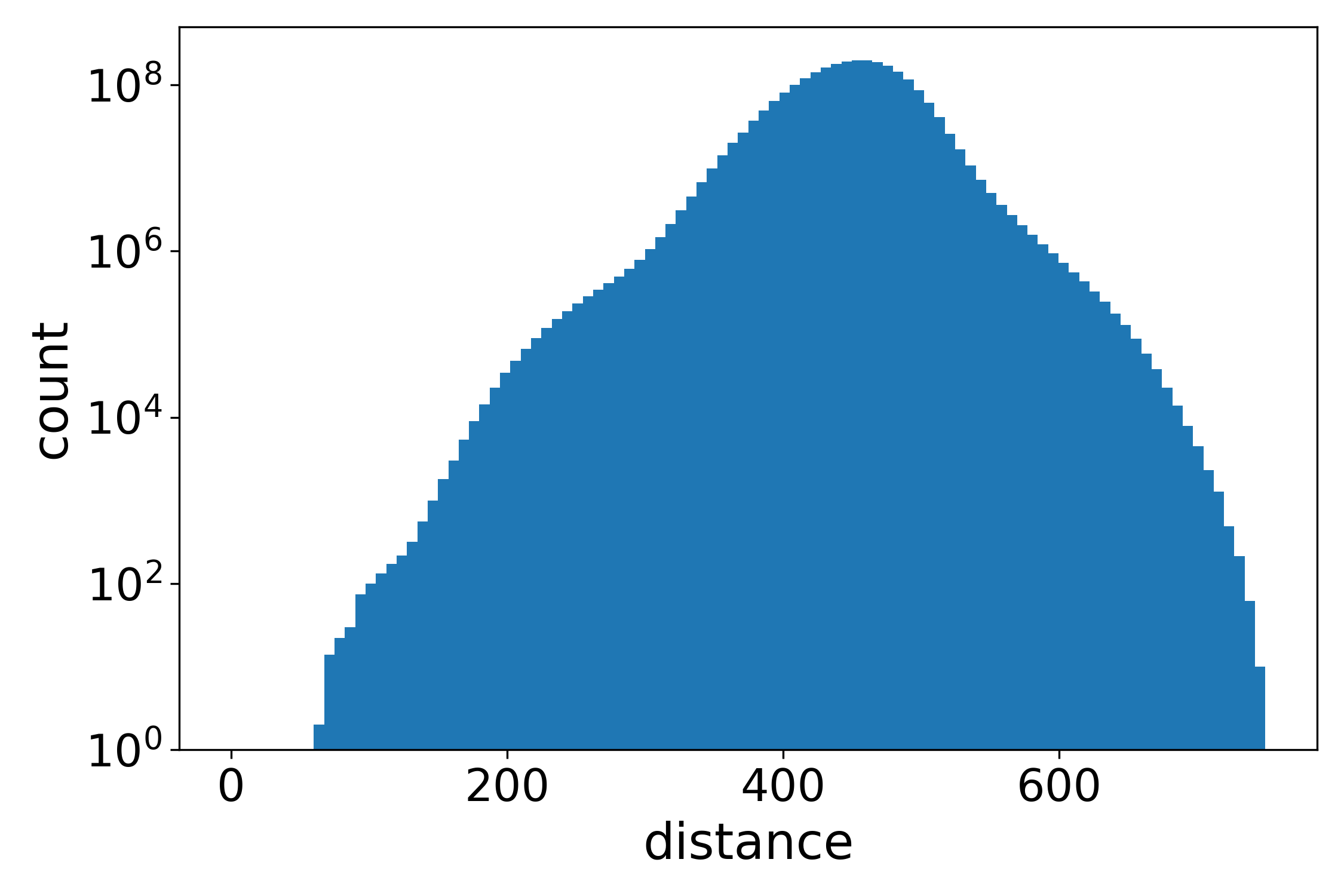}
\end{minipage}
\caption{Distribution of context-to-context Euclidean distances for Mistral on Pile10k (left), Bookcorpus (middle) and WikiText-103 (right). We see concentration of distances in the sense that there is a gap from zero to the lowest distance values. Here, we do not include the distance of a context to itself, since it will always be zero for this distance measure.}
\label{fig:cc_euc_dist_distribution_mistral}
\end{figure*}

\begin{figure*}[htb]
\begin{minipage}[b]{0.33\linewidth}
\centering
\includegraphics[width=\textwidth]{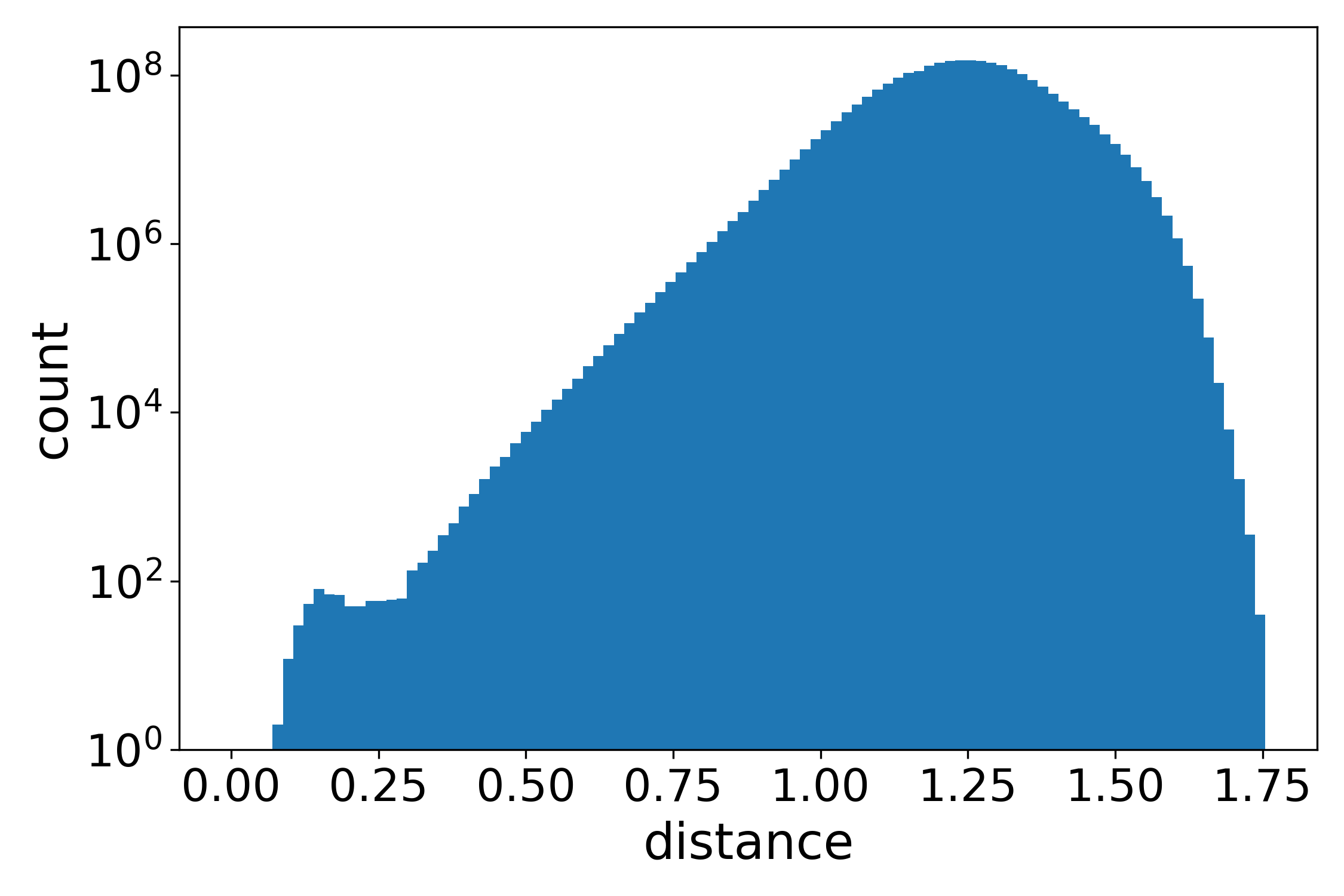}
\end{minipage}
\hspace{-0.1cm}
\begin{minipage}[b]{0.33\linewidth}
\centering
\includegraphics[width=\textwidth]{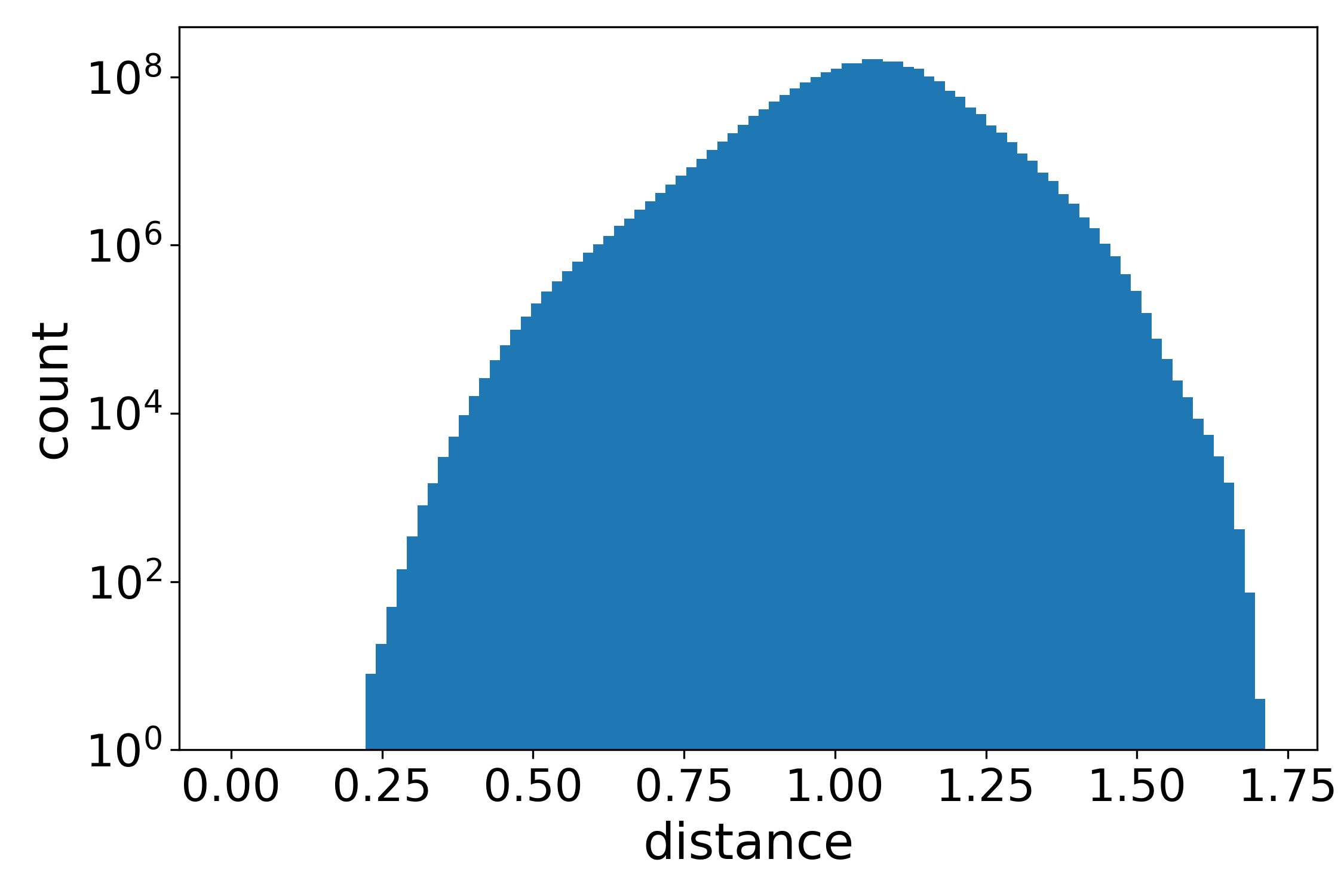}
\end{minipage}
\hspace{-0.1cm}
\begin{minipage}[b]{0.33\linewidth}
\centering
\includegraphics[width=\textwidth]{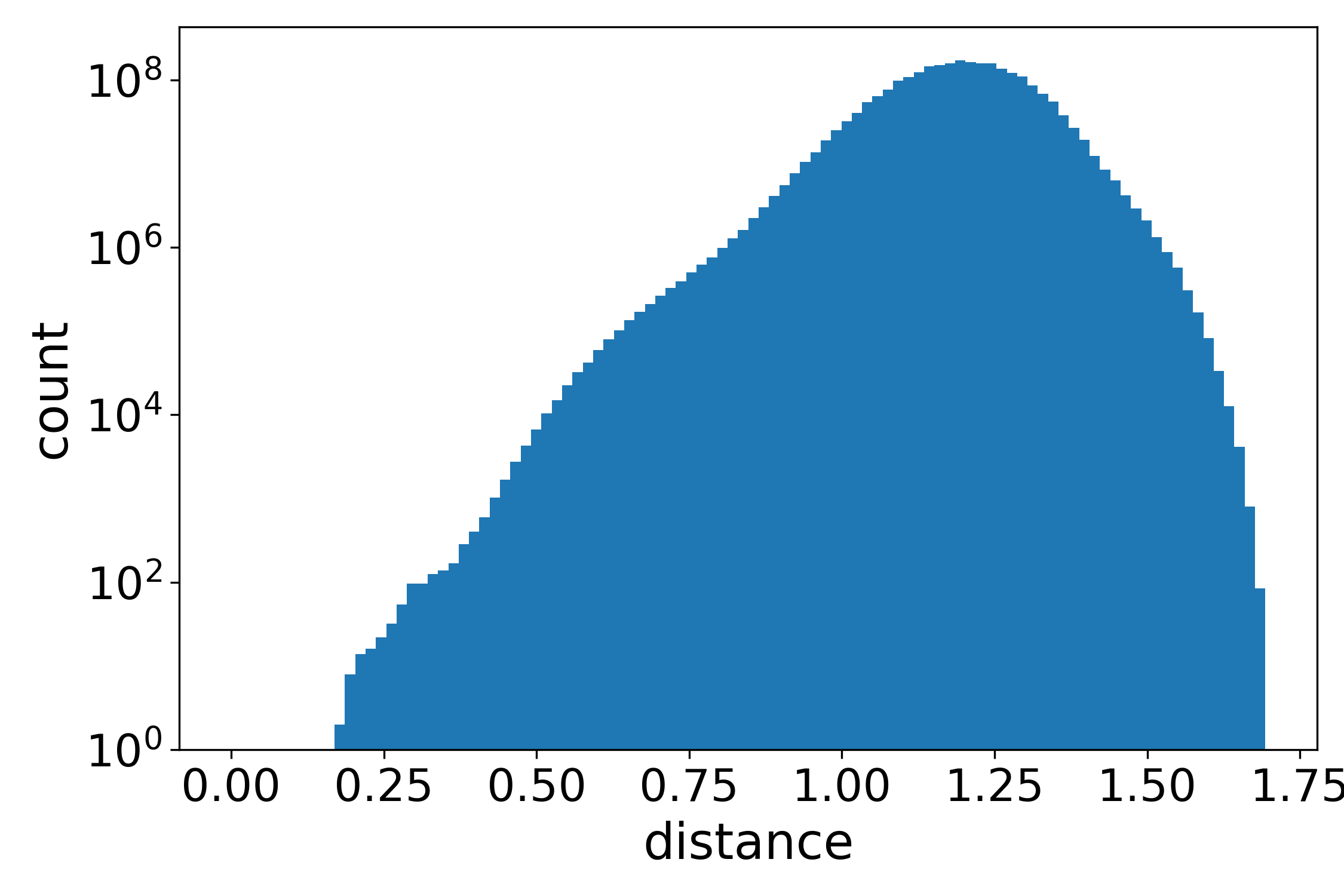}
\end{minipage}
\caption{Distribution of context-to-context normalized Euclidean distances for Mistral on Pile10k (left), Bookcorpus (middle) and WikiText-103 (right). We see concentration of distances in the sense that there is a gap from zero to the lowest distance values. Here, we do not include the distance of a context to itself, since it will always be zero for this distance measure.}
\label{fig:cc_norm_euc_dist_distribution_mistral}
\end{figure*}

\begin{figure*}[htb]
\begin{minipage}[b]{0.33\linewidth}
\centering
\includegraphics[width=\textwidth]{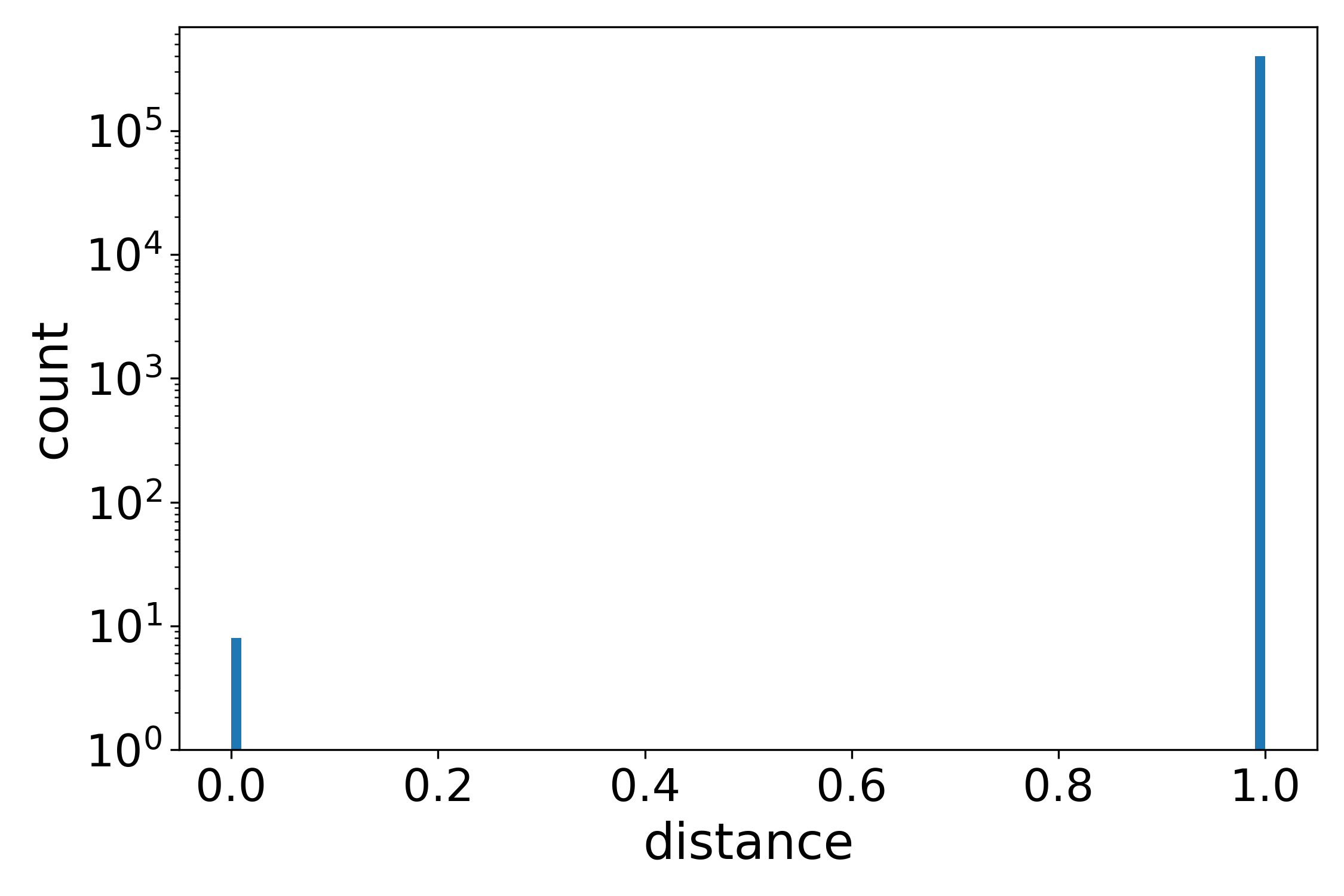}
\end{minipage}
\hspace{-0.1cm}
\begin{minipage}[b]{0.33\linewidth}
\centering
\includegraphics[width=\textwidth]{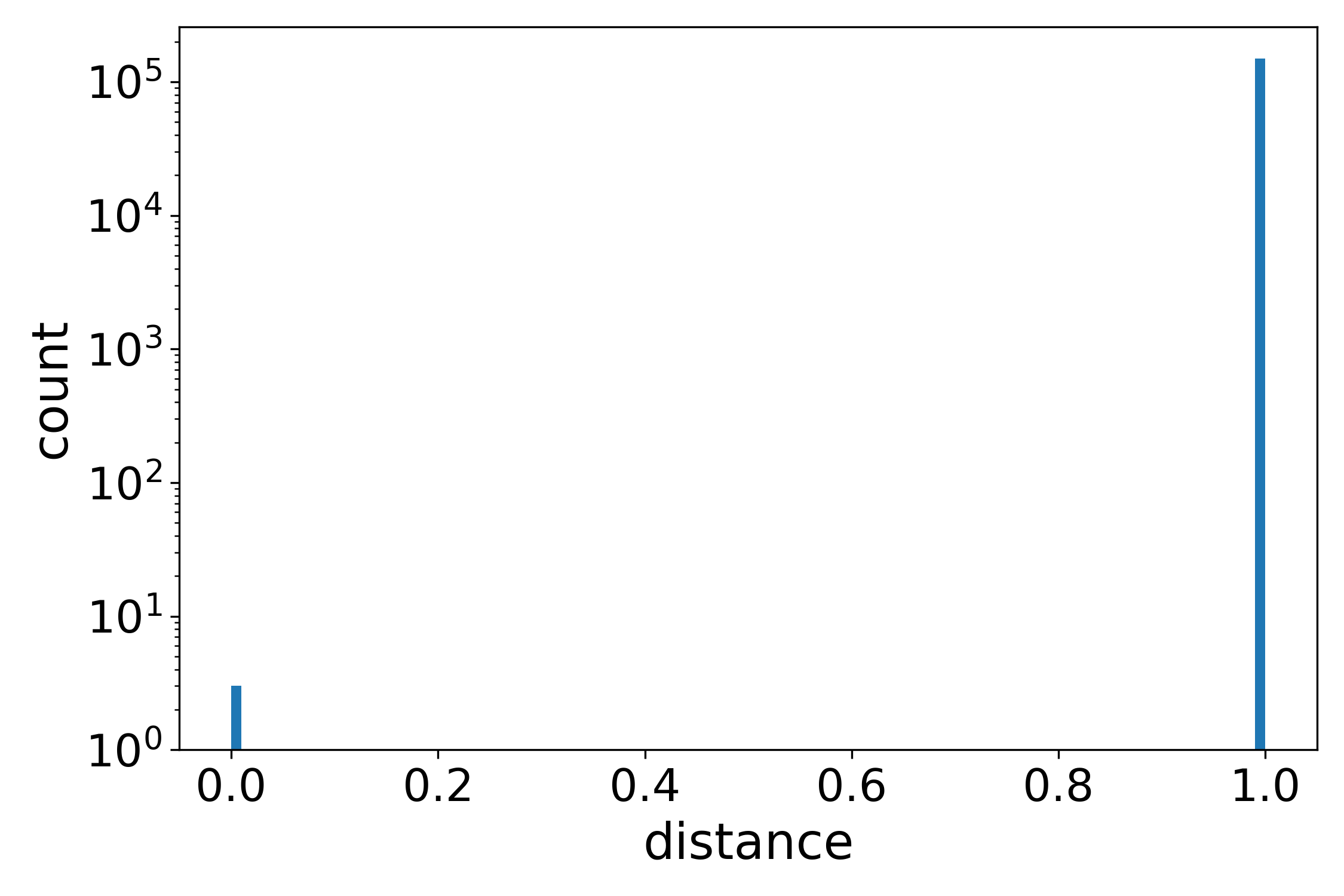}
\end{minipage}
\hspace{-0.1cm}
\begin{minipage}[b]{0.33\linewidth}
\centering
\includegraphics[width=\textwidth]{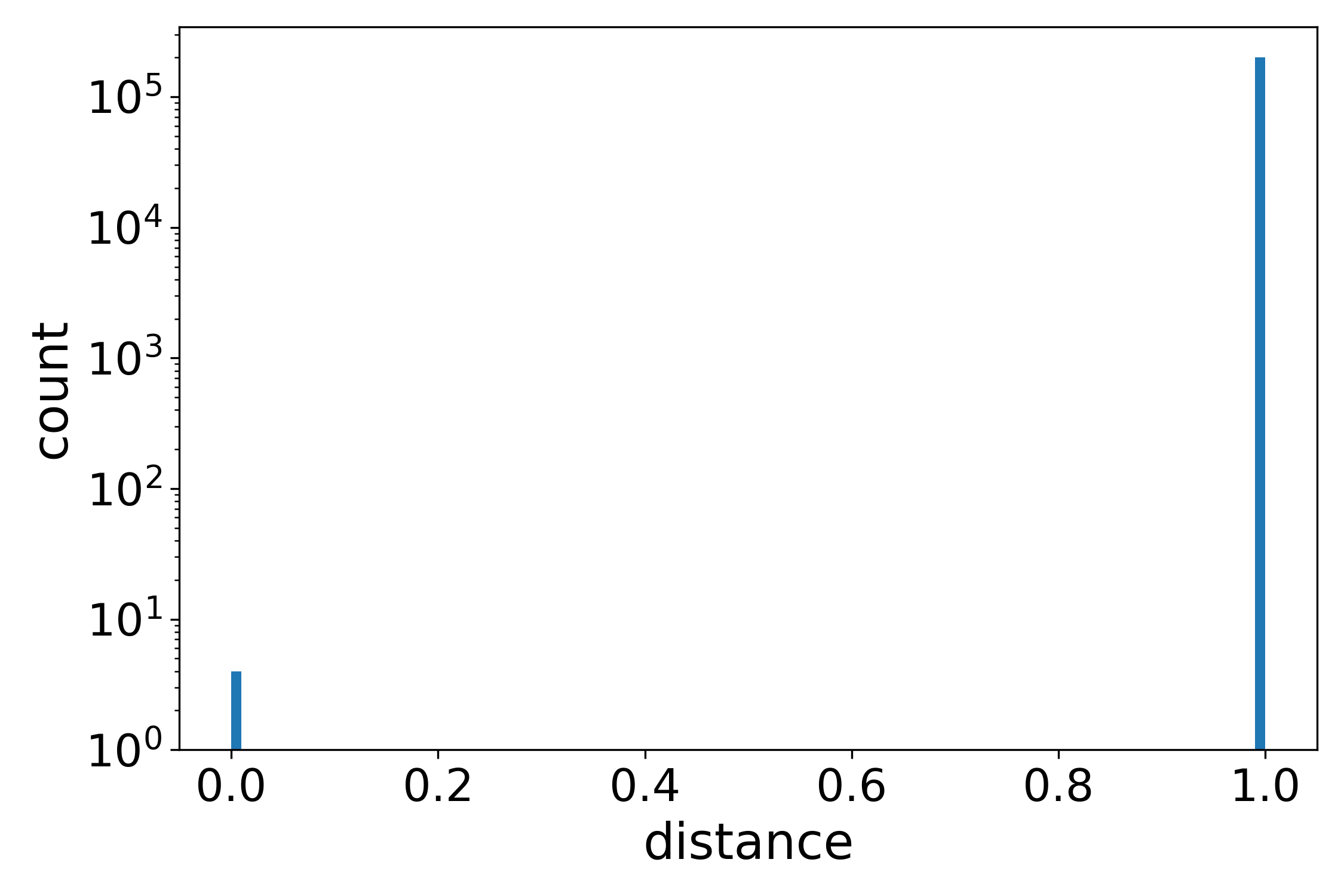}
\end{minipage}
\caption{Distribution of context-to-context softmaxed dot product distances for Mistral on Pile10k (left), Bookcorpus (middle) and WikiText-103 (right). Here we have included the distance of a context to itself. Note that, when using a dot product, there is no guarantee that a context will get the largest score with itself. For Mistral, most contexts do not have a significantly different dot product with themselves compared to that with other contexts.}
\label{fig:cc_softmax_dot_distribution_mistral}
\end{figure*}

\section{Distribution of vocabulary-item-to-vocabulary-item distances}
\label{app:distribution_of_tt_distances}
We present here plots showing the distribution of distances when comparing vocabulary item with vocabulary item for Llama (Fig.~\ref{fig:tt_dist_distribution_llama}), Pythia (Fig.~\ref{fig:tt_dist_distribution_pythia}), Opt (Fig.~\ref{fig:tt_dist_distribution_opt}), Olmo (Fig.~\ref{fig:tt_dist_distribution_olmo}) and Mistral (Fig.~\ref{fig:tt_dist_distribution_mistral}). In these plots we see a concentration of distances for all models when using softmaxed dot product, but for Euclidean and normalized Euclidean distance the behaviour is more varied.  

\begin{figure*}[htb]
\begin{minipage}[b]{0.33\linewidth}
\centering
\includegraphics[width=\textwidth]{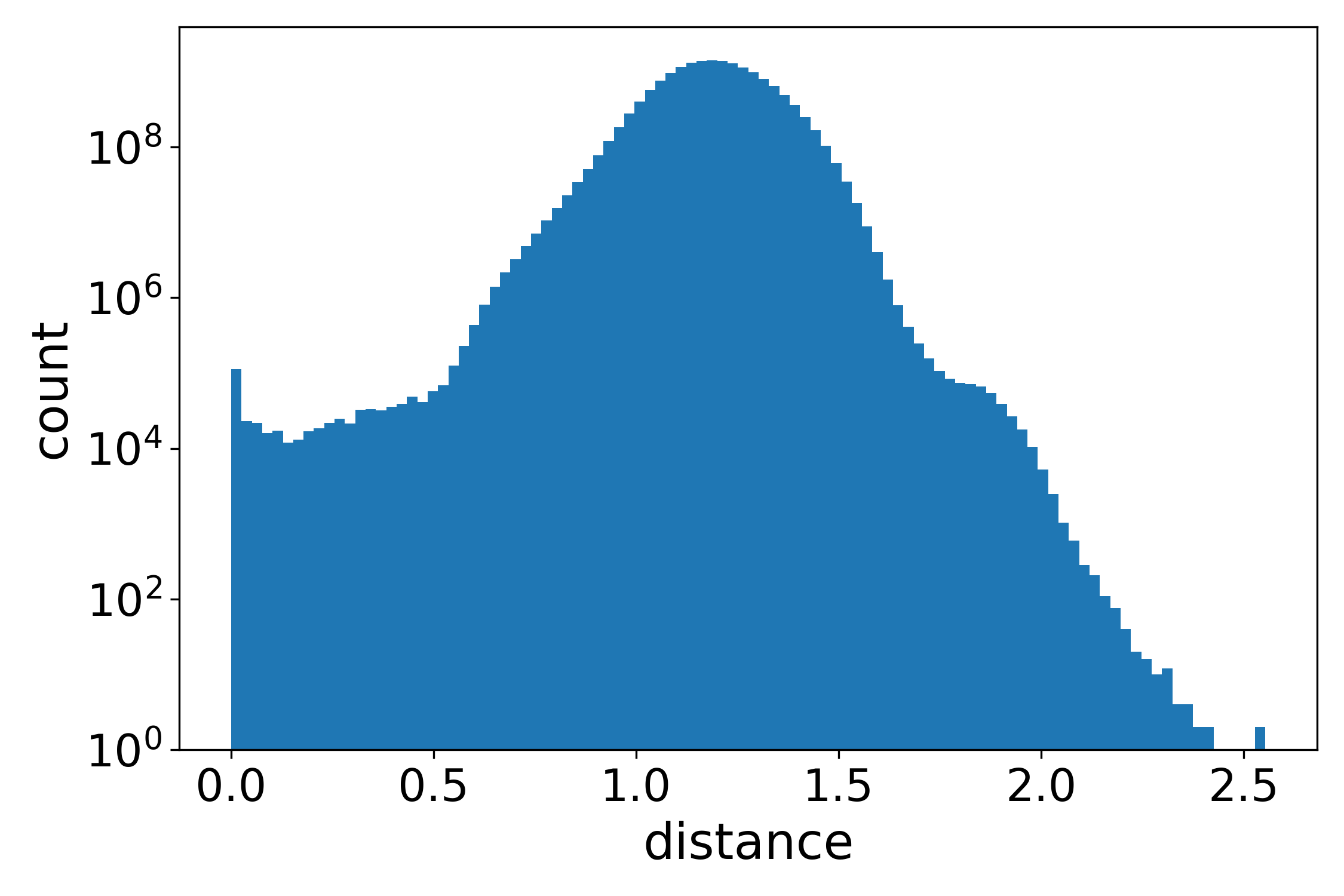}
\end{minipage}
\hspace{-0.1cm}
\begin{minipage}[b]{0.33\linewidth}
\centering
\includegraphics[width=\textwidth]{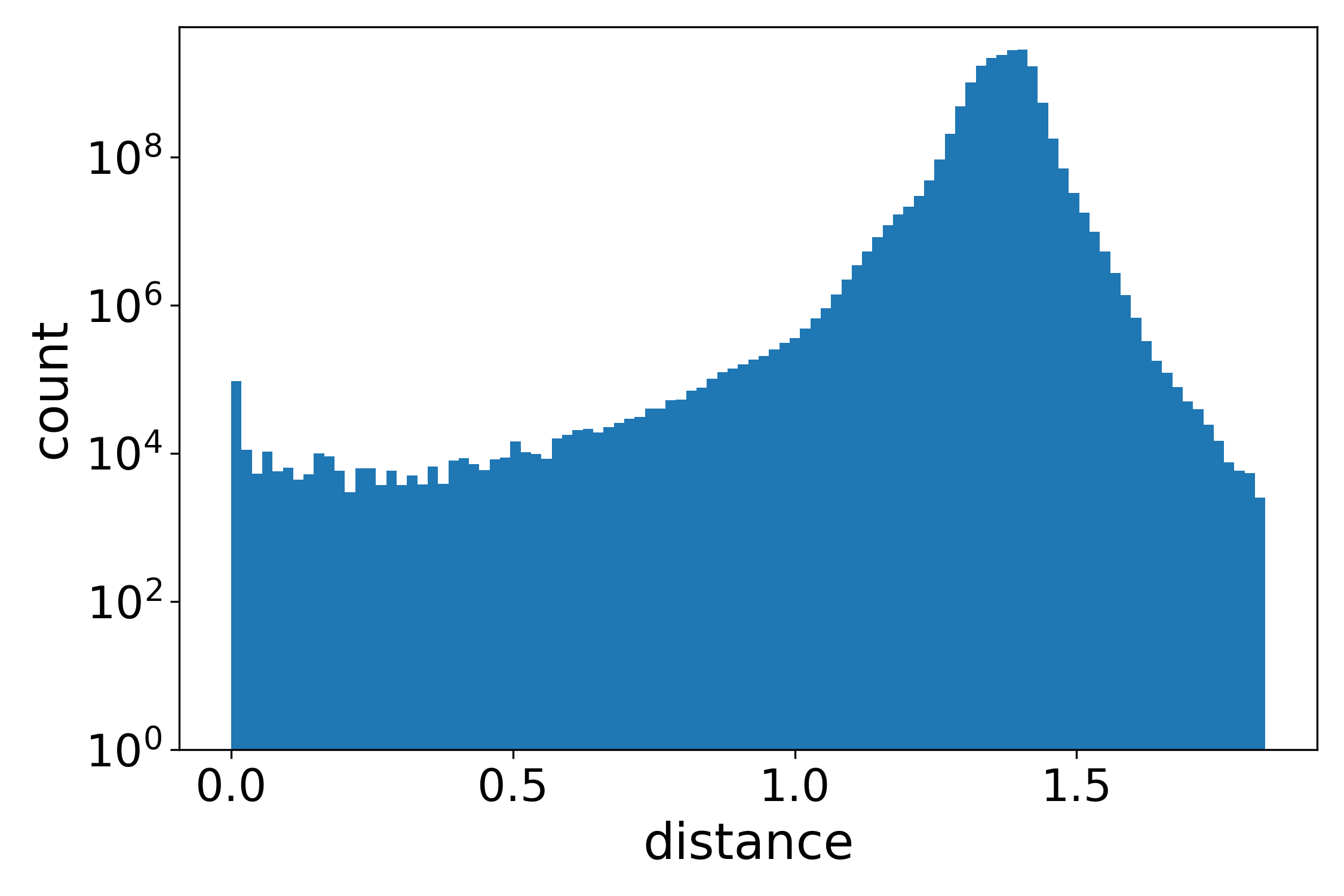}
\end{minipage}
\hspace{-0.1cm}
\begin{minipage}[b]{0.33\linewidth}
\centering
\includegraphics[width=\textwidth]{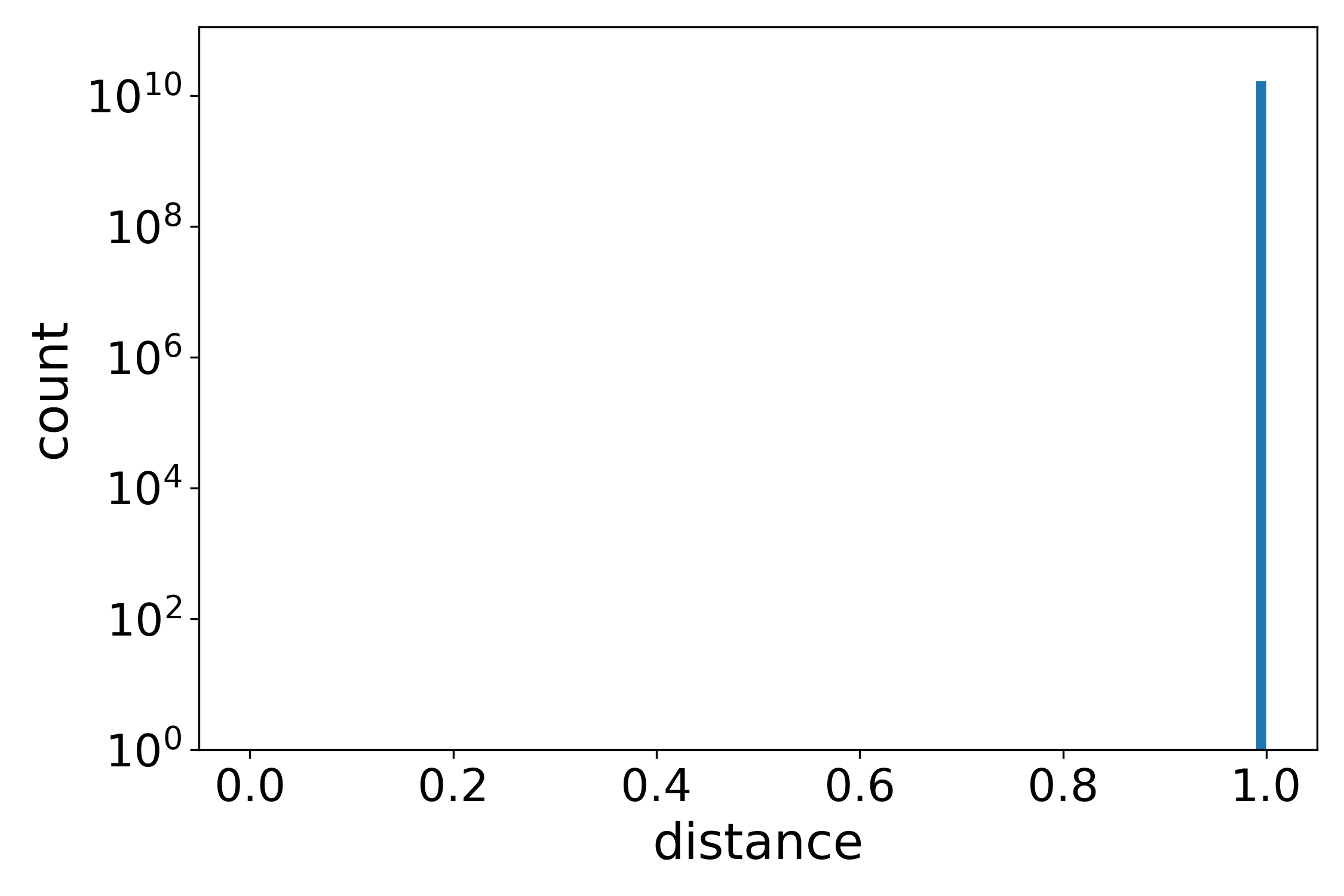}
\end{minipage}
\caption{Distribution of vocabulary to vocabulary distances for Llama using Euclidean (left), normalized Euclidean (middle) and softmaxed dot product (right) distances. For Euclidean and normalized Euclidean, we do not include the distance of an item to itself, since it will always be zero. The spread of distances goes all the way to zero for Euclidean and normalized Euclidean. However, we get a concentration of distances for the softmaxed dot product.} 
\label{fig:tt_dist_distribution_llama}
\end{figure*}

\begin{figure*}[htb]
\begin{minipage}[b]{0.33\linewidth}
\centering
\includegraphics[width=\textwidth]{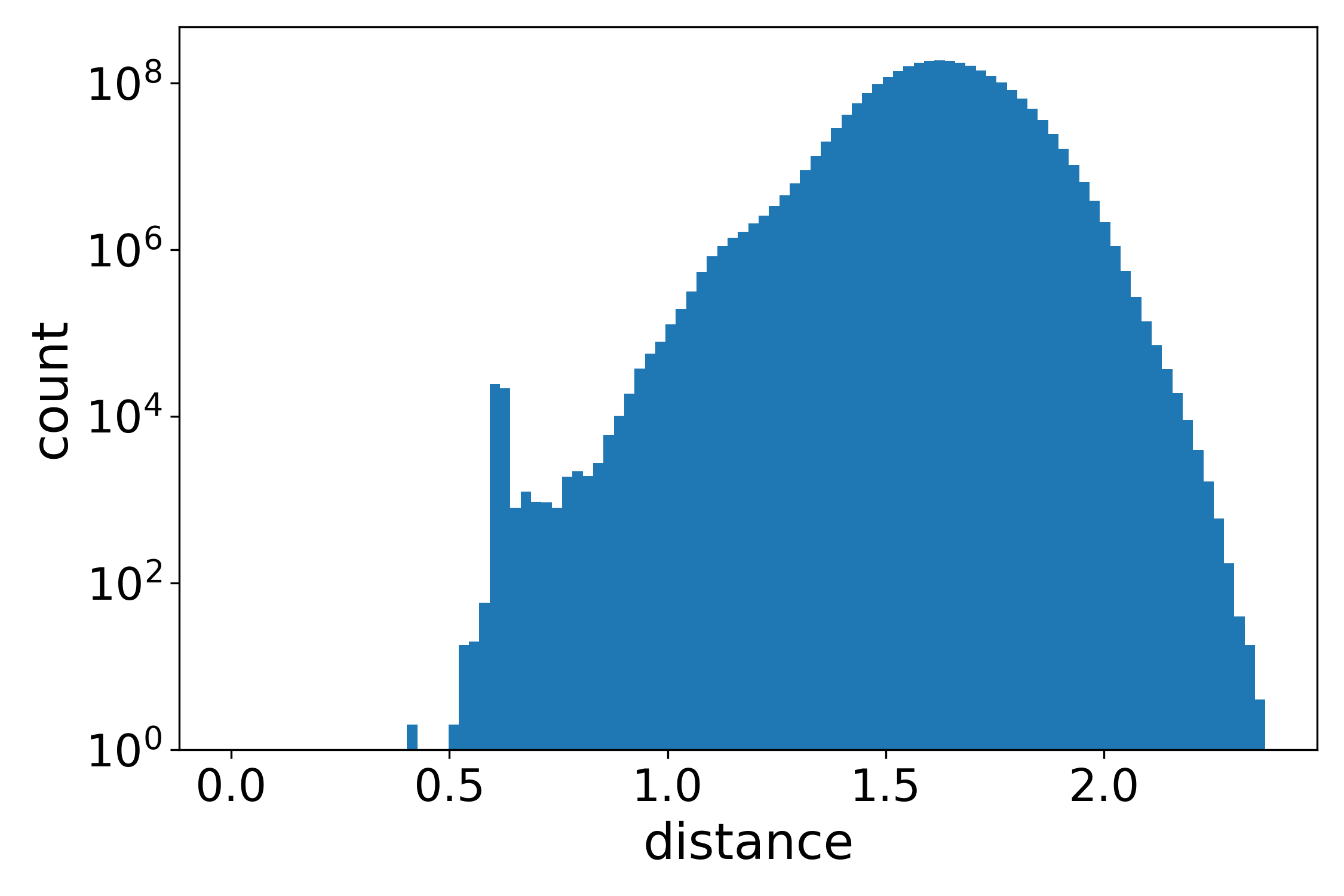}
\end{minipage}
\hspace{-0.1cm}
\begin{minipage}[b]{0.33\linewidth}
\centering
\includegraphics[width=\textwidth]{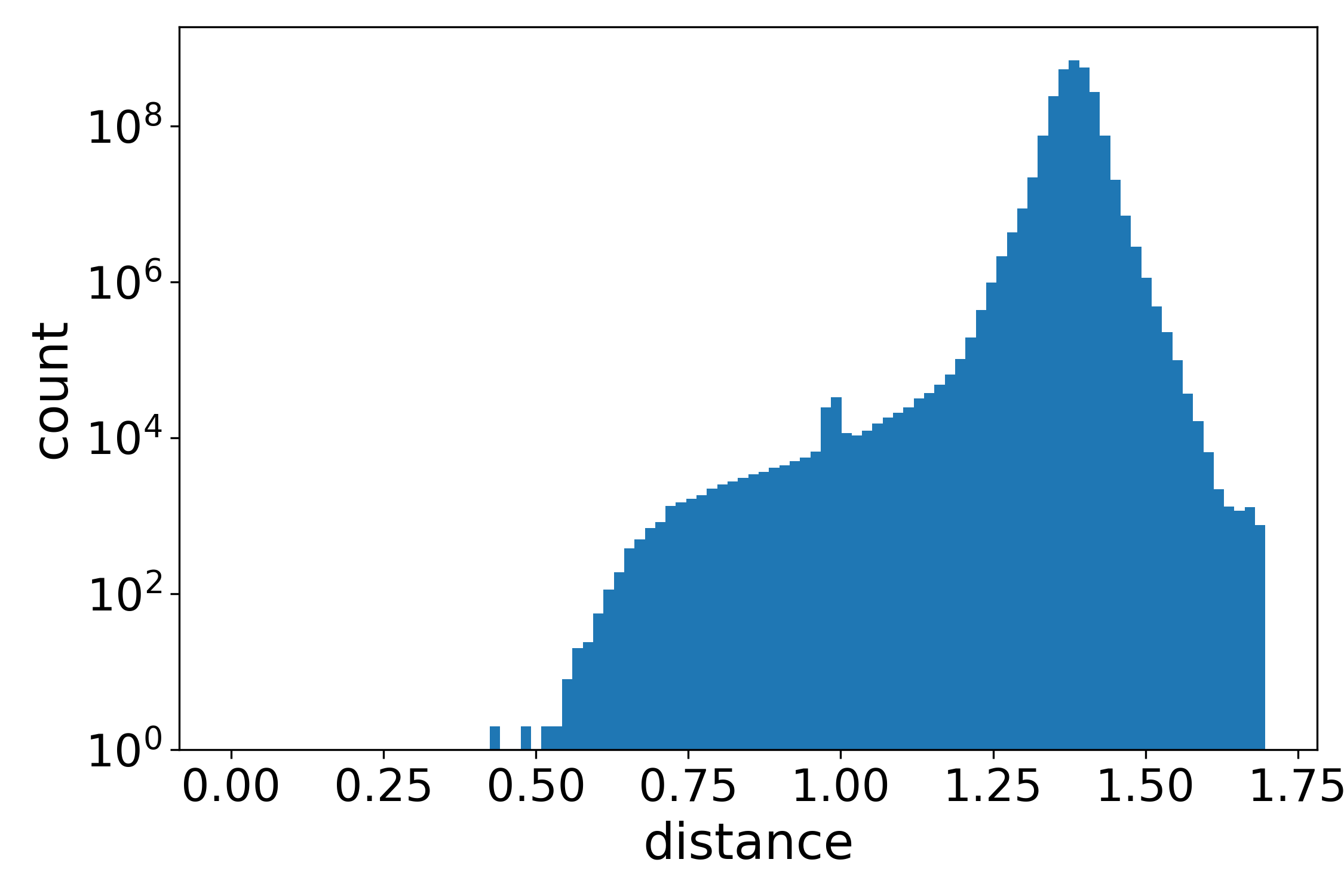}
\end{minipage}
\hspace{-0.1cm}
\begin{minipage}[b]{0.33\linewidth}
\centering
\includegraphics[width=\textwidth]{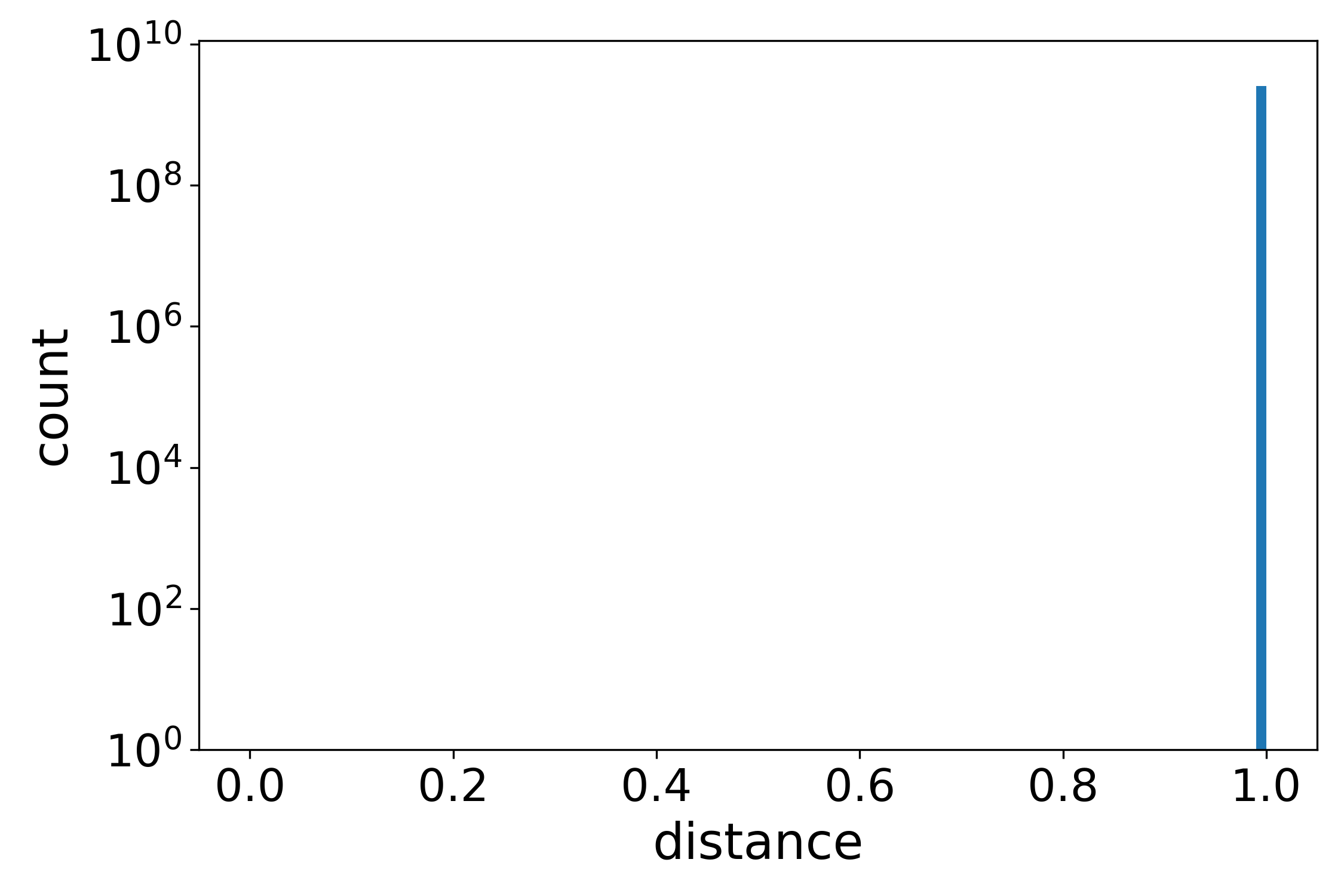}
\end{minipage}
\caption{Distribution of vocabulary to vocabulary distances for Pythia using Euclidean (left), normalized Euclidean (middle) and softmaxed dot product (right) distances. For Euclidean and normalized Euclidean, we do not include the distance of an item to itself, since it will always be zero. We get a concentration of distances for all distance measures.} 
\label{fig:tt_dist_distribution_pythia}
\end{figure*}

\begin{figure*}[htb]
\begin{minipage}[b]{0.33\linewidth}
\centering
\includegraphics[width=\textwidth]{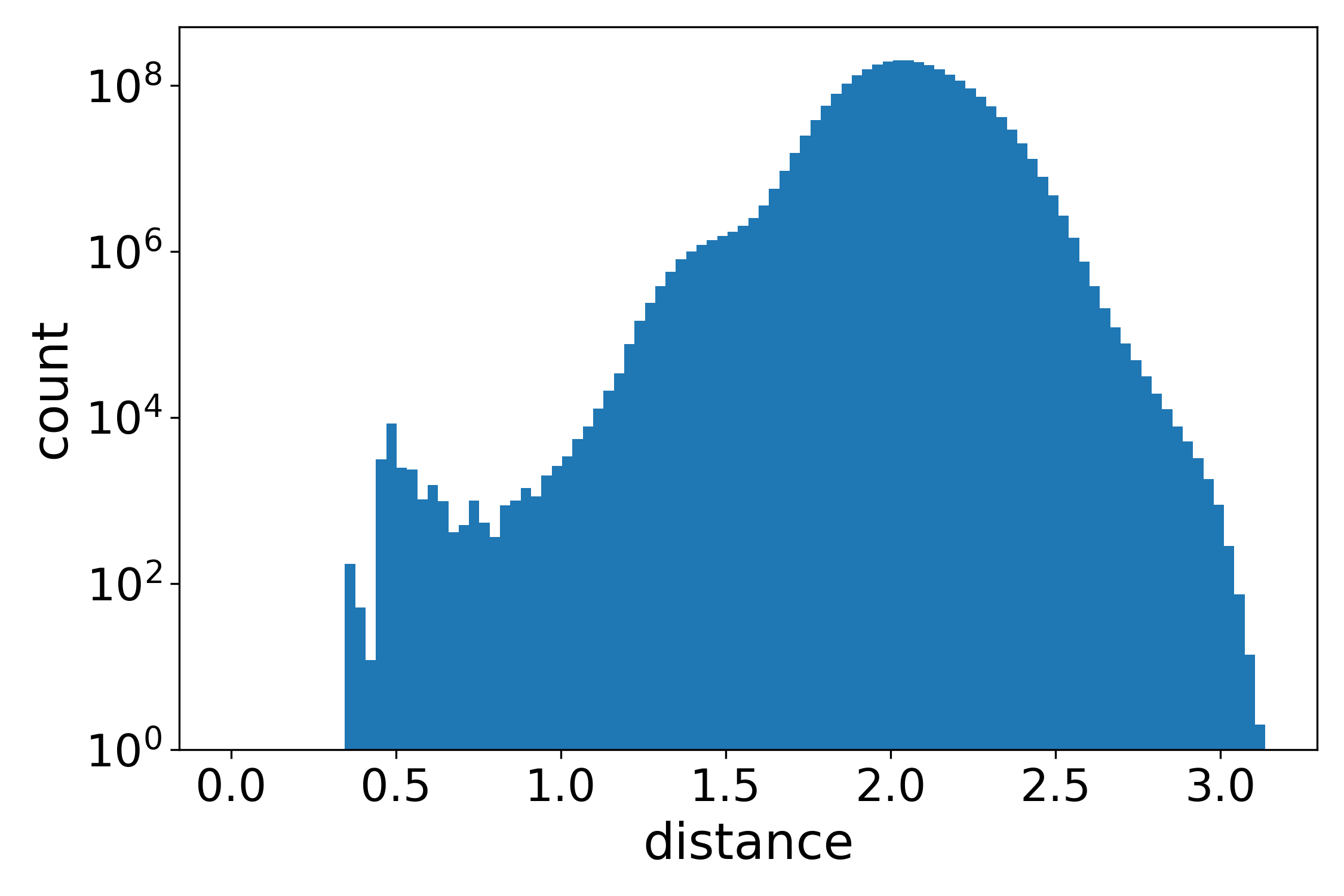}
\end{minipage}
\hspace{-0.1cm}
\begin{minipage}[b]{0.33\linewidth}
\centering
\includegraphics[width=\textwidth]{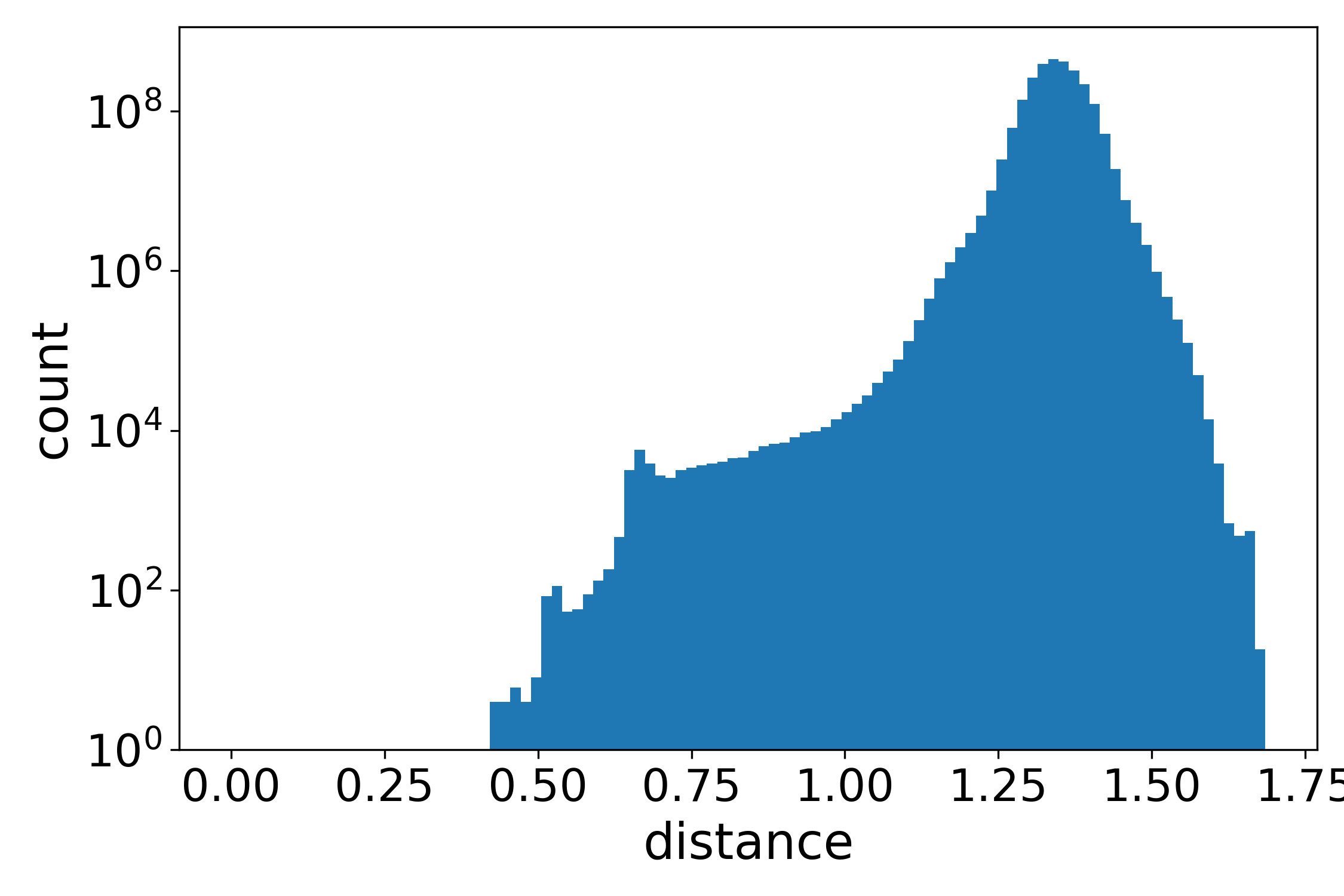}
\end{minipage}
\hspace{-0.1cm}
\begin{minipage}[b]{0.33\linewidth}
\centering
\includegraphics[width=\textwidth]{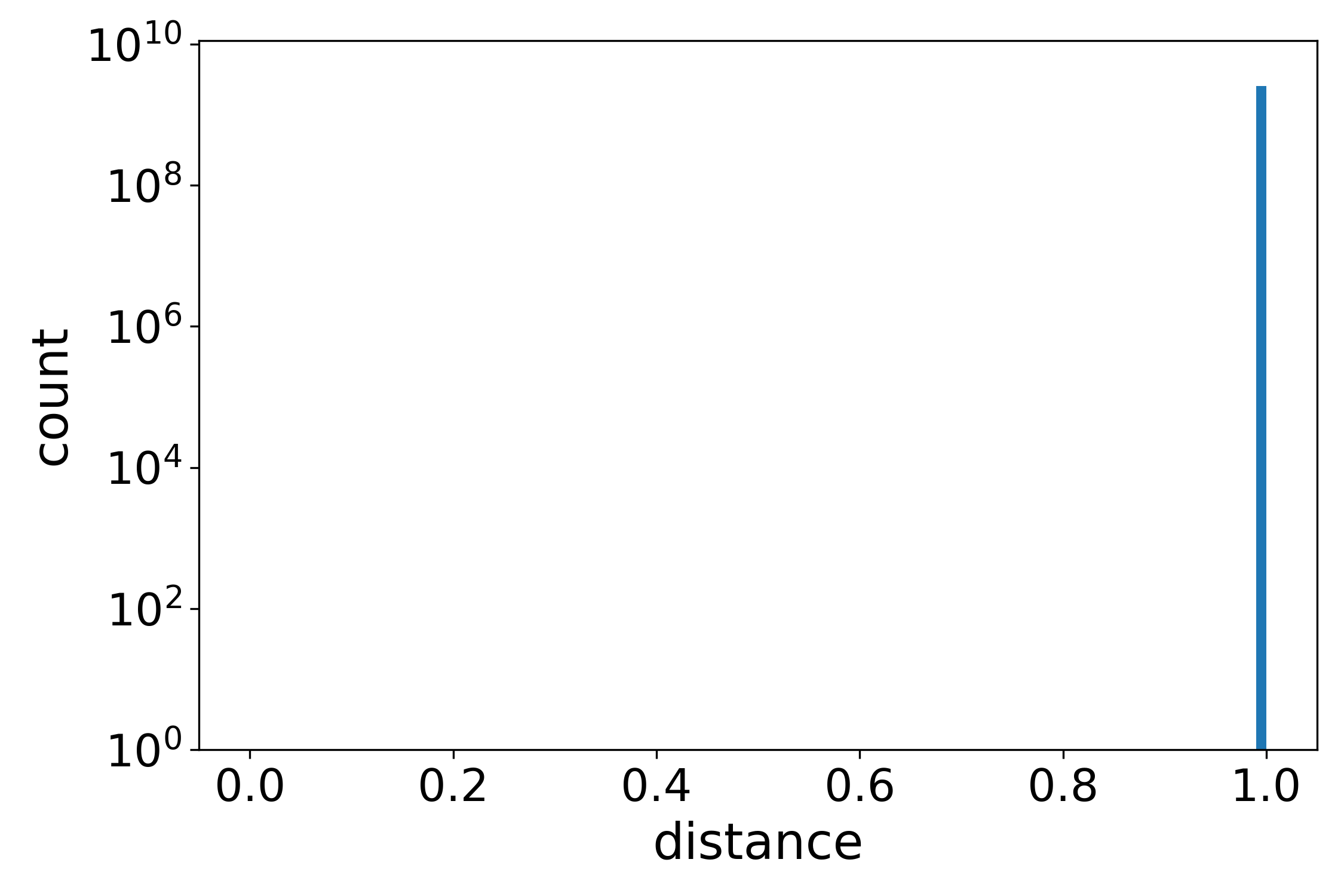}
\end{minipage}
\caption{Distribution of vocabulary to vocabulary distances for Opt using Euclidean (left), normalized Euclidean (middle) and softmaxed dot product (right) distances. For Euclidean and normalized Euclidean, we do not include the distance of an item to itself, since it will always be zero. We get a concentration of distances for all distance measures.} 
\label{fig:tt_dist_distribution_opt}
\end{figure*}

\begin{figure*}[htb]
\begin{minipage}[b]{0.33\linewidth}
\centering
\includegraphics[width=\textwidth]{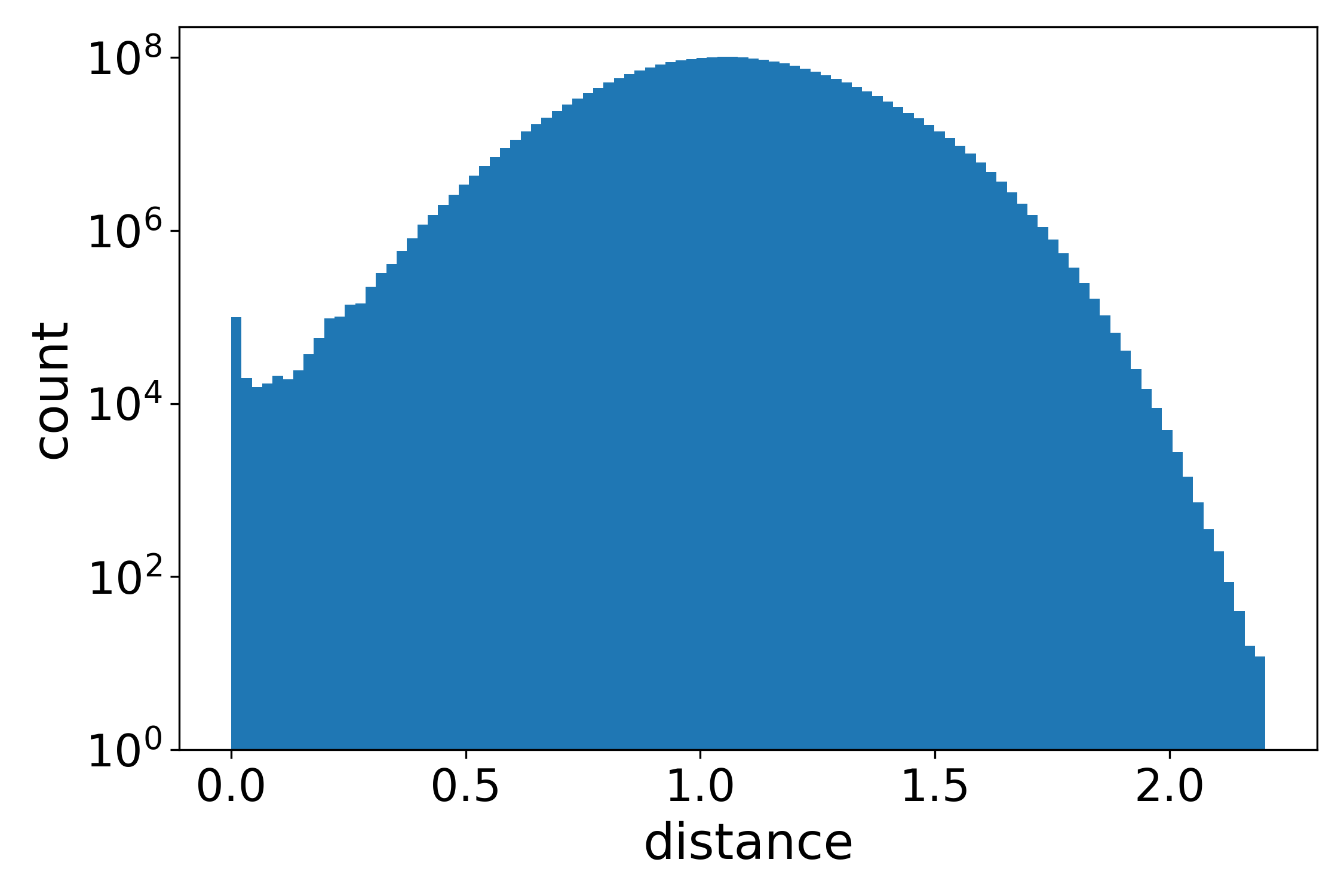}
\end{minipage}
\hspace{-0.1cm}
\begin{minipage}[b]{0.33\linewidth}
\centering
\includegraphics[width=\textwidth]{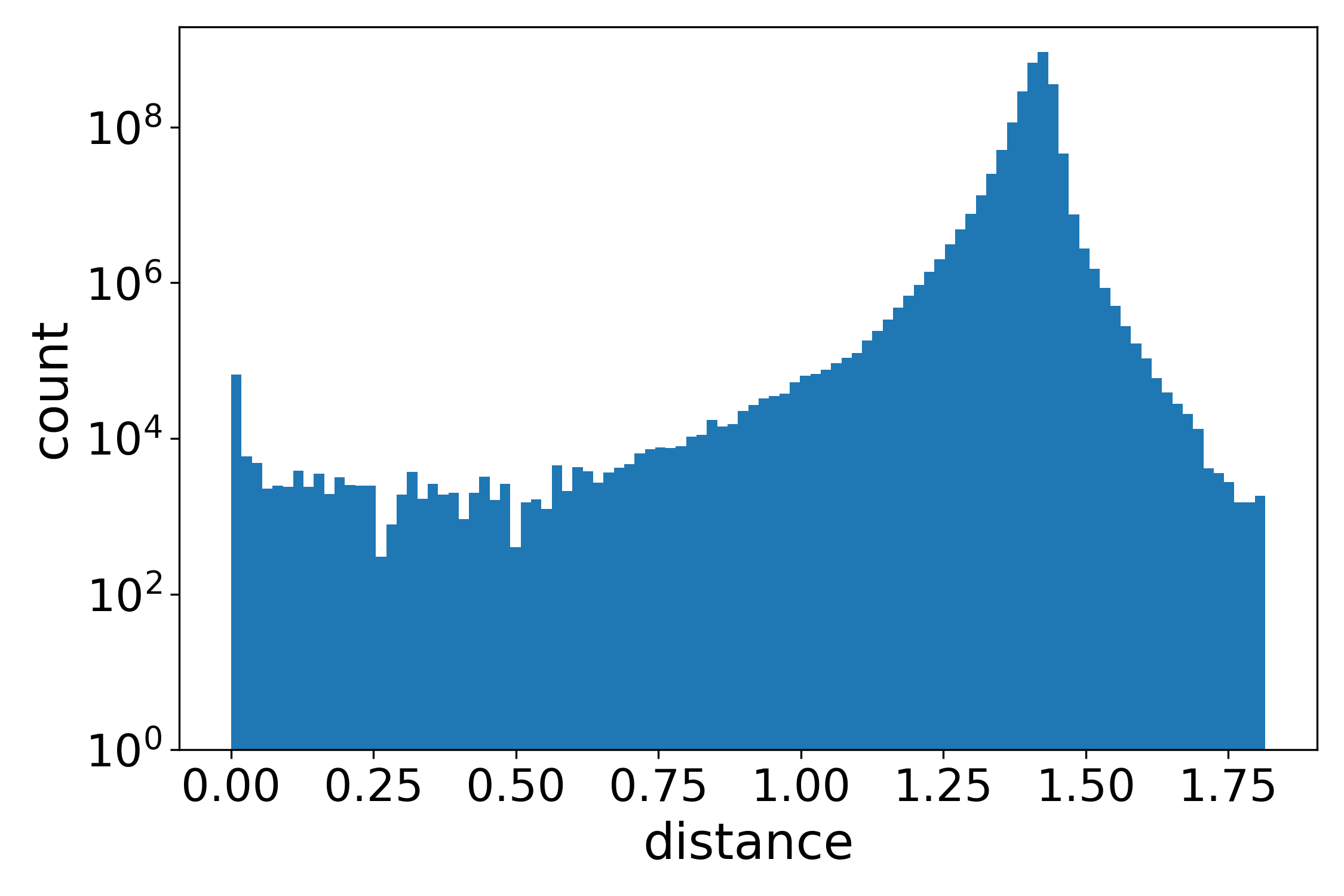}
\end{minipage}
\hspace{-0.1cm}
\begin{minipage}[b]{0.33\linewidth}
\centering
\includegraphics[width=\textwidth]{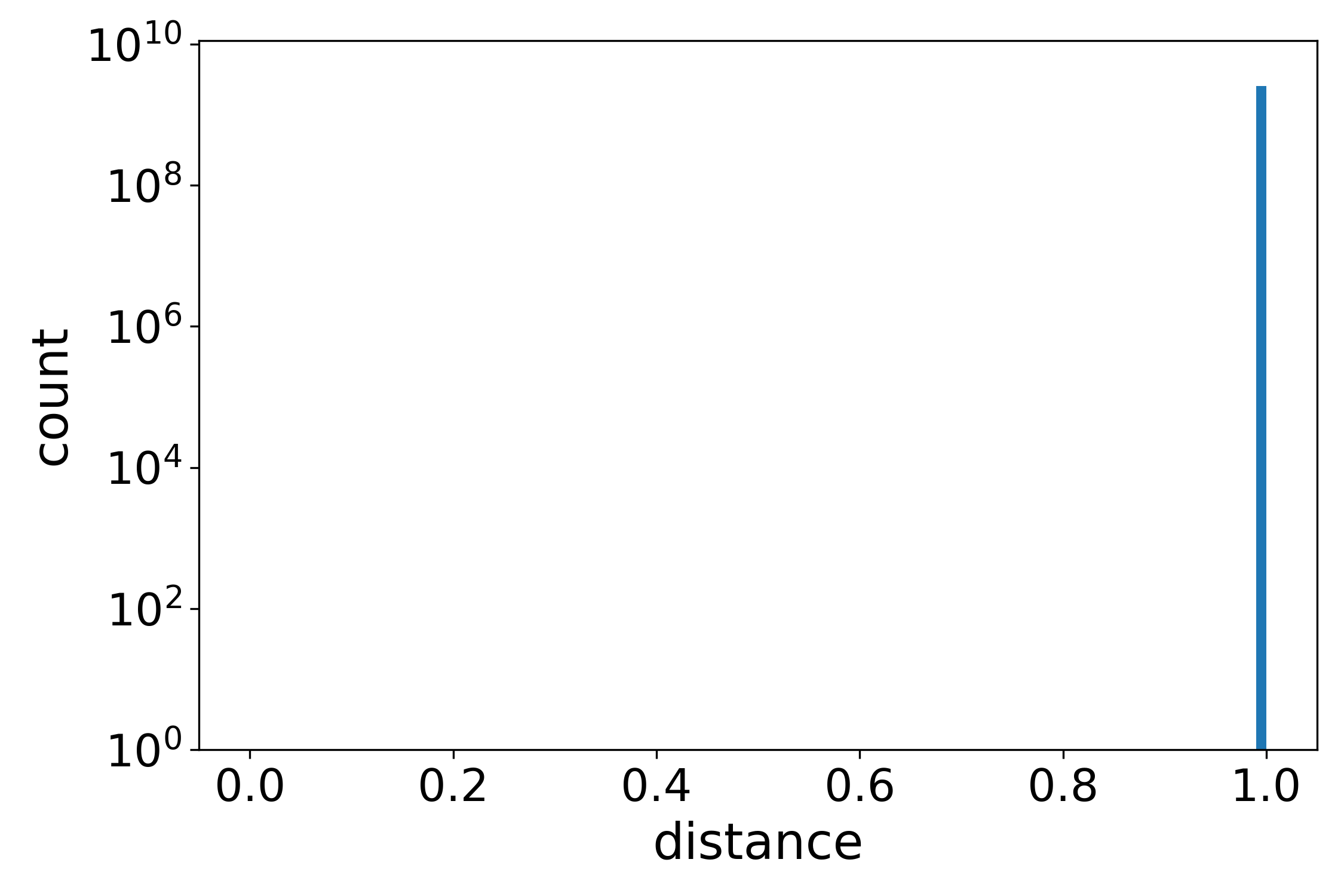}
\end{minipage}
\caption{Distribution of vocabulary to vocabulary distances for Olmo using Euclidean (left), normalized Euclidean (middle) and softmaxed dot product (right) distances. For Euclidean and normalized Euclidean, we do not include the distance of an item to itself, since it will always be zero. The spread of distances goes all the way to zero for Euclidean and normalized Euclidean. However, we get a concentration of distances for the softmaxed dot product.} 
\label{fig:tt_dist_distribution_olmo}
\end{figure*}

\begin{figure*}[htb]
\begin{minipage}[b]{0.33\linewidth}
\centering
\includegraphics[width=\textwidth]{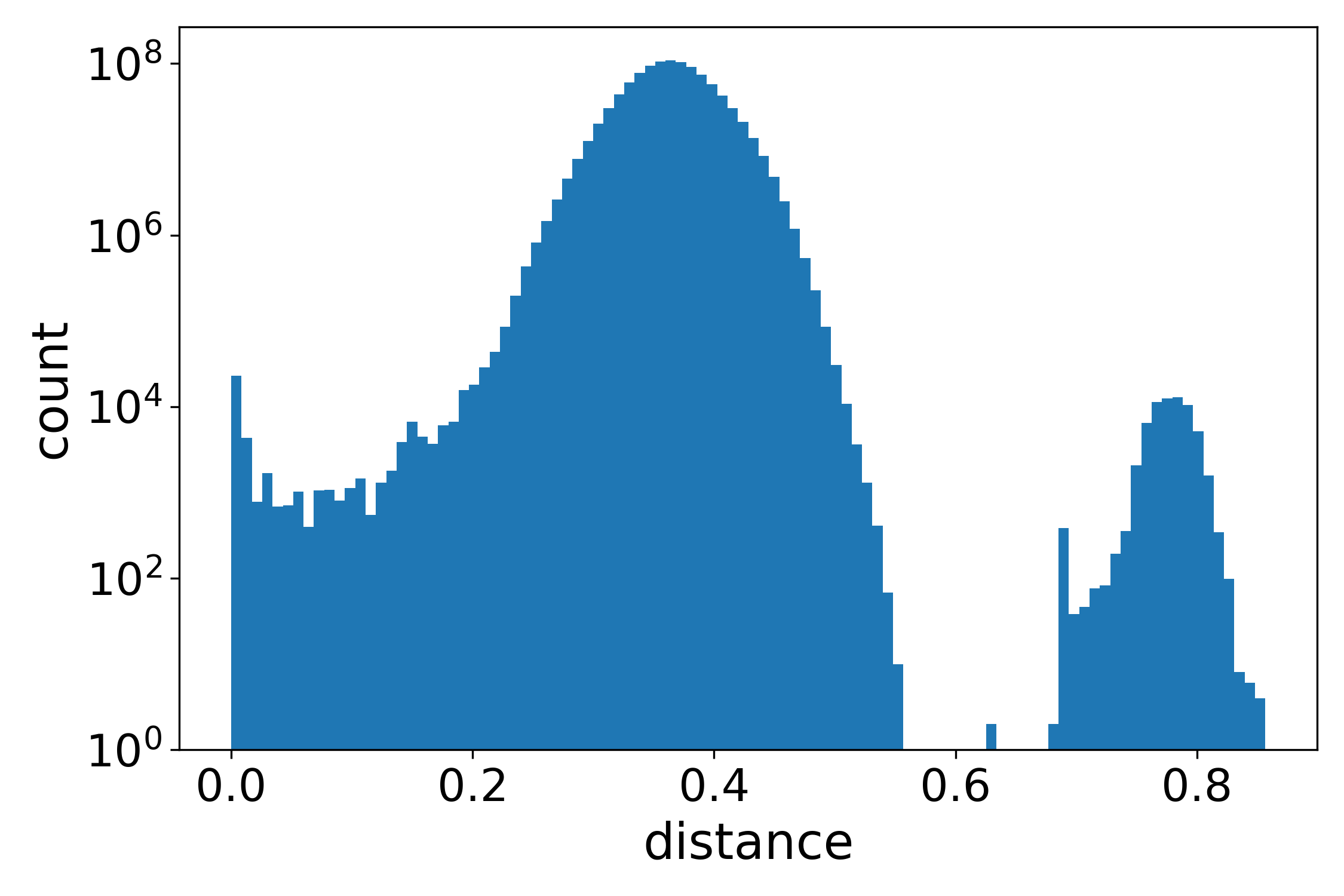}
\end{minipage}
\hspace{-0.1cm}
\begin{minipage}[b]{0.33\linewidth}
\centering
\includegraphics[width=\textwidth]{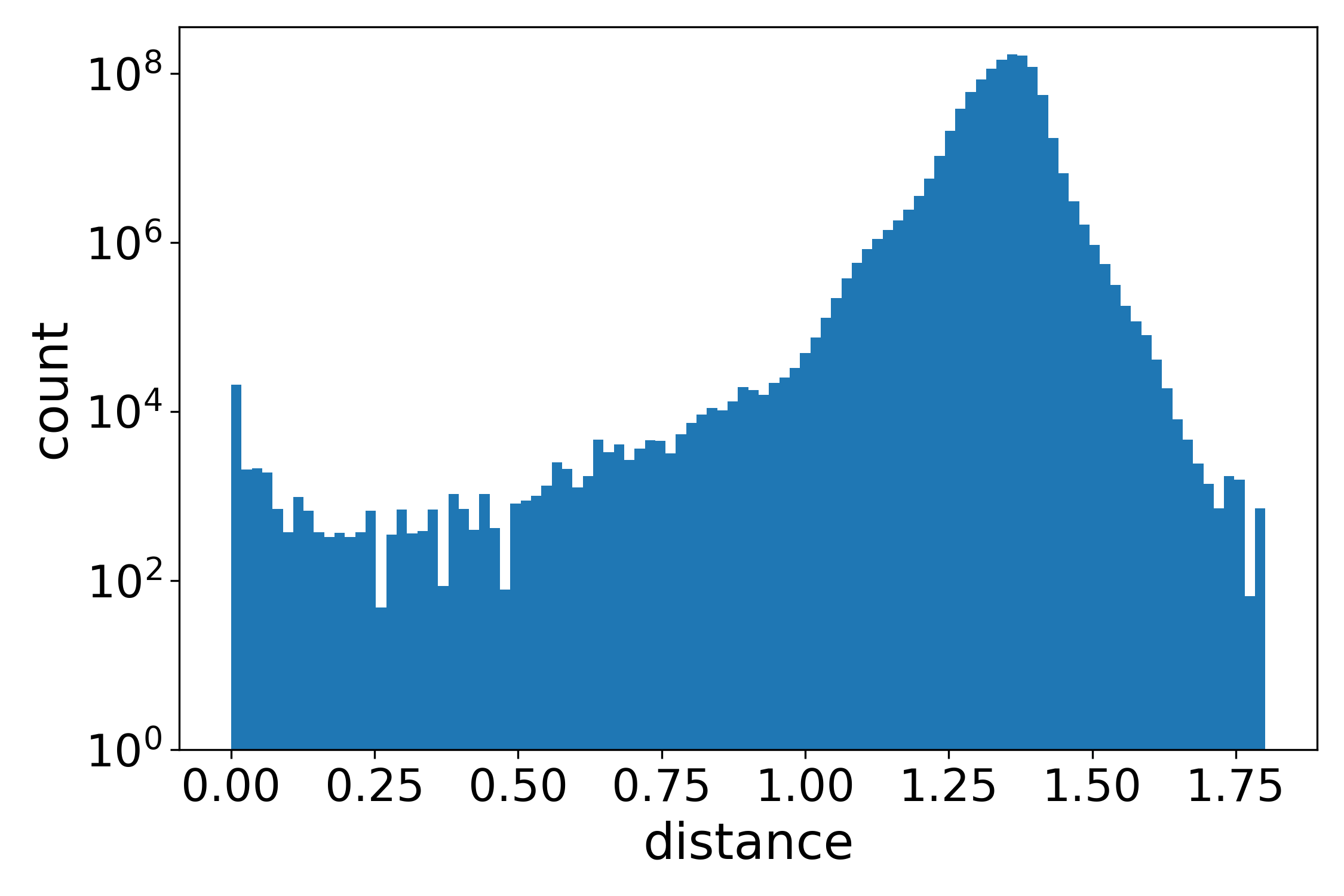}
\end{minipage}
\hspace{-0.1cm}
\begin{minipage}[b]{0.33\linewidth}
\centering
\includegraphics[width=\textwidth]{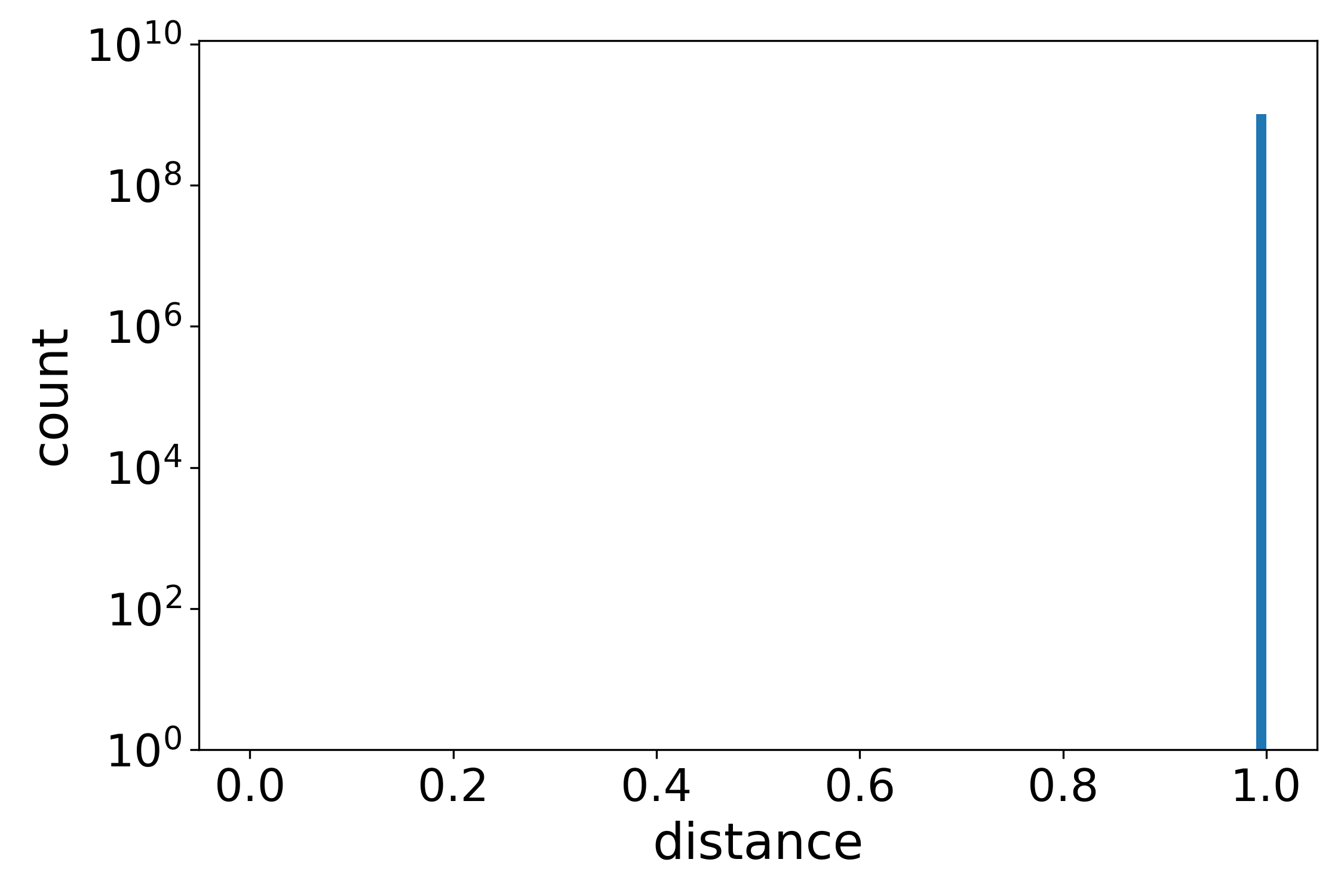}
\end{minipage}
\caption{Distribution of vocabulary to vocabulary distances for Mistral using Euclidean (left), normalized Euclidean (middle) and softmaxed dot product (right) distances. For Euclidean and normalized Euclidean, we do not include the distance of an item to itself, since it will always be zero. The spread of distances goes all the way to zero for Euclidean and normalized Euclidean. However, we get a concentration of distances for the softmaxed dot product. Mistral is the only model to display a second ``hump'' when using the Euclidean distance.} 
\label{fig:tt_dist_distribution_mistral}
\end{figure*}

\section{Hubs \textit{k}-occurrence correlation with frequency of tokens}
\label{app:k-occurrence-frequency-correlation}
In table \ref{table:pred_hub_corr} we see that the $k$-occurrence of prediction hubs is strongly correlated with the frequency of vocabulary items in the corpus the contexts come from. For Pythia and Olmo, we also have access to the original training corpora, namely the (full) Pile \citep{Gao:etal:2020} and Dolma \citep{Soldaini:etal:2024}, and we use them to compute their training token frequency distributions. These frequencies are used in the rows of the table where \textbf{freq from} is ``train dataset''. The correlations are also higher for frequencies based on the corpora the contexts come from than for frequencies from the training data. In Table \ref{table:tt_hub_corr}, we see that, when comparing vocabulary items with other vocabulary items, we do not get a good correlation between $k$-occurrence of the hubs and frequency of vocabulary items.

With respect to checkpoints from Pythia, we see in Table \ref{table:pred_pythia_steps_hub_corr} that correlation with frequencies from the relevant dataset increases as the model trains for longer. We also see that the correlation for Pile10k saturates quite fast, which is probably due to Pythia being trained on the Pile. In Table \ref{table:tt_pythia_steps_hub_corr} we see that there is no strong correlation for the hubs emerging from comparing vocabulary items with vocabulary items.   

\begin{table*}
    \small
    \centering
    \begin{tabular}{lllr}
        \hline
        \textbf{model} & \textbf{context} & \textbf{freq from} & \textbf{Spearman corr} \\
        \hline
        Pythia & Pile10k & Pile10k & 0.71 \\
        Pythia & Pile10k & WikiText-103 & 0.45 \\
        Pythia & Pile10k & Bookcorpus & 0.25 \\
        Pythia & Pile10k & train dataset & 0.70 \\
        Pythia & WikiText-103 & Pile10k & 0.64 \\
        Pythia & WikiText-103 & WikiText-103 & 0.70 \\
        Pythia & WikiText-103 & Bookcorpus & 0.28 \\
        Pythia & WikiText-103 & train dataset & 0.68 \\
        Pythia & Bookcorpus & Pile10k & 0.60 \\
        Pythia & Bookcorpus & WikiText-103 & 0.46 \\
        Pythia & Bookcorpus & Bookcorpus & 0.72 \\
        Pythia & Bookcorpus & train dataset & 0.66 \\
        Olmo & Pile10k & Pile10k & 0.74 \\
        Olmo & Pile10k & WikiText-103 & 0.45 \\
        Olmo & Pile10k & Bookcorpus & 0.27 \\
        Olmo & Pile10k & train dataset & 0.66 \\
        Olmo & WikiText-103 & Pile10k & 0.63 \\
        Olmo & WikiText-103 & WikiText-103 & 0.70 \\
        Olmo & WikiText-103 & Bookcorpus & 0.27 \\
        Olmo & WikiText-103 & train dataset & 0.65 \\
        Olmo & Bookcorpus & Pile10k & 0.59 \\
        Olmo & Bookcorpus & WikiText-103 & 0.45 \\
        Olmo & Bookcorpus & Bookcorpus & 0.70 \\
        Olmo & Bookcorpus & train dataset & 0.61 \\
        Opt & Pile10k & Pile10k & 0.76 \\
        Opt & Pile10k & WikiText-103 & 0.44 \\
        Opt & Pile10k & Bookcorpus & 0.31 \\
        Opt & WikiText-103 & Pile10k & 0.64 \\
        Opt & WikiText-103 & WikiText-103 & 0.69 \\
        Opt & WikiText-103 & Bookcorpus & 0.32 \\
        Opt & Bookcorpus & Pile10k & 0.61 \\
        Opt & Bookcorpus & WikiText-103 & 0.45 \\
        Opt & Bookcorpus & Bookcorpus & 0.73 \\
        Mistral & Pile10k & Pile10k & 0.79 \\
        Mistral & Pile10k & WikiText-103 & 0.49 \\
        Mistral & Pile10k & Bookcorpus & 0.29 \\
        Mistral & WikiText-103 & Pile10k & 0.62 \\
        Mistral & WikiText-103 & WikiText-103 & 0.73 \\
        Mistral & WikiText-103 & Bookcorpus & 0.28 \\
        Mistral & Bookcorpus & Pile10k & 0.64 \\
        Mistral & Bookcorpus & WikiText-103 & 0.47 \\
        Mistral & Bookcorpus & Bookcorpus & 0.70 \\
        Llama & Pile10k & Pile10k & 0.69 \\
        Llama & Pile10k & WikiText-103 & 0.43 \\
        Llama & Pile10k & Bookcorpus & 0.29 \\
        Llama & WikiText-103 & Pile10k & 0.57 \\
        Llama & WikiText-103 & WikiText-103 & 0.66 \\
        Llama & WikiText-103 & Bookcorpus & 0.29 \\
        Llama & Bookcorpus & Pile10k & 0.57 \\
        Llama & Bookcorpus & WikiText-103 & 0.43 \\
        Llama & Bookcorpus & Bookcorpus & 0.63 \\
        \hline
        \end{tabular}
    \caption{\label{table:pred_hub_corr}
    For prediction hubs: correlation of $k$-occurrence with frequencies of vocabulary items for all tested models on all tested datasets. Note correlation is strongest when the columns \textbf{context} and \textbf{freq from} agree. 
  }
\end{table*}

\begin{table*}
    \centering
    \begin{tabular}{lllr}
    \hline
    \textbf{model} & \textbf{similarity} & \textbf{freq from} & \textbf{Spearman corr} \\
    \hline
    Pythia & euc & Pile10k & -0.20 \\
    Pythia & euc & WikiText-103 & -0.20 \\
    Pythia & euc & Bookcorpus & -0.12 \\
    Pythia & norm euc & Pile10k & -0.11 \\
    Pythia & norm euc & WikiText-103 & -0.02 \\
    Pythia & norm euc & Bookcorpus & -0.04 \\
    Pythia & softmax dot & Pile10k & -0.07 \\
    Pythia & softmax dot & WikiText-103 & 0.04 \\
    Pythia & softmax dot & Bookcorpus & 0.29 \\
    Olmo & euc & Pile10k & -0.22 \\
    Olmo & euc & WikiText-103 & 0.03 \\
    Olmo & euc & Bookcorpus & 0.05 \\
    Olmo & norm euc & Pile10k & - \\
    Olmo & norm euc & WikiText-103 & - \\
    Olmo & norm euc & Bookcorpus & - \\
    Olmo & softmax dot & Pile10k & -0.59 \\
    Olmo & softmax dot & WikiText-103 & -0.67 \\
    Olmo & softmax dot & Bookcorpus & - \\
    Opt & euc & Pile10k & -0.00 \\
    Opt & euc & WikiText-103 & -0.14 \\
    Opt & euc & Bookcorpus & 0.01 \\
    Opt & norm euc & Pile10k & -0.01 \\
    Opt & norm euc & WikiText-103 & -0.13 \\
    Opt & norm euc & Bookcorpus & -0.00 \\
    Opt & softmax dot & Pile10k & -0.14 \\
    Opt & softmax dot & WikiText-103 & -0.16 \\
    Opt & softmax dot & Bookcorpus & -0.12 \\
    Mistral & euc & Pile10k & -0.45 \\
    Mistral & euc & WikiText-103 & -0.29 \\
    Mistral & euc & Bookcorpus & -0.23 \\
    Mistral & norm euc & Pile10k & - \\
    Mistral & norm euc & WikiText-103 & - \\
    Mistral & norm euc & Bookcorpus & -0.18 \\
    Mistral & softmax dot & Pile10k & -0.17 \\
    Mistral & softmax dot & WikiText-103 & -0.30 \\
    Mistral & softmax dot & Bookcorpus & -0.14 \\
    Llama & euc & Pile10k & -0.22 \\
    Llama & euc & WikiText-103 & - \\
    Llama & euc & Bookcorpus & - \\
    Llama & norm euc & Pile10k & -0.13 \\
    Llama & norm euc & WikiText-103 & -0.13 \\
    Llama & norm euc & Bookcorpus & -0.13 \\
    Llama & softmax dot & Pile10k & -0.12 \\
    Llama & softmax dot & WikiText-103 & -0.14 \\
    Llama & softmax dot & Bookcorpus & -0.14 \\
    \hline
    \end{tabular}
        \caption{\label{table:tt_hub_corr}
    For hubs in comparisons of vocabulary with vocabulary: $k$-occurrence correlation with frequencies of vocabulary items for all tested models and three different distance measures. We write ``-''  in cases where the correlation coefficient is not well-defined. In the case of OLMo and normalized Euclidean distance, it is because there are only two hubs. In the rest of the cases, it is because all the frequencies are the same.  
  }
\end{table*}

\begin{table*}
    \centering
    \begin{tabular}{lllr}
        \hline
        \makecell{\textbf{Pythia} \\ \textbf{train step}} & \textbf{context} & \textbf{freq from} & \textbf{Spearman corr} \\
        \hline
        512 & Pile10k & Pile10k & 0.59 \\
        512 & Pile10k & WikiText-103 & 0.43 \\
        512 & Pile10k & Bookcorpus & 0.30 \\
        512 & WikiText-103 & Pile10k & 0.47 \\
        512 & WikiText-103 & WikiText-103 & 0.44 \\
        512 & WikiText-103 & Bookcorpus & 0.24 \\
        512 & Bookcorpus & Pile10k & 0.47 \\
        512 & Bookcorpus & WikiText-103 & 0.35 \\
        512 & Bookcorpus & Bookcorpus & 0.39 \\
        4000 & Pile10k & Pile10k & 0.70 \\
        4000 & Pile10k & WikiText-103 & 0.42 \\
        4000 & Pile10k & Bookcorpus & 0.26 \\
        4000 & WikiText-103 & Pile10k & 0.61 \\
        4000 & WikiText-103 & WikiText-103 & 0.64 \\
        4000 & WikiText-103 & Bookcorpus & 0.28 \\
        4000 & Bookcorpus & Pile10k & 0.54 \\
        4000 & Bookcorpus & WikiText-103 & 0.42 \\
        4000 & Bookcorpus & Bookcorpus & 0.62 \\
        16000 & Pile10k & Pile10k & 0.72 \\
        16000 & Pile10k & WikiText-103 & 0.44 \\
        16000 & Pile10k & Bookcorpus & 0.27 \\
        16000 & WikiText-103 & Pile10k & 0.64 \\
        16000 & WikiText-103 & WikiText-103 & 0.70 \\
        16000 & WikiText-103 & Bookcorpus & 0.31 \\
        16000 & Bookcorpus & Pile10k & 0.61 \\
        16000 & Bookcorpus & WikiText-103 & 0.47 \\
        16000 & Bookcorpus & Bookcorpus & 0.66 \\
        64000 & Pile10k & Pile10k & 0.71 \\
        64000 & Pile10k & WikiText-103 & 0.45 \\
        64000 & Pile10k & Bookcorpus & 0.26 \\
        64000 & WikiText-103 & Pile10k & 0.63 \\
        64000 & WikiText-103 & WikiText-103 & 0.71 \\
        64000 & WikiText-103 & Bookcorpus & 0.28 \\
        64000 & Bookcorpus & Pile10k & 0.59 \\
        64000 & Bookcorpus & WikiText-103 & 0.46 \\
        64000 & Bookcorpus & Bookcorpus & 0.71 \\
        \hline
        \end{tabular}
    \caption{\label{table:pred_pythia_steps_hub_corr}
    For prediction hubs in Pythia training checkpoints: correlation of $k$-occurrence with frequencies of vocabulary items on all three datasets. Correlation where the columns \textbf{context} and \textbf{freq from} agree increases with the training step. The correlation saturates faster for Pile10k, probably because Pythia was trained on the Pile. 
  }
\end{table*}

\begin{table*}
    \centering
    \begin{tabular}{lllr}
        \hline
        \makecell{\textbf{Pythia} \\ \textbf{train step}} & \textbf{context} & \textbf{freq from} & \textbf{Spearman corr} \\
        \hline
        512 & euc & Pile10k & -0.03 \\
        512 & euc & WikiText-103 & 0.02 \\
        512 & euc & Bookcorpus & -0.05 \\
        512 & norm euc & Pile10k & - \\
        512 & norm euc & WikiText-103 & - \\
        512 & norm euc & Bookcorpus & - \\
        512 & softmax dot & Pile10k & - \\
        512 & softmax dot & WikiText-103 & - \\
        512 & softmax dot & Bookcorpus & - \\
        4000 & euc & Pile10k & 0.06 \\
        4000 & euc & WikiText-103 & 0.09 \\
        4000 & euc & Bookcorpus & -0.04 \\
        4000 & norm euc & Pile10k & - \\
        4000 & norm euc & WikiText-103 & - \\
        4000 & norm euc & Bookcorpus & - \\
        4000 & softmax dot & Pile10k & - \\
        4000 & softmax dot & WikiText-103 & - \\
        4000 & softmax dot & Bookcorpus & - \\
        16000 & euc & Pile10k & 0.04 \\
        16000 & euc & WikiText-103 & 0.04 \\
        16000 & euc & Bookcorpus & -0.19 \\
        16000 & norm euc & Pile10k & - \\
        16000 & norm euc & WikiText-103 & - \\
        16000 & norm euc & Bookcorpus & - \\
        16000 & softmax dot & Pile10k & -1.00 \\
        16000 & softmax dot & WikiText-103 & - \\
        16000 & softmax dot & Bookcorpus & -1.00 \\
        64000 & euc & Pile10k & -0.16 \\
        64000 & euc & WikiText-103 & -0.16 \\
        64000 & euc & Bookcorpus & -0.11 \\
        64000 & norm euc & Pile10k & -0.51 \\
        64000 & norm euc & WikiText-103 & 0.20 \\
        64000 & norm euc & Bookcorpus & -0.47 \\
        64000 & softmax dot & Pile10k & -0.34 \\
        64000 & softmax dot & WikiText-103 & -0.27 \\
        64000 & softmax dot & Bookcorpus & 0.38 \\
        \hline
        \end{tabular}
    \caption{\label{table:tt_pythia_steps_hub_corr}
    For vocabulary to vocabulary hubs in training checkpoints of Pythia: correlation of $k$-occurrence with frequencies of vocabulary items on all three datasets. There is no general correlation with frequent tokens. We write ``-''  in cases where the correlation coefficient is not well-defined. 
  }
\end{table*}

\section{Computing resources}
All experiments were run using a single NVIDIA A30 GPU. Extracting context representations took about 2 hours. Calculating probabilities for all models took about 2 days. Calculations of distance distributions (with precomputed probabilities) took about 10 hours. Calculations for comparing prediction hubs with frequent tokens about 2 hours. Calculations for vocabulary to vocabulary hubs took about 3 hours. Calculations for context to context hubs, about 1 hour. Calculations for plotting k-occurence distributions took about 8 hours. Getting hub examples took less than a minute.  All in all, about 3 days of compute time were needed to run all experiments. 

\section{Assets}
\label{app:assets}
Besides standard tools such as Python (version 3.10.14) and its main libraries, we used the following tools and datasets, in accordance with their respective terms and licenses. 
\begin{description}
    \item[Bookcorpus] \url{https://huggingface.co/datasets/bookcorpus}; license: unknown
    \item[Pile-10k] \url{https://huggingface.co/datasets/NeelNanda/pile-10k}; license: bigscience-bloom-rail-1.0
    \item[Wikitext] \url{https://huggingface.co/datasets/wikitext}; license: Creative Commons Attribution Share Alike 3.0
    \item[Llama] \url{https://huggingface.co/meta-llama/Meta-Llama-3-8B}; license: llama3
    \item[Mistral] \url{https://huggingface.co/mistralai/Mistral-7B-v0.1}; license: apache-2.0
    \item[OLMo] \url{https://huggingface.co/allenai/OLMo-7B}; license: apache-2.0 
    \item[OPT] \url{https://huggingface.co/facebook/OPT-6.7b}; license: OPT-175B license 
    \item[Pythia] \url{https://huggingface.co/EleutherAI/pythia-6.9b-deduped}; license: apache-2.0
    \item[scikit-learn] \url{https://scikit-learn.org/}; license: bsd; scikit-learn 1.5.1 py310h1128e8f\_0
    \item[PyTorch] \url{https://pytorch.org/}; license: bsd; pytorch 2.4.1 py3.10\_cuda12.1\_cudnn9.1.0\_0
    \item[Dolma] \url{https://huggingface.co/datasets/allenai/dolma}; license: ODC-By
    \item [The Pile] \url{https://pile.eleuther.ai/}; license: MIT
    \item [Huggingface Transformers] \url{https://github.com/huggingface/transformers}; license:apache-2.0; transformers 4.45.2 pyhd8ed1ab\_1
\end{description}

\section{AI use disclosure}
Microsoft Copilot has been used for minor auto completions in the code.

\end{document}